\documentclass{article}

\PassOptionsToPackage{numbers, compress}{natbib}

\newcommand{\revision}[1]{#1}

\usepackage[final]{neurips_2022}

\usepackage[utf8]{inputenc} %
\usepackage[T1]{fontenc}    %
\usepackage[dvipsnames]{xcolor}         %
\usepackage[colorlinks=true,linkcolor=violet,citecolor=MidnightBlue,allbordercolors={1 1 1}]{hyperref}       %
\usepackage{url}            %
\usepackage{booktabs,multirow}       %
\usepackage{amsfonts}       %
\usepackage{nicefrac}       %
\usepackage{microtype}      %
\usepackage{algorithm}
\usepackage{algorithmic}
\usepackage{subfigure}

\usepackage{enumitem,cancel}

\usepackage[disable]{todonotes}

\usepackage{mathtools}

\usepackage{enumitem}
\usepackage{amsmath,amssymb,amsthm,amsfonts,graphicx,latexsym,xcolor,nicefrac}
\usepackage{verbatim,framed,booktabs,comment}
\usepackage{accents,cancel}
\usepackage{aligned-overset}

\usepackage[capitalize,noabbrev]{cleveref}

\usepackage{minitoc}

\input{def.set}
\theoremstyle{plain}
\newtheorem{theorem}{Theorem}[section]
\newtheorem{proposition}[theorem]{Proposition}
\newtheorem{lemma}[theorem]{Lemma}
\newtheorem{claim}[theorem]{Claim}
\newtheorem{corollary}[theorem]{Corollary}
\newtheorem{example}{Example}[section]
\theoremstyle{definition}
\newtheorem{definition}{Definition}[section]
\newtheorem{assumption}{Assumption}[section]
\theoremstyle{remark}
\newtheorem{remark}{Remark}[section]

\DeclareMathOperator{\Tr}{Tr}

\newcommand{\Proj}[2]{\mrm{Proj}_{#1}{#2}}

\newcommand{\gpi}[1]{g^{(#1)}}
\newcommand{\esti}[1]{\hat g^{(#1)}_{n_1}}
\newcommand{\estiRaw}[1]{\hat g^{(#1)}_{u,n_1}}
\newcommand{\data}{\cD^{(n_1)}}
\newcommand{\dataSI}{\cD_{s1}^{(n_1)}}
\newcommand{\dataSII}{\cD_{s2}^{(n_2)}}
\newcommand{\SkipNOTE}[1]{}
\newcommand{\beff}{{\bar b}}
\newcommand{\fApprox}[1][m]{f^\dagger_{#1}}
\newcommand{\gApprox}[1][m]{g^\dagger_{#1}}

\newcommand{\Ztest}{Z^{n_2}}
\newcommand{\itI}{{\em (i) }}
\newcommand{\itII}{{\em (ii) }}
\newcommand{\itIII}{{\em (iii) }}
\let\rabak\>
\renewcommand{\>}{\rangle}

\title{Fast Instrument Learning with Faster Rates}

\author{%
Ziyu Wang,~  
Yuhao Zhou,~ 
Jun Zhu\thanks{JZ is the corresponding author. Code available at \url{https://github.com/meta-inf/fil}.} \\ 
Dept.~of Comp.~Sci.~and Tech., BNRist Center, State Key Lab for Intell. Tech. \& Sys., \\
Institute for AI,
  Tsinghua-Bosch Joint Center for ML, Tsinghua University
\\
\texttt{\small \{wzy196,yuhaoz.cs\}@gmail.com, 
dcszj@tsinghua.edu.cn}
}

\begin{document}

\maketitle

\begin{abstract}
We investigate nonlinear instrumental variable (IV) regression given high-dimensional instruments. We propose a simple algorithm which combines kernelized IV methods and an arbitrary, adaptive regression algorithm, accessed as a black box. Our algorithm enjoys faster-rate convergence and adapts to the dimensionality of informative latent features, while avoiding an expensive minimax optimization procedure, which has been necessary to establish similar guarantees. It further brings the benefit of flexible machine learning models to quasi-Bayesian uncertainty quantification, likelihood-based model selection, and model averaging. Simulation studies demonstrate the competitive performance of our method.
\end{abstract}

\doparttoc %
\faketableofcontents %

\section{Introduction}\label{sec:intro}

\todo{disable todonotes and the extra margin}
Instrumental variable (IV) analysis is widely used for causal inference \citep{angrist2008mostly,brookhart2010instrumental,burgess2017review}. 
Given confounded observational data, 
IV analysis identifies the causal effect through the use of {\em instruments}. 
Nonlinear IV regression is typically defined by the following conditional moment restrictions (CMRs): 
\begin{equation}\label{eq:npiv}
\EE(\by-f_0(\bx)\mid \bz)=0\quad a.s.~[P(dz)]
\end{equation}
where $f_0$ is the causal effect function of interest, 
and $\bx,\by,\bz$ denote the observed treatment, response and instrument, respectively. Similar CMR problems also appear in other applications of causal statistics and machine learning \citep[see, e.g.,][for examples]{chen_estimation_2012,liao_provably_2020}.

Starting from \cite{hartford2017deep}, 
recent works have demonstrated great promise in applying flexible machine learning (ML) methods to IV regression. Modern ML methods are appealing due to 
their {\em adaptivity} to the {\em informative latent structure} in data \citep{bach2017breaking}: they may adapt to the low dimensionality of the informative latent features even if the observed input is high-dimensional and its signal-to-noise ratio is low. 
Sample complexity gaps have been established in such settings, between deep models based on neural networks (NNs) or trees, and linear models such as fixed-form kernels \citep{wei2019regularization,ghorbani2019limitations,schmidt-hieber_nonparametric_2020}. 
In IV regression, such adaptivity will be highly desirable when the observed instruments are high-dimensional, which is prevalent in applications such as genomics \citep{lin_regularization_2015}, and may also 
arise from the general desire to use structured data %
as instrument.  
Following previous work in the parametric setting \citep{singh2020machine,chen_mostly_2021}, we refer to this problem of learning informative latents in %
instruments as {\em instrument learning}. It generalizes %
the classical problem of instrument selection \citep{okui2011instrumental,carrasco2012regularization,belloni2012sparse}. 

Comparing with standard supervised learning, IV regression is more challenging, due to the need to estimate a conditional expectation operator which defines %
\eqref{eq:npiv}. 
Consequently, establishing adaptivity guarantees becomes more difficult. 
While many recent works have demonstrated promising empirical results using deep models, they are often used as heuristics \citep[e.g.,][]{hartford2017deep}, or justified with crude {\em slow-rate} analyses, which %
establish convergence rates that saturate at $\Omega(n^{-1/4})$ \citep[e.g.,][]{xu_learning_2020}. This is in contrast to {\em faster rates} which approach $n^{-1/2}$ as the regularity of model improves. 
The only exception is a minimax formulation of IV estimation \citep{liao_provably_2020,bennett2019deep,dikkala_minimax_2020,muandet_dual_2020}. \cite{dikkala_minimax_2020} establish %
faster rate convergence for this formulation, for models with local Rademacher complexity bounds. Though local Rademacher analysis %
covers many adaptive ML procedures \citep[e.g.,][]{schmidt-hieber_nonparametric_2020,syrgkanis_estimation_2020},  
it still does not fully explain the success of modern ML approaches, with prominent alternatives %
including implicit regularization \citep{neyshabur2014search} and PAC-Bayesian analyses \citep{alquier2021user}. 
From a practical perspective, minimax optimization is computationally expensive, yet it cannot be avoided in the framework of \cite{dikkala_minimax_2020}, unless we instantiate their method with the less flexible kernel models. 
It requires additional hyperparameter tuning, which can be challenging in causal problems where validation is indirect and more difficult. 
It also prevents the use of such flexible models for {\em uncertainty quantification}, or {\em inference}, for which reliable methods have only been developed for linear nonparametric models \citep{kato_quasi-bayesian_2013,chen_optimal_2015,wang2021quasibayesian}. 

This work bridges the gap between sharp theoretical guarantees and robust, practical implementation. We assume our prior knowledge about the causal effect function $f_0$ is characterized by a reproducing kernel Hilbert space (RKHS) $\cH$, and focus on the flexibility of conditional expectation estimation. 
This is often possible, because the treatment variable is determined by the problem at hand, and thus has a fixed dimensionality.\footnote{Our method can also be applied when $\bx$ contains high-dimensional exogenous covariates; see Appendix~\ref{app:exo-algo}.} 
We then present a surprisingly simple algorithm, with faster rate guarantees which in many cases match the best known in literature. 
The algorithm defines the conditional expectation estimates using a learned kernel, the basis of which is %
defined by applying adaptive regression algorithms to random draws from a Gaussian process (GP) prior. %
Given this learned first-stage kernel and $\cH$, we can estimate $f_0$ using kernelized IV methods \citep{dikkala_minimax_2020,muandet_dual_2020,singh_kernel_2020}, which have closed-form solutions and can be efficiently approximated (e.g.,~with Nystr\"om \cite{dikkala_minimax_2020}). 
Our method allows easy hyperparameter tuning, and exhibits competitive performance in simulations. 
It accesses the regression algorithm as a black box, thus allowing for the use of any ML methods and %
benefits from their established theoretical guarantees. 
It also enables fast quasi-Bayesian uncertainty quantification \citep{kato_quasi-bayesian_2013,chernozhukov_mcmc_2003,florens_nonparametric_2012} with improved flexibility. 

Our algorithm connects to many ideas in literature. Most notably, it 
can be viewed as an infinite-dimensional generalization of \cite{singh2020machine,chen_mostly_2021}, which consider a linear outcome model with fixed dimensionality, and use ML methods to learn the {\em optimal instruments} \citep{kelejian1971two,chamberlain_asymptotic_1987}. 
Our setting requires different analyses, for defining an infinite-dimensional estimation target and quantifying errors with its intrinsic complexity. Additionally, 
analysis of the resultant IV estimator is complicated by 
the ill-posedness of infinite dimensional IV models \cite{horowitz_ill-posed_2014}. 
A more subtle distinction is in the choice of basis: while for finite-dimensional function spaces we can pick any set of basis (i.e., features) and apply the black-box regressor separately, in our case seemingly obvious choices of basis lead to inferior results (Appendix~\ref{rmk:iv-alternative}).  
Our analysis also connects to the multi-task learning literature, to which we make technical contributions. Section~\ref{sec:related-work} discusses related work in detail. 

The remaining of the paper is organized as follows: Section~\ref{sec:setup} reviews background knowledge. 
Section~\ref{sec:iv-to-klearn} 
introduces the instrument learning problem, and reduces it to a general kernel learning problem. We solve the latter problem in Section~\ref{sec:main}, and return to IV %
in Section~\ref{sec:iv} with our main results. 
We review related work in Section~\ref{sec:related-work}, and present numerical experiments in Section~\ref{sec:experiments}. 

\section{Notations and Setup}\label{sec:setup}

\paragraph{Notations} 
Denote the joint data distribution as $P(dz\times dx\times dy)$, its marginal distributions as $P(dx),P(dz)$, etc., and their support as $\cX,\cZ$. 
For functions of observed variables (e.g., $\bx$ or $\bz$), $\|\cdot\|_2$ denotes the $L_2$ norm w.r.t.~the respective marginal data distribution. $\|\cdot\|_\infty$ denotes the $L_\infty$ norm. 
We use the notation $[m]:=\{1,\ldots,m\}$. %
Boldface ($\bx,\by,\bz$) emphasizes the denotation of random variables. For any kernel $k$, $\mc{GP}(0,k)$ refers to the ``standard Gaussian process'' \citep{van2008reproducing} with zero mean, and covariance defined by $k$.
$\lesssim,\gtrsim,\asymp$ represent (in)equalities up to constants; the hidden constants will not depend on any sample size. $\tilde \cO(\cdot)$ denotes inequality up to logarithm factors.

\paragraph{Problem Setup} %
Nonparametric IV regression (NPIV) is formulated as \eqref{eq:npiv} \citep{newey2003instrumental,horowitz_applied_2011}. Introduce the {\em conditional expectation operator} 
$E: L_2(P(dx))\to L_2(P(dz)), f\mapsto \EE(f(\bx)\mid\bz=\cdot)$, and define $g_0(z) := \EE(\by\mid \bz=z)$. We can then express \eqref{eq:npiv} as a linear inverse problem:
\begin{equation}\label{eq:npiv-inv-problem}
E f_0 = g_0,
\end{equation}
where we observe $g_0$ up to regression error. 
NPIV deviates from standard inverse problems in its need to estimate both $f_0$ and 
$E$. Following conventions in the two stage least square method \citep{angrist2008mostly}, %
we refer to the modeling of $f_0$ as the \emph{second stage}, and that of 
$E$ -- or equivalently, that of $Ef$ for all $f$ in a hypothesis space -- as the \emph{first stage}. 
The following assumption describes the setup in full detail: 
\begin{assumption}[NPIV]\label{ass:npiv} {\em(i)}
The data variables $\bx,\by,\bz$ satisfy \eqref{eq:npiv}, and $\by$ is bounded by $B$. {\em(ii)}
We observe two sets of i.i.d.~samples, with matching (marginal) distributions: $\dataSII := \{(z_i,x_i,y_i): i\in [n_2]\}$, %
$\dataSI := \{(\tilde z_i,\tilde x_i): i\in [n_1]\}$. 
\end{assumption}
\vspace{-0.1em}
We impose \itII since such additional samples are sometimes available, as discussed in e.g.~\cite{singh_kernel_2020}. %
If only $n$ samples from the joint distribution are available, we can set $n_1=n_2=n/2$. 

In the main text, we assume that our prior knowledge about $f_0$ is fully characterized by an RKHS $\cH$, in the following sense. 
\begin{assumption}[second stage RKHS]\label{ass:s2}
{\em(i)} $\cX$ is a bounded subset of $\RR^{d_x}$; the reproducing kernel $k_x$ of $\cH$ is bounded and continuous. %
{\em(ii)} The integral operator $T_x: f\mapsto \int k_x(x,\cdot)f(x)P(dx)$ has eigenvalues 
$\lambda_i(T_x) \lesssim i^{-(b+1)}$, for some $b>0$.  
{\em(iii)} {\em One of the following} holds true:
\begin{enumerate}[label={\em(iii)}.\alph*,leftmargin=*]
    \item\label{it:kernel-scheme}\!({\underline{``kernel scheme''}}):  $f_0\in\cH$.
    \item\label{it:gp-scheme}\!({\underline{``GP scheme''}}): $b>1$; and for all $n$, $\exists \fApprox[n]\in\cH$ s.t.~$%
\|\fApprox[n]-f_0\|_\cH\lesssim n^{\frac{1/2}{b+1}}, ~
\|\fApprox[n]-f_0\|_2\lesssim n^{-\frac{b/2}{b+1}}.
$%
\end{enumerate}
\end{assumption}
\vspace{-0.1em}
In the above, 
\itI and \itII are common technical assumptions: \itI ensures the existence of Mercer's representation, and \itII is a complexity measure, \revision{with a larger value of $b$ indicating a smaller hypothesis space}. %
\itIII requires $\cH$ is correctly specified for regression; its two cases cover the different assumptions in standard RKHS-based estimation and GP modeling. {\ref{it:kernel-scheme}} is intuitive. {\ref{it:gp-scheme}} is standard in the posterior contraction literature \citep{van_der_vaart_rates_2008}; 
it roughly requires $f_0$ to be (at least) as regular as %
``typical'' samples from $\mc{GP}(0, k_x)$, in the sense of \cite[Theorem 2.1]{van_der_vaart_rates_2008}. This is different from \ref{it:kernel-scheme} because when $\cH$ is infinite dimensional, almost all GP samples fall out of $\cH$ \cite{driscoll1973reproducing,van_der_vaart_information_2011}. 
Our algorithm applies to both settings, but analysis of the 
GP scheme requires additional effort. It is useful as it allows for quasi-Bayesian uncertainty quantification using a $\mc{GP}(0,k_x)$ prior \cite{wang2021quasibayesian}. 

\revision{
The RKHS assumption has been employed in a thread of recent work \citep{singh_kernel_2020,muandet_dual_2020,zhang_maximum_2020,wang2021quasibayesian}, and generalizes the sieve method in literature \citep{newey2003instrumental,blundell_semi-nonparametric_2007,chen_estimation_2012}.
} 
It is most reasonable when $\bx$ has moderate dimensions; for example, when $f_0$ satisfies certain $L_2$-Sobolev regularity conditions, we can set $k_x$ to be a suitable Mat\'ern kernel (Example~\ref{ex:matern}-\ref{ex:f0}). 
Appendix~\ref{app:exo-algo} studies a more general setting, where $\bx$ and $\bz$ include additional, high-dimensional exogenous covariates. 
\revision{
Nonetheless, the assumption will be less reasonable when the treatment variable is high-dimensional and variable selection is needed for it. 
}

NPIV is typically an ill-posed inverse problem \citep{horowitz_applied_2011}. We now quantify the degree of ill-posedness:
\begin{assumption}\label{ass:mildly-ill-posed}
The operator $E$ is compact, with singular values 
$s_i(E)\asymp i^{-p}$, where $p>0$. 
\end{assumption}
\vspace{-0.1em}
Such \emph{mildly ill-posed} settings \cite{cavalier2008nonparametric} match our polynomial eigendecay assumption for the kernel. 
In the \emph{severely ill-posed} setting where the decay of $s_i(E)$ is exponential, kernels with a similar eigendecay should be used. %
While the analyses of the two settings share many ideas, 
the Bayesian inverse problem literature typically restricts to the former for technical reasons \citep{knapik_bayesian_2011,knapik2016bayes}. %

\section{From Instrument to Kernel Learning}\label{sec:iv-to-klearn}

As we assume $\cH$ is a correctly specified second-stage model, it remains to determine the first stage. 
In this section, we show that an ideal first stage model can be defined using another RKHS $\cI$, determined by $\cH$ and $E$. 
Although its kernel $k_z$ has an unknown form, we demonstrate that we can access noisy samples from $\mc{GP}(0, k_z)$, which, as Section~\ref{sec:main} below shows, enable efficient learning of $\cI$. %
This can be viewed as instrument learning, as $\cI$ will only depend on the informative features in $\bz$ (Example~\ref{ex:informative-latents}). 

Let us first consider the GP scheme (Assumption \ref{ass:s2} \ref{it:gp-scheme})
which roughly requires $f_0$ to be similar to typical samples from $\mc{GP}(0, k_x)$. From a Bayesian perspective, an ideal prior for $E f_0$ should match the distribution of $E f$, for $f\sim\mc{GP}(0, k_x)$. %
This distribution is ``almost equivalent'' to another GP: %
\begin{lemma}[proof in Appendix~\ref{app:proof-sec-iv-to-kl}]\label{lem:iv-to-kl-1}
Denote by $[g]_\sim$ the $L_2$ equivalence class of $g$.\footnote{Recall the $L_2$ space is not a function space, and consists of equivalence classes of functions. 
Note that for readability, we may occasionally ignore this distinction in the main text, and use (a version of) $E$ to also denote the corresponding map between function spaces. All such denotations can be made unambiguous (Remark~\ref{rmk:regular-cond-exp}), and all null set ambiguities in this section can be removed under mild additional assumptions (Lemma~\ref{lem:iv-to-kl-further}).} 
Under Assumptions~\ref{ass:npiv}, \ref{ass:s2}, there exists a kernel $k_z$, with integral operator $T_z = E T_x E^\top$,
s.t.~for $f\sim\mc{GP}(0,k_x),g\sim\mc{GP}(0, k_z)$, $[g]_\sim$ has the same distribution as $E[f]_\sim$. 
\end{lemma}
Informally, the lemma shows that $\mc{GP}(0,k_z)$ matches the distribution of $E f$.  
It is thus intuitive that $k_z$ could be a good choice for the first stage. 
The following lemma further motivates its use in the kernel scheme:  
its \itI shows that $\cI$ fulfills the conditions in previous work \cite{dikkala_minimax_2020,singh_kernel_2020}: the restriction of $E$ on $\cH$ has image contained in $\cI$, and is a bounded linear map %
to $\cI$. \itII shows that $\cI$, as a set of functions, cannot be made smaller while maintaining \itI. 
\begin{lemma}[proof in Appendix~\ref{app:proof-sec-iv-to-kl}]\label{lem:optimal-I}
Let $\cI$ be the RKHS defined by $k_z$. 
Under Assumptions~\ref{ass:npiv}, \ref{ass:s2}, {\em(i)}
for any $f\in\cH$, there exists $g\in\cI$ s.t.~$[g]_\sim = E[f]_\sim$; {\em(ii)}
for any $g\in\cI$, there exists $f\in\cH$ satisfying the above. 
In both cases, we have $\|f\|_\cH = \|g\|_\cI$. 
\end{lemma}

We now demonstrate that $\cI$ only depends on the informative latent features. 
\begin{example}[informative latent structure]\label{ex:informative-latents}
Let $\Phi: \cZ\to\bar\cZ$ be a {\em feature extractor} that maps the observed instruments $\bz$ to latent features $\bar\bz:=\Phi(\bz)$, s.t.~$\EE(f(\bx)\mid\bz)=\EE(f(\bx)\mid\Phi(\bz))$ for all $L_2$-integrable $f$. 
Then %
we can apply Lemma~\ref{lem:iv-to-kl-1}, with $E$ replaced by $\bar E: f\mapsto \EE(f(\bx)\mid \bar\bz)\in L_2(P(d\bar z))$, leading to a latent-space
RKHS $\bar\cI$ with kernel $\bar k_z$. $\bar\cI$ induces the input-space RKHS 
$$
\cI := \{g = \bar g\circ \Phi: \bar g\in\bar\cI\}, ~~\|\bar g\circ\Phi\|_\cI := \|\bar g\|_{\bar\cI}; ~~ k_z(z,z') = \bar k_z(\Phi(z),\Phi(z')).
$$
The above $k_z$ satisfies Lemma~\ref{lem:iv-to-kl-1}-\ref{lem:optimal-I}.\footnote{There may be multiple kernels satisfying Lemma~\ref{lem:iv-to-kl-1}, but they are equivalent up to null sets (Claim~\ref{claim:ex-fs}); the ambiguity can be removed under mild assumptions (Lemma~\ref{lem:iv-to-kl-further}).} 
Observe that $\cI$ perfectly approximates $\{E f:f\in\cH\}$, but its complexity only depends on $\bar\cI$. 
In particular, 
$k_z$ has the same Mercer eigenvalues as $\bar k_z$ (Claim~\ref{claim:ex-fs}), %
the decay of which is a standard complexity measure \citep[e.g.,][Ch.~7]{steinwart2008support}.

\end{example}

While $k_z$ has ideal properties, it cannot be used directly as it involves the unknown operator $E$. Instead, we need to construct an %
approximation from data. 
Our main insight is that {\em we can effectively draw noisy samples from $\mc{GP}(0,k_z)$}; as we develop in Section~\ref{sec:main}, such samples enable the approximation of $k_z$. 
To see how the noisy samples are obtained, consider $f\sim\mc{GP}(0,k_x)$. By Lemma~\ref{lem:iv-to-kl-1}, $g = E f$ is $L_2$-equivalent to clean samples from $\mc{GP}(0,k_z)$; and we have $f(\bx) = g(\bz) + (f(\bx)-(Ef)(\bz))$, 
where the latter term is unpredictable given $\bz$, \revision{and from this perspective can be viewed as noise}. Thus, if we apply any regression algorithm to $f\sim\mc{GP}(0,k_x)$, with $\bz$ as input, we will recover a ``denoised'' sample from $\mc{GP}(0,k_z)$, up to regression errors. 

In the informative latent feature setting, optimal regression error can only be achieved by methods that adapt to such structures \citep{wei2019regularization,ghorbani2019limitations,schmidt-hieber_nonparametric_2020}. Approximating $\cI$ with such ``denoised'' samples can then be viewed as a knowledge distillation procedure, which results in a compact representation of the adaptive regression algorithm. This is particularly beneficial in the NPIV setting: as discussed in the introduction, 
using a learned kernel eliminates the need of minimax optimization in estimation, and allows the (indirect) use of adaptive methods for uncertainty quantification.

\vspace{-0.3em}
\section{Black-Box Kernel Learning}\label{sec:main}

In this section, we address the problem of kernel learning given noisy GP samples. %
As our results apply to more general settings, we first state the assumptions with full generality. 

\begin{assumption}[RKHS]\label{ass:s1-rkhs}
There exist a continuous function $\Phi:\cZ\to \bar\cZ$, and a reproducing kernel $\bar k_z$ over $\bar\cZ$, s.t.
{\em(i)} the random variable $\bar \bz=\Phi(\bz)$ is supported on a bounded subset of $\RR^{d_l}$; $\bar k_z$ is bounded. 
{\em(ii)} The eigenvalues of the integral operator $T_{\bar z}: \bar g\mapsto \int \bar k_z(\bar z, \cdot) g(\bar z) P(d\bar z)$ 
satisfy $\lambda_i(T_{\bar z}) \lesssim i^{-(\beff+1)}$, for some $\beff>0$. 
{\em(iii)} $\bar g\sim\mc{GP}(0,\bar k_z)$ have finite sup norm with probability 1.
\end{assumption}

The above assumption applies to a latent-space kernel $\bar k_z$. %
As shown in Example~\ref{ex:informative-latents}, $\Phi$ and $\bar k_z$ induce an input-space kernel $k_z$, and RKHS $\cI$, which inherit the assumed regularity conditions. Our goal is to estimate $k_z$. This is harder than the estimation of $\bar k_z$, as it also involves $\Phi$. 

All conditions for $\bar k_z$ are satisfied by Mat\'ern kernels with a sutiable order; see Appendix~\ref{app:technical}. Appendix~\ref{app:regularity} discusses its applicability in the IV setting, where $\cI$ is defined as in Section~\ref{sec:iv-to-klearn}. Briefly, 
{\em(ii)} always holds for $\beff\ge \max\{b, 2p-1\}$, and if $\cH$ is further correctly specified in the sense of Assumption~\ref{ass:qb-1}, $\beff = b+2p$. 
{\em(i)} and {\em(iii)} hold under mild technical assumptions. 

We will ``denoise'' noisy $\mc{GP}(0,k_z)$ samples using a regression oracle, which is specified below: 
\begin{assumption}[regression oracle]\label{ass:s1-oracle}
Let $\cD^{(n_1)} := \{(\tilde z_i, g(\tilde z_i)+e_i)\}$ be $n_1$ iid replications of the rvs $(\bz,g(\bz)+\be)$, 
s.t.~$\EE(\be\mid \bz)=0$ and $g(\bz)+\be$ has a $1$-subgaussian distribution. Then 
the oracle returns estimator $\hat g_{u,n_1}$ s.t.
$
\EE_{g\sim\mc{GP}(0,k_z)}\EE_{\data} \|\hat g_{u,n_1} - g\|_2^2 \le \xi_{n_1}^2,
$
 for some $\xi_{n_1}\to 0.$
\end{assumption}
In the IV setting, we have $g=Ef \sim\mc{GP}(0,k_z)$, and $\EE(\be\mid\bz)= \EE(f(\bx)-(Ef)(\bz)\mid\bz)=0$; the subgaussian condition is verified by Lemma~\ref{lem:borell-tis}.

To provide some intuition on adaptivity, %
we instantiate the assumption with the DNN model in \cite{schmidt-hieber_nonparametric_2020}, and compare the resulted $\xi_n$ with fixed-form kernels: %
\begin{example}[adaptivity of DNN oracles]\label{ex:dnn-oracles}
Let $\cZ\subset \RR^{d_z}$, 
$\Phi: \cZ \to \bar\cZ$ be $\beta_1$-H\"older regular, 
$\bar\cI$ be a Mat\'ern-$\beta_2$ RKHS, and $\beta_1,\beta_2 \ge 1$. Let the regression oracle return a $\epsilon_{opt}^2$-approximate empirical risk minimizer for the 
model in \cite{schmidt-hieber_nonparametric_2020}. Then for any $\epsilon>0$, it holds that (see Appendix~\ref{app:derivation-examples} for derivations)
\begin{equation}\label{eq:xi-n-dnn}
\xi_n =\tilde \cO\Big(n^{-\frac{\beta_1}{2\beta_1+d_z}} + n^{-\frac{\beta_2-\epsilon}{2\beta_2+d_l}} + \epsilon_{opt}\Big) =: 
\tilde \cO\Big(\epsilon_{fea,n} + n^{-\frac{\beta_2-\epsilon}{2\beta_2+d_l}} + \epsilon_{opt}\Big)
\end{equation} %
In the above, $\epsilon_{fea,n}$ characterizes the hardness of feature learning, i.e., learning $\Phi$. The second term characterizes that of kernelized regression given the optimal features: it matches the optimal regression rate if we {\em had} full knowledge about $\Phi$, or equivalently, $\cI$, and would be attainable by kernel ridge regression (KRR) using $\cI$. %

As long as $\epsilon_{opt}$ is small, \eqref{eq:xi-n-dnn} will match the minimax rate up to logarithms. When $\nicefrac{\beta_1}{d_z} < \nicefrac{\beta_2}{d_l}$, the minimax rate is 
$
\epsilon_{fea,n} \gg n^{-\beta_2/(2\beta_2+d_l)}, 
$
meaning that the hardness of feature learning cannot be overlooked. 
Otherwise, the rate $\xi_n$ nearly matches the %
rate given full knowledge of the unknown $\cI$, up to the infinitesimal $\epsilon>0$; 
this is realistic when, e.g., $d_z\gg d_l$ and $\Phi$ is linear ($\beta_1=\infty$). %

We are interested in the high-dimensional regime where $d_z \gg d_l$. In this case, 
fixed-form Mat\'ern or RBF kernels could only attain the rate of $\cO\big(n^{-\frac{\min\{\beta_1,\beta_2\}}{2\min\{\beta_1,\beta_2\}+d_z}}\big)$, which  
can always be much worse than \eqref{eq:xi-n-dnn}, {\em regardless of the hardness of feature learning}. This comparison suggests 
that fixed-form kernels cannot adapt to the latent feature structure %
to avoid the curse of dimensionality.\footnote{\cite{schmidt-hieber_nonparametric_2020} establishes formal lower bounds. 
Also, for small $\beta_2$, we can replace $d_z$ with a manifold dimensionality of $\cZ$, but it can still be much larger than $d_l$.}
\end{example}

We now define the approximate RKHS. 
Let
$\{\gpi{j}:j\in[m]\}$ be i.i.d.~samples from the GP prior, 
and $\estiRaw{j}$ be the respective estimate returned by the regression oracle, constructed from the shared dataset 
$
\data = \{(\tilde z_i, \gpi{j}(\tilde z_i)+\epsilon^{(j)}_i): i\in [n_1], j\in [m]\}
$ where $\epsilon^{(j)}_i$ are subgaussian, mean-zero noise. 
Let $\esti{j} := \min\{C_k \log n, \estiRaw{j}(\cdot)\}$, %
where $C_k$ is a constant determined by $\cI$.
Define $\hat G_n(z) := (\esti{1}(z),\ldots,\esti{m}(z))$. Our approximate RKHS is defined as 
\vspace{-0.1em}
\begin{equation}\label{eq:approx-rkhs}
\tilde\cI := \big\{g(z) =\theta^\top\hat G_{n_1}(z)\text{ for some }\theta\in\RR^m\big\},
~~\text{with norm }\|g\|_{\tilde\cI} := \sqrt{m}\|\theta\|_2.
\end{equation}
As $\tilde\cI$ is a finite-dimensional linear space, it is an RKHS. %
We can check that %
$\|g\|_\infty\le C_k\log n_1\|g\|_{\tilde\cI}$.

\vspace{-0.3em}
\paragraph{Theoretical Results}
Under a given model,  
regression error is decomposed into approximation and estimation (i.e., generalization) errors. 
We first present the approximation error bound:
\begin{theorem}[proof in Appendix~\ref{app:proof-prop-approx}]\label{prop:approx} 
Under Assumptions~\ref{ass:s1-rkhs}, \ref{ass:s1-oracle}, 
there exists a universal constant $c_r>0$, and an event $E_{n_1}$ determined by $\gpi{1\ldots m}$ and $\data$ with $\PP_{\data} E_{n_1}\to 1$, on which 
for any $g^*\in L_2(P(dz))$, there exists $\tilde g^*\in\tilde\cI$ s.t.
\begin{align}
\|\tilde g^*\|_{\tilde\cI} &\le c_r\|\Proj{m'}{g^*}\|_\cI, \label{eq:approx-norm-bound}\\ 
\|\tilde g^* - g^*\|_2 &\le c_r\|\Proj{m'}{g^*}\|_{\cI}(\xi_{n_1} + m^{-(\beff+1)/2}) \sqrt{\log n_1}
+ \|g^* - \Proj{m'}{g^*}\|_2,\label{eq:approx-bound}
\end{align}
where $m'=[m/2]$, and $\mrm{Proj}_{m'}$ denotes the projection onto the top $m'$ Mercer eigenfunctions of $k_z$.
\end{theorem}

We will use $\tilde\cI$ to estimate functions on a separate dataset with $n_2$ samples. For a single regression task, the estimation error can be simply bounded as $\tilde\cO(\sqrt{m/n_2})$ \cite{gyorfi2002distribution}. However, our analysis of IV estimation will require quantifying the intrinsic complexity of $\tilde\cI$, which will also allow the use of a larger $m$ in practice. 
The following proposition provides one such result; it will be used in Section~\ref{sec:iv}, to analyze IV estimation in the kernel scheme (Assumption~\ref{ass:s2}~\ref{it:kernel-scheme}). 

\begin{proposition}[proof in Appendix~\ref{app:estimation}]\label{prop:est}
Let $\tilde\cI,\data$ be defined as above, and $\delta_{n_2}$ be the critical radius of the local Rademacher complexity of the norm ball $\tilde\cI_1$ \citep[Ch.~14]{wainwright2019high}. 
On the event defined in Theorem~\ref{prop:approx}, we have 
$
\delta_{n_2} = \tilde O(
n_2^{-(\beff+1)/2(\beff+2)} + m^{-(\beff+1)/2} +\xi_{n_1}
).
$
\end{proposition}
IV estimation in the GP scheme is more delicate, and requires additional analysis of $\tilde\cI$, which is deferred to App.~\ref{app:gp-fixed-design-regr}. 
Before we proceed, however, we illustrate the results on a simple regression task: 
\begin{example}[Example~\ref{ex:dnn-oracles}, cont'd]\label{ex:dnn-oracles-cont}
Let $\Phi,\cI$ be defined as before, and $\esti{j}$ be estimated by the DNN oracle. Suppose $\epsilon_{opt}$ is not greater than the other terms. Then 
\begin{enumerate}[nosep,leftmargin=*,label=\roman*.]
    \item 
Let $m = \lceil n_1^{\beff/(\beff+1)^2}\rceil$. On the event in Theorem~\ref{prop:approx}, 
for any $g^*\in\cI$, there exists $\tilde g^*\in\tilde\cI$ s.t.
$\|\tilde g^*\|_{\tilde\cI} \le c_r \|g^*\|_\cI, 
\|\tilde g^*-g^*\|_2 %
= \tilde O(\|g^*\|_\cI \xi_{n_1})$.  
\item Let $g^*\sim\mc{GP}(0,k_z)$. 
A refined analysis, based on Corollary~\ref{thm:ml}, shows that when $m =\lceil n_1^{1/(\beff+1)}\rceil$, there exists $\tilde g$ s.t.
$
\|\tilde g\|_{\tilde\cI} \lesssim n_1^{1/2(\beff+1)}, 
\EE_{g^*\sim\mc{GP}(0,k_z)} \|\tilde g-g^*\|_2 = \tilde \cO(\xi_{n_1}).
$
\end{enumerate}
(See Appendix~\ref{app:derivation-example-cont} for derivations, and another high-probability bound in the GP scheme.)

Let $\hat g^*_n$ be the truncated OLS estimate using $\tilde\cI$, on a dataset $\{(z_i,g^*(z_i)+e_i):i\in [n_2]\}$ where $\EE(e_i\mid z_i)=0,\mrm{Var}(e_i)\le 1,\|g^*\|_\infty\le B$. Then 
$\EE\|\hat g^*_n-g^*\|_2=\tilde\cO(\|\tilde g^*-g^*\|_2+B\sqrt{m/n_2})$ \citep[Thm.~11.3]{gyorfi2002distribution}. When $n_1=n_2$, the latter term is $\ll \xi_{n_2}$, and 
case (ii) above always matches the DNN rate. 
Case (i) matches the DNN rate when feature learning becomes harder ($\xi_{n_1}^2\gtrsim n_1^{-\beff/(\beff+1)}$); otherwise the rate may be slightly inferior, but still approaches $n_1^{-1/2}$ as the regularity $\beff$ improves. 
\end{example}
As discussed in Example~\ref{ex:dnn-oracles}, when $d_z>d_l$, the DNN rate can outperform fixed-form kernels by a large margin. 
The above example demonstrates a similar superiority of the learned kernel.

\section{Results for IV Regression}\label{sec:iv}

We shall use the approximate first stage $\tilde\cI$ for IV regression, by plugging $\tilde\cI$ and $\cH$ to the kernelized estimators in \cite{dikkala_minimax_2020,wang2021quasibayesian}; see Algorithm~\ref{alg:main}. We analyze the resulted estimators in this section, while deferring implementation details, including hyperparameter selection, to 
Appendix~\ref{app:algo}. 

\begin{algorithm}[h]
\caption{Kernelized IV with learned instruments.}\label{alg:main}
\begin{algorithmic}[1]
\REQUIRE $\dataSI,\dataSII$; regression algorithm $\msf{Regress}$; second-stage kernel $k_x$; $m\in\mb{N}$
\FOR{$j \gets 1$ to $m$}
\STATE Sample %
$f^{(j)}\sim \mc{GP}(0, k_x)$
\STATE 
$\estiRaw{j}\gets \msf{Regress}(\{(\tilde z_i, f^{(j)}(\tilde x_i)): i\in [n_1]\})$  
\ENDFOR~{\color{gray}\COMMENT{the $m$ invocations of $\msf{Regress}$ may be replaced with a single vector-valued regression}}
\STATE Define 
$
\tilde k_z(z,z') := %
\frac{1}{m} \sum_{j=1}^{m} \esti{j}(z)\esti{j}(z'), 
$%
where $\esti{j} := \min\{\estiRaw{j}(\cdot), C\log m\}$. %
\STATE {\bf return} $\msf{KernelizedIV}(\dataSII,\tilde k_z,k_x)$~{\color{gray}\COMMENT{See \eqref{eq:minimax-estimator} below, or Appendix~\ref{app:algo} for the closed-form solution}}
\end{algorithmic}
\end{algorithm}

Both \cite{dikkala_minimax_2020} and the posterior mean estimator of \cite{wang2021quasibayesian} have the form
\begin{equation}\label{eq:minimax-estimator}
\arg\min_{f\in\cH}\ell_{n_2}(f)+\mu\|f\|_\cH^2 := \arg\min_{f\in\cH}\max_{g\in\tilde\cI} \frac{1}{n_2}\sum_{i=1}^{n_2}
(y_i-f(x_i)- \kappa g(z_i))g(z_i) - \lambda \|g\|_{\tilde\cI}^2 + \mu\|f\|_\cH^2.
\end{equation}
Their difference lies in the regularization scaling, which arises from the different assumptions about $f_0$ and $\cH$ (Assumption~\ref{ass:s2}). 
Thus, we analyze the resulted two estimators separately, in Section~\ref{sec:dikkala-regime} and Section~\ref{sec:qb-regime} below. 
In the setting of \cite{wang2021quasibayesian} we are also able to justify the use of likelihood-based model comparison and (quasi-)Bayesian model averaging (BMA).

\subsection{Estimation in the Kernel Scheme} \label{sec:dikkala-regime}
\cite{dikkala_minimax_2020} establish faster rate convergence of the point estimator under simple assumptions. We now provide corresponding results using our learned $\tilde\cI$, by plugging in the results in Section~\ref{sec:main}. 

\begin{proposition}[proof in Appendix~\ref{app:dikkala-proof}]\label{prop:dikkala-regime}
Assume Assumptions~\ref{ass:npiv}, \ref{ass:s2} (kernel scheme), \ref{ass:s1-rkhs} and \ref{ass:s1-oracle}. %
Let $\tilde\cI$ be defined by $\tilde k_z$ in Algorithm~\ref{alg:main}, and 
$\hat f_{n_2}$ be defined by \eqref{eq:minimax-estimator}, with $\kappa,\lambda,\nu$ set as in Appendix E.1. %
On the event defined in Theorem~\ref{prop:approx}, we have
\begin{equation}\label{eq:iv-rate-dikkala-setting}
\!\|E(\hat f_{n_2} - f_0)\|_2 = \tilde \cO_p\big(\big(\xi_{n_1}+n_2^{-\frac{b+1}{2(b+2)}}\big)(1+\|f_0\|_\cH^2)\big).\!\!
\end{equation}
\end{proposition}
Let us compare the result with \cite{dikkala_minimax_2020} in the setting of Example~\ref{ex:dnn-oracles}. Suppose $n_1=n_2$. \cite{dikkala_minimax_2020} establishes the rate of $\cO_p((\xi'_{n_2}+n_2^{-(b+1)/(b+2)})(1+\|f_0\|_\cH^2))$, where $\xi'_{n_2}$ %
is comparable with a first-stage regression rate established from local Rademacher analysis. For the DNN model in Example~\ref{ex:dnn-oracles}, we have 
$\xi'_n=\tilde\Theta(\xi_n)=\tilde\cO(\epsilon_{fea,n} + n^{-\beff/2(\beff+1)})$ in the general case,\footnote{With some abuse of notation, we also use $\tilde\cO$ to hide the infinitesimal deterioration of the polynomial order.}%
 or $\xi_n' = \tilde\cO(\epsilon_{fea,n} + n^{-(\beff+1)/2(\beff+2)})$ assuming additional regularity (Remark~\ref{rmk:xi-optimality-rkhs}). Thus, the two rates are equivalent if $\epsilon_{fea}$ is sufficiently large, meaning that 
the difficulty of feature learning cannot be ignored. Otherwise, \cite{dikkala_minimax_2020} may be better if $\beff < b+1$; this is a somewhat narrow range, as $\beff\ge \max\{2p-1,b\}$. Recall that our method is more appealing computationally: 
directly instantiating \cite{dikkala_minimax_2020} with DNNs requires solving a minimax problem similar to \eqref{eq:minimax-estimator}, while for our learned kernel \eqref{eq:minimax-estimator} can be evaluated in closed form. %

We can also compare \eqref{eq:iv-rate-dikkala-setting} with kernelized IV using a fixed-form first stage. In the above setting, its best rate is also provided by \cite{dikkala_minimax_2020}, and is dominated by the kernel regression error in Ex.~\ref{ex:dnn-oracles} which, as we discussed, can be much worse than $\xi_n$. \revision{
Our improved rate has been made possible by the fact that 
we are approximating a first-stage model with optimal adaptivity (Ex.~\ref{ex:informative-latents}), at a rate that is also adaptive to the informative latent structure (Ex.~\ref{ex:dnn-oracles}). 
} 

In summary, {\em our algorithm combines the best of both worlds}: it maintains the sharp guarantees of adaptive models, and the simplicity of kernel methods.

\subsection{Quasi-Bayesian Estimation and Uncertainty Quantification}\label{sec:qb-regime} %

Quasi-Bayesian analysis enables efficient uncertainty quantification for NPIV, without introducing extra risks of model misspecification \citep{kato_quasi-bayesian_2013,chernozhukov_mcmc_2003}. 
\cite{wang2021quasibayesian} studies a quasi-Bayesian posterior constructed from \eqref{eq:minimax-estimator} and a $\mc{GP}(0,k_x)$ prior. It is defined through the Radon-Nikodym derivative w.r.t.~the prior, 
$
\big(\nicefrac{d\Pi(\cdot\mid\dataSII)}{d\Pi}\big)(f) \propto e^{-n_2\ell_{n_2}(f)}
$. For kernel first-stage models, the quasi-posterior can be evaluated in closed form (App.~\ref{app:algo}). 
For general models, however, it is entirely unclear if approximate inference can be possible, since for any parameter $f$, evaluation of $\ell_n(f)$ involves solving a separate optimization problem. %
Our kernel learning algorithm enables the (indirect) use of such models.

Analysis of (quasi-)Bayesian procedures is more challenging, partly because of the weaker regularization. Thus, \cite{wang2021quasibayesian} introduced additional assumptions. Our analysis is further complicated by a different assumption on $\cI$, and approximation errors in $\tilde\cI$, which necessitate further assumptions. 
App.~\ref{app:qb-ass} discusses these assumptions in detail. 
For simplicity, we state the result in a ``rate-optimal'' case:\footnote{
Classical NPIV lower bounds continue to hold given full knowledge of $E$ \citep{chen_rate_2007}, so the rate $n_2^{-\nicefrac{b}{2(b+2p+1)}}$ is minimax optimal irrespective of $n_1$. In our setting, it is certainly desirable to improve the dependency on $n_1$, and our restriction is only employed to simplify proof. In simulations we find the choice of $n_1=n_2$ works well.
}
\begin{theorem}[posterior contraction; proof in App.~\ref{app:qb-proof}]\label{prop:qb-regime}
Assume Asms.~\ref{ass:npiv}, \ref{ass:s2} (GP scheme), \ref{ass:mildly-ill-posed}, \ref{ass:s1-rkhs}, \ref{ass:s1-oracle}, \ref{ass:emb-general}, \ref{ass:oracle-sup-norm-err}, \ref{ass:qb-1}, \ref{ass:qb-n}. 
Let $n_1$ be s.t.~$\xi_{n_1}^2\log n_2+n_1^{-(b+2p)/(b+2p+1)}\lesssim n_2^{-1}$, and $m\asymp n_1^{1/(b+2p+1)}$. 
Let $\Pi_{n_1}(\cdot\mid\dataSII)$ be defined in \eqref{eq:quasi-posterior} in appendix. 
Then, with $\dataSI$-probability $\to 1$, we have 
\begin{align*}
\EE_{\dataSII}\Pi_{n_1}(\{f:\|f-f_0\|_2\ge M\bar\epsilon_{n_2}\}\mid\dataSII) &\to 0, \\
\EE_{\dataSII}\Pi_{n_1}(\{f:\|E(f-f_0)\|_2\ge M\bar\delta_{n_2}\}\mid\dataSII) &\to 0,
\end{align*}
where $\bar\delta_{n_2} = 
\tilde \cO(n_2^{-(b+2p)/2(b+2p+1)})$, 
$\bar\epsilon_{n_2}=\tilde \cO(n_2^{-b/2(b+2p+1)})$. 
\end{theorem}
Theorem~\ref{prop:qb-regime} immediately implies Theorems 5, 6 in \cite{wang2021quasibayesian} for our $\tilde\cI$, with the extra logarithms, as their proofs do not involve %
the first stage. Those results establish Sobolev norm rates, %
and justify uncertainty quantification by lower bounding the magnitude of posterior spread.

In the nonparametric Bayes literature, contraction results like Theorem~\ref{prop:qb-regime} often lead to the justification of marginal likelihood-based model selection and averaging. This is also the case here. The key ingredient is the following marginal quasi-likelihood bound:
\begin{corollary}[proof in Appendix~\ref{app:proof-qb-mlh}]\label{lem:qb-mlh}
In the setting of Theorem~\ref{prop:qb-regime}, for some $C>0$ we have 
\begin{align*}
\PP_{\dataSII}\big(
C^{-1}n_2^{\frac{1}{b+2p+1}}\log^{-\frac{6}{b}}n_2
 &\le -\log \Pi_{n_1}(\dataSII)  %
 \le Cn_2^{\frac{1}{b+2p+1}} \log^2 n_2\big) \to 1.
\end{align*}
\end{corollary}
This result allows the comparison of a finite number of second-stage RKHSes. Of particular interest is the comparison between {\em power RKHSes} (Defn.~\ref{defn:power-spaces}), which often have intuitive interpretations: e.g., for a Mat\'ern RKHS $\cH$ and $\gamma\in (2/(b+1),1)$, the power RKHS $\cH^\gamma$ is equivalent to lower-order Mat\'ern RKHSes (\citep{fischer2020sobolev}; Ex.~\ref{ex:matern-gp}). We can verify that such $\cH^\gamma$ fulfills the assumptions about $\cH$. Thus, provided the other assumptions continue to hold, 
Corollary~\ref{lem:qb-mlh} will hold for all such $\cH^\gamma$, with $b+1$ replaced by $\gamma(b+1)$, 
showing the marginal likelihood has a different asymptotics. 
Consequently, it establishes asymptotically valid comparison between such models, and justifies the use of BMA.%

Analysis of more general settings requires additional effort: 
NPIV is an inverse problem, and we anticipate the subtleties of model selection in
nonparametric inverse problems. For example, analyses are usually restricted to the selection of $\gamma$ \citep{knapik2016bayes,szabo2015frequentist,jia2018posterior}, and 
the $\gamma>1$ case requires additional assumptions \citep{szabo2015frequentist}.\footnote{We do not cover it here for brevity, noting that it is well-understood in inverse problem settings \citep{knapik2016bayes,szabo2015frequentist}.}
In the IV setting, it should also be noted that valid model comparison 
requires a good approximation to $E|_\cH$, since otherwise the quasi-likelihood becomes less meaningful at any finite sample size. 
The same intuition applies to other model selection procedures \citep{bennett2019deep,muandet_dual_2020,singh_kernel_2020} based on the estimated violation of \eqref{eq:npiv}. 
When the approximation cannot be guaranteed, 
it could be preferable to stick to the prior knowledge and fix a conservative choice for $\cH$. %

\section{Related Work}\label{sec:related-work}

\paragraph{Multi-Task Learning} Our %
Example~\ref{ex:dnn-oracles-cont} can also be viewed as quantifying sample efficiency improvements in multi-task learning, if we view the GP prior draws %
as the labeling functions for a handful of diverse training tasks, which share the representation $\Phi$. 
This general idea is not new: starting from \cite{tripuraneni_theory_2020,du_few-shot_2021}, a line of recent work establishes similar results. 
Most related is \cite[Sec.~5]{du_few-shot_2021}, which assumes a fixed-dimensional linear model for $\bar g$, and an adaptive $\Phi$ with metric entropy bounds. We assume more general models for both components, and do not require different training tasks to have separate inputs. 
On the flip side, \cite{du_few-shot_2021} allows for non-iid training tasks. \cite[Sec.~6]{du_few-shot_2021} investigated infinite-dimensional %
$\bar g$, but established a slow rate. We are unaware of any work that established fast-rate convergence for infinite-dimensional top-level models, or used ML models as a black box. 
Both aspects may be interesting for multi-task learning, and are necessary for instrument learning. 

\vspace{-0.2em}
\paragraph{Causal Statistics} %
The double machine learning framework \citep{chernozhukov_doubledebiased_2018} also uses black-box ML models to estimate certain nuisance parameters in the model. 
While the operator $E$ can be viewed as a nuisance parameter, the structure of the NPIV problem is quite different: \cite[p.~8]{foster2019orthogonal} noted that 
it is \revision{very} unclear if such a view can be helpful for NPIV estimation; \revision{consistent with their remarks, we have also been unable to cast our problem into the double ML framework.}
Note that double ML has been applied to semiparametric estimation and inference for IV \citep{chernozhukov_doubledebiased_2018,syrgkanis2019machine,singh2020generalized,jung2021double}, which are orthogonal to our goal. %

It has long been known \citep{kelejian1971two,chamberlain_asymptotic_1987} that under a linear outcome model $f_0(\bx)=\theta^\top\bx$, using $\EE(\bx\mid \bz)$ as instrument leads to $\sqrt{n}$-consistent estimates. 
Our Section~\ref{sec:iv-to-klearn} can be viewed as an infinite-dimensional generalization of this observation.\footnote{As noted in \cite{singh_kernel_2020}, when $f_0\in\cH$ for some RKHS $\cH$, the first stage should model $\EE(f(\bx)\mid\bz)$ for all $f\in\cH$, as opposed to merely modeling $\EE(\bx\mid\bz)$. Note that \cite{singh_kernel_2020} did not study the optimal choice of the first stage.  
} 
Given high-dimensional instruments and a parametric outcome model, 
there is a large body of literature on efficient inference; see \cite{singh2020machine} for a review. 
As we move to nonparametric models, we focus on estimation which becomes much more challenging, 
in the spirit of \cite{foster2019orthogonal}. 
Still, we have provided qualitative characterization for uncertainty estimates in Section~\ref{sec:qb-regime}. 

For the use of ML for nonlinear IV, 
\cite{hartford2017deep} studied a heuristic application of NNs. 
We discussed the minimax formulation in introduction. 
\cite{liao_provably_2020,wang2021quasibayesian} justified the use of NNs with the respective neural tangent kernels (NTKs) which, like other fixed-form kernels, cannot adapt to the informative latent structure \citep{wei2019regularization,ghorbani2019limitations}. 
\cite{zhang_maximum_2020,chen_efficient_2021} investigated the combination of an NN-based second stage and a linear first stage, which could be useful in complementary scenarios. 
\cite{xu_learning_2020} considered feature learning in both stages, but only established a slow rate; as the authors noted, it is also unclear if their algorithm reliably minimizes the empirical risk. 

For model selection in the setting of Section~\ref{sec:qb-regime}, \cite{chen2021adaptive} prove the validity of bootstrap-based selection for the sieve estimator \citep{newey2003instrumental}. 
\cite{zhang_maximum_2020,hsu2019bayesian} investigate the use of marginal likelihood for two different kernel-based IV estimators: 
\cite{hsu2019bayesian} establish a crude $-\nicefrac{1}{4}\log n$ upper bound for the log marginal likelihood,
and \cite{zhang_maximum_2020} connect it to the empirical leave-one-out validation error. %
Neither result fully justifies model selection as our Corollary~\ref{lem:qb-mlh}. 
For kernelized IV models, 
\cite{muandet_dual_2020,singh_kernel_2020,zhang_instrument_2021} proposed 
validation statistics for comparing {\em a finite number of} first stage models.

\begin{figure*}[bt]
    \centering
    \includegraphics[width=0.98\linewidth,clip,trim={0.25cm 0cm 0.35cm 0.25cm}]{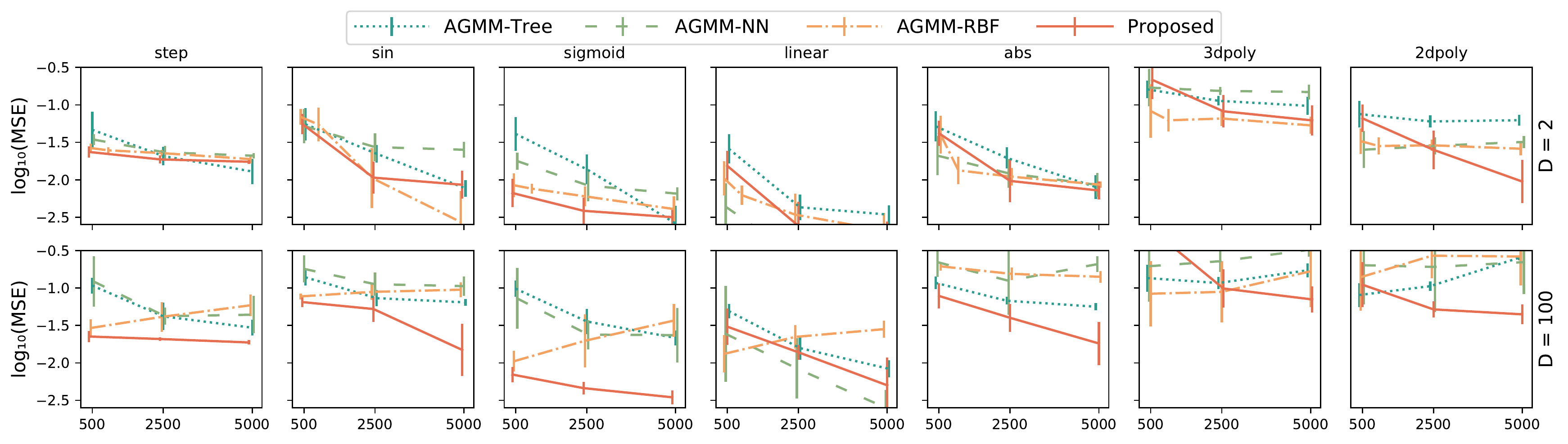}
    \caption{Predictive performance: test MSE vs sample size $n_1=n_2$ for all method, and $D\in\{2,100\}$. Full results are in App.~\ref{sec:exp-main}. %
    }\label{fig:predictive-main}
\end{figure*}

\section{Simulation Study}\label{sec:experiments}

Our main simulation setup is adapted from \cite{bennett2019deep,dikkala_minimax_2020}; Appendix~\ref{app:exp-exo} presents additional experiment on the demand dataset \cite{hartford2017deep,xu_learning_2020}. In \cite{bennett2019deep,dikkala_minimax_2020}, the observed $\bz,\bx,\by$ are generated by
\begin{align*}
\bar \bz &\sim \mrm{Unif}[-3, 3]^{\lfloor\frac{D}{2}\rfloor}, ~  \bz=h(\bar\bz), ~\bu \sim\cN(0, 1), ~
\bx := \bar \bz_1 + \bu + \be_x, ~
\by := ({f_0(\bx) + \bu + \be_y - \mu})/\sigma, 
\end{align*}
where $\bu$ is the confounder, 
$\be_x,\be_y\sim \cN(0, 0.1^2)$ are independent noise, and the constants $\mu,\sigma$ standardize $\by$. 
We consider three choices for $h$: {\bf(i)} $D=2$, $h$ is the identity function; this recovers the setup in previous work, and quantifies the hardness of the NPIV problem given true instruments. {\bf(ii)} $\mrm{dim}\:\bz=D\in\{40,100\}$, $h$ is a three-layer DNN; this simulates a feature learning scenario, and ensures the observation has a low signal-to-noise ratio ($O(1/D)$). {\bf (iii)} $h$ maps $\bar\bz_1$ to a MNIST \cite{lecun1998mnist} or CIFAR-10 \cite{krizhevsky2009learning} image with matching label; the MNIST setting also appeared in previous work. 

\begin{table}[bt]\centering\small
\caption{\revision{Runtime results for all methods in the predictive experiment, for $N=2500,D=100$.}}\label{tbl:runtime}
\begin{tabular}{ccccc}\toprule 
Method & AGMM-Tree & AGMM-NN & AGMM-RBF & Proposed \\ 
\midrule 
Runtime / s & $1374\pm 418$ & $303\pm 16$ & $6.7\pm 0.1$ & $25.9\pm 5.6$ \\ 
\bottomrule
\end{tabular}
\end{table}

\begin{table}[tb]\centering\small
\caption{%
Test MSE, radius and estimated coverage rate of the $90\%$ $L_2$ credible ball (CB), and the average coverage of pointwise $90\%$ credible interval (CI), for $f_0\sim\mc{GP},D=100$. %
For the CB coverage rate estimate, we report its $95\%$ Wilson score interval \citep{wilson27}. 
Full results are in App.~\ref{app:exp-uq}. 
}\label{tbl:uq-gprand-main}
    \begin{tabular}{cccccc}\toprule
Method & 
$n_1=n_2$ & Test MSE & 90\% CB.~Rad. & 90\% CB.~Cvg. & 90\% CI.~Cvg. \\ 
\midrule  %
\multirow{3}{*}{Proposed} 
 & $500$ & $.097$ {\tiny $\pm .065$} & $.201$ {\tiny $\pm .025$} & $.923$ {\tiny $[.888, .948]$} & $.915$ {\tiny $\pm .123$}\\
 & $2500$ & $.035$ {\tiny $\pm .024$} & $.074$ {\tiny $\pm .008$} & $.917$ {\tiny $[.880, .943]$} & $.908$ {\tiny $\pm .127$}\\
 & $5000$ & $.024$ {\tiny $\pm .016$} & $.049$ {\tiny $\pm .004$} & $.920$ {\tiny $[.884, .946]$} & $.905$ {\tiny $\pm .134$}\\
\midrule \multirow{3}{*}{RBF} 
 & $500$ & $.431$ {\tiny $\pm .192$} & $.240$ {\tiny $\pm .036$} & $.187$ {\tiny $[.147, .235]$} & $.640$ {\tiny $\pm .191$}\\
 & $2500$ & $.176$ {\tiny $\pm .089$} & $.175$ {\tiny $\pm .023$} & $.517$ {\tiny $[.460, .573]$} & $.822$ {\tiny $\pm .136$}\\
 & $5000$ & $.126$ {\tiny $\pm .072$} & $.156$ {\tiny $\pm .019$} & $.660$ {\tiny $[.605, .711]$} & $.855$ {\tiny $\pm .143$}\\
\bottomrule
\end{tabular}
\end{table}

We consider two choices for $f_0$: {\bf (i)} a widely used collection of functions (e.g., $\sin,\mrm{abs}$) in \cite{bennett2019deep}. {\bf (ii)} $f_0\sim\mc{GP}(0,k_x)$. (ii) ensures the correct specification of $\cH$ and allows us to focus on the first stage. %

We use a DNN as the black-box learner, and  a RBF kernel for $\cH$, with bandwidth determined by marginal likelihood \eqref{eq:s2-log-qlh}.
We set $N_1=N_2\in\{500,2500,5000\}$. We defer 
setup details and full results to Appendix~\ref{app:simulations}, and summarize the findings below:

\paragraph{Hyperparameter Selection (App.~\ref{sec:exp-s1-randgp})} We first study 
hyperparameter selection in instrument learning. We set $f_0\sim\mc{GP}$, $D\in\{2,40,100\}$. We find our %
validation statistics \eqref{eq:s1-val-stats} always correlates with the counterfactual MSE $\|\hat f_n-f_0\|_2^2$, and that across a large hyperparameter space, trained DNNs %
always outperform first-stage models based on RBF kernels, or randomly initialized DNNs.
    
\paragraph{Predictive Performance (App.~\ref{sec:exp-main})} 
For $h$ defined as in (i-ii), we compare our algorithm with \cite[AGMM]{dikkala_minimax_2020}, instantiated with kernel, tree and NN models. As shown in \cite{dikkala_minimax_2020}, the baselines have competitive performance on this setup; 
the latter two models also enjoy adaptivity guarantees.
A representative subset of results are plotted in Fig.~\ref{fig:predictive-main}: our method has stable performance as we move to high dimensions, demonstrating excellent adaptivity. In contrast, fixed-form kernels fail to identify the informative features. AGMM-tree and AGMM-NN also have deteriorated performance as $D$ increases, despite their theoretical guarantees, presumably due to the challenges in optimization. 
\revision{
\cref{tbl:runtime}
reports the run time of all methods in this experiment. As we can see, our method is more efficient than both adaptive baselines. 
}

For image-based $h$, we compare with AGMM-NN and \cite{zhang_maximum_2020}, which report the best results in the MNIST setting. Our method outperforms both baselines. 

\paragraph{Uncertainty Quantification (App.~\ref{app:exp-uq})} Table~\ref{tbl:uq-gprand-main} presents a subset of results for $f_0\sim\mc{GP}$. 
Comparing with a fixed-form RBF first stage, 
our method produces sharper credible intervals, which also have better coverage. 
For $f_0$ specified as in \cite{bennett2019deep}, we experiment with BMA over a grid of RBF kernels, and present visualizations in Appendix~\ref{app:exp-uq}. We find that 
when the model is more correctly specified, BMA 
produces conservative uncertainty estimates which are nonetheless informative. However, when all models are severely misspecified (e.g., when $f_0$ is a step function), we cannot expect model-based uncertainty estimates to have ideal coverage.

\paragraph{Exogenous Covariates (App.~\ref{app:exp-exo})} \revision{We evaluate the extended algorithm in \cref{app:exo-algo} on the demand dataset \citep{hartford2017deep}, which is a widely used simulation design with high-dimensional exogenous covariates. As shown in the appendix, our extended algorithm has competitive performance.}

\section*{Acknowledgements}
This work was supported by NSFC Projects (Nos. 62061136001, 62076145, 62076147, U19B2034, U1811461, U19A2081, 61972224), Beijing NSF Project (No. JQ19016), BNRist (BNR2022RC01006), Tsinghua Institute for Guo Qiang, and the High Performance Computing
Center, Tsinghua University. J.Z is also supported by the XPlorer Prize.

\renewcommand{\>}{\rabak}
\bibliographystyle{IEEEtranN}
\bibliography{bib}

\newpage
\appendix

\addcontentsline{toc}{section}{Appendix} %
\part{Appendix} %
\parttoc %

\renewcommand{\>}{\rangle}
\section{Background and Technical Lemmas}\label{app:technical}

\paragraph{Kernels}
The two following lemma applies to our $\cH$ satisfying Assumption~\ref{ass:s2}, but we will also apply them to the RKHS $\bar\cI$ defined in Section~\ref{sec:main}.\footnote{$\cI$ may not necessarily satisfy the requirement in Lemma~\ref{lem:mercer} (iv), but a weaker version always holds; see Appendix~\ref{app:proof-sec-iv-to-kl}.} In the latter case, $x,\cX$ should be replaced by $\bar z,\bar\cZ$; and as discussed in the main text, the results will transfer to the RKHS $\cI = \{\bar g\circ\Phi:\bar g\in\cI\}$. 
\begin{lemma}[Mercer's representation]\label{lem:mercer}
Let $\cH$ be any RKHS with kernel $k_x$ s.t.~$
\int P(dx) k_x(x,x) < \infty.
$
Then 
\begin{enumerate}[leftmargin=*,label=\roman*.]
   \item $\cH$ can be embedded into $L_2(P(dx))$, and 
   the natural inclusion operator $\iota_x:\cH\to L_2(P(dx))$ and $\iota_x^\top$ are Hilbert-Schmidt; the map $T_x: f\mapsto \int P(dx) k_x(x,\cdot) f(x)$ defines a positive, self-adjoint and trace-class operator; $T_x = \iota_x\iota_x^\top$. 
   \item $T_x$ has the decomposition $$
T_x f = \sum_{i\in I} \mu_i \<\bar e_i, f\>_2 \bar e_i,
$$
where the index set $I\subset \mb{N}$ is at most countable, and $\{\bar e_i\}$ is an orthonormal system in $L_2(P(dx))$. 
\item There exists an orthogonal system $\{e_i: i\in I\}$ of $\cH$ s.t.~$[e_i]_\sim = \sqrt{\lambda_i} \bar e_i$. 
\item If $k_x$ is additionally bounded and continuous, 
$\{e_i: i\in I\}$ will %
define a Mercer's representation whose convergence is absolute and uniform. %
\end{enumerate}
\end{lemma}
\begin{proof}
\cite[Lemma 2.3, 2.2 (for i), %
       2.12 (for ii-iii), Corollary 3.5 %
       (for iv)]{steinwart_mercers_2012}.
\end{proof}

The following material on power spaces are adapted from \cite{wang2021quasibayesian}, which collected them from \cite{steinwart_mercers_2012,fischer2020sobolev}. 
\begin{definition}[power space, embedding property]\label{defn:power-spaces}
Let $\cH$ be an RKHS with Mercer's representation $\{(\lambda_i,\varphi_i):i\in\mb{N}\}$. For $\gamma\ge 1$,  the \emph{power space} $[\cH]^\gamma\subset L_2(P(dx))$ is defined as 
$$
[\cH]^\gamma = \left\{[f]_\sim := \sum_{i=1}^\infty a_i [\varphi_i]_\sim: \|[f]_\sim\|_{[\cH]^\gamma}^2 := \sum_{i=1}^\infty \lambda_i^{-\gamma}a_i^2<\infty\right\}.
$$
We say $\cH$ satisfies an \emph{embedding property} with order $\gamma$ if 
$[\cH]^\gamma$ is continuously embedded into $L_\infty(P(dx))$, denoted as 
\begin{equation}\label{eq:emb}
[\cH]^\gamma \hookrightarrow L_\infty(P(dx)). \tag{EMB}
\end{equation}
\end{definition}
Clearly, $\cI$ and $\bar\cI$ will satisfy \eqref{eq:emb} with the same order. 

\begin{lemma}\label{lem:emb}
Under \eqref{eq:emb}, (i) the function space
$$
\cH^\gamma = \left\{f := \sum_{i=1}^\infty a_i \varphi_i: (\lambda_i^{-\gamma/2}a_i)_{i\in\mb{N}}\in\ell_2(\mb{N})\right\},
$$
with a similarly defined norm, will be an RKHS (a ``power RKHS'') with a bounded kernel; (ii) the kernel has a pointwise convergent Mercer representation $\{(\lambda_i^{\gamma},\varphi_i):i\in\mb{N}\}$. %
(iii) We have the interpolation inequality 
\begin{equation}\label{eq:interpolation-general}
\|f\|_\infty \lesssim \|f\|_{\cH}^\gamma \|f\|_2^{1-\gamma},\quad\forall f\in\cH.
\end{equation}
\end{lemma}
\begin{proof}
\cite[Theorem 5.5 (for i)]{steinwart_convergence_2019}, \cite[Theorem~3.1 (for ii), Thm.~5.3 (for iii)]{steinwart_mercers_2012}.
\end{proof}

The embedding property is stronger when $\gamma$ can be chosen to be smaller. The following example shows that for Mat\'ern kernels, we can choose the best posible $\gamma$:\todo{discuss the assumptions in the text}
\begin{example}[Mat\'ern kernels and Sobolev regularity]\label{ex:matern}
Let $\cX$ be a bounded open set in $\RR^d$ with a smooth boundary, 
$P(dx)$ have %
its Lebesgue density bounded from both sides, 
and $\cH$ be the Mat\'ern-$\alpha$ RKHS. Then 
\begin{enumerate}[leftmargin=*,label=\roman*.]
  \item $\cH$ is norm-equivalent to the $L_2$-Sobolev space $W^{\alpha+\frac{d}{2}, 2}$ \citep[Example 2.6]{kanagawa_2018_gaussian}.
  \item Its Mercer eigenvalues decay at $\lambda_i \asymp i^{-(1+\frac{2\alpha}{d})}$, and it satisfies \eqref{eq:emb} for all $\gamma > (1+\frac{2\alpha}{d})^{-1}$ \citep[Section 4]{fischer2020sobolev}. 
\end{enumerate}
Among kernels with the same eigendecay, this is the best possible $\gamma$ \citep{steinwart2009optimal}. 
\end{example}

We now provide some intuition on the ``GP scheme'' approximation condition, Assumption~\ref{ass:s2} \ref{it:gp-scheme}:
\begin{example}\label{ex:f0}
For any $\cH$ satisfying the eigendecay assumption, simple calculation shows that any $f_0\in [\cH]^{b/(b+1)}$ satisfies Assumption~\ref{ass:s2} \ref{it:gp-scheme} \citep[Lemma 23]{wang2021quasibayesian}. If we are further in the setting of Example~\ref{ex:matern}, $f_0$ will satisfy Assumption~\ref{ass:s2}~\ref{it:gp-scheme} if $f_0\in W^{\alpha,2}$ \citep[Chapter 7]{adams2003sobolev}. 
\end{example}

\begin{lemma}\label{lem:trunc-err}
Let $\cI$ satisfy Assumption~\ref{ass:s1-rkhs}. Then for all $g\in\cI, m\in\mb{N}$, 
\begin{equation}\label{eq:trunc-err}
\|\Proj{m}{g}\|_\cI \le \|g\|_\cI, ~~
\|g - \Proj{m}{g}\|_2^2 \lesssim \|g\|_\cI^2 m^{-(\beff+1)},
\end{equation}
where the constant hidden in $\lesssim$ only depends on $\cI$.
\end{lemma}
\begin{proof}\footnote{Similar result has been stated in \cite{wang2021quasibayesian}. We restate the proof to drop some unnecessary assumptions.}
Let $\{(\lambda_i,\varphi_i)\}$ be the Mercer eigendecomposition, so that $\{\sqrt{\lambda_i}\varphi_i\}$ constintute a countable ONB for $\cI$ (Lemma~\ref{lem:mercer}). 
Thus, the RKHS norm bound holds, and 
\begin{align*}
\|g - \Proj{m}{g}\|_2^2 = \sum_{j=m+1}^{\infty} \<f, \sqrt{\lambda_i} \varphi_i\>_\cI^2 \|\sqrt{\lambda_i}\varphi_i\|_2^2 \le 
\|g-\Proj{m}{g}\|_\cI^2\cdot \lambda_m\|\varphi_m\|_2^2 \lesssim
\|g\|_\cI^2 m^{-(\beff+1)}.
\end{align*}
\end{proof}

\paragraph{Gaussian Measure and Gaussian Process}

\begin{definition}[Gaussian measure, \cite{eldredge_analysis_2016}]\label{defn:gaussian-measure}
Let $(\mb{B},\|\cdot\|)$ be a Banach space, $W\sim \mu$ be a Borel measurable map. $\mu$ is a \emph{Gaussian measure} if for any $b^*\in\mb{B}^*$, the pushforward measure $b^*_{\#\mu}$ is normally distributed.

A Gaussian measure defines a bilinear form on $\mb{B}^*$: $q(f,g) = \EE f(W) g(W)$. When $\mb{B}$ is additionally a Hilbert space, $q$ will correspond to a bilinear form on $\mb{B}$, denoted as $\Lambda$. We then introduce the notation $W\sim N(0,\Lambda)$, meaning that for all $l\in \mb{B}$, the random variable $\<l,W\>_H\sim\cN(0, \<l, \Lambda l\>_{\mb{B}})$.
\end{definition}

\begin{lemma}[Borell-TIS, \cite{ghosal2017fundamentals}, Proposition I.8]\label{lem:borell-tis}
Let $W$ be any mean-zero Gaussian process defined on a Banach space $\mb B$, and $\|\cdot\|$ denote any Banach norm. If $\|W\|<\infty$ a.s., it will hold that 
$$
\PP(|\|W\|-E\|W\|| > x) \le 2e^{-x^2/(2\sigma^2(W))}\quad \forall x>0,
$$
where $\sigma(W) := \sup_{b^* \in \mb B^* : \| b^* \| = 1} \sqrt{\EE_W [ b^*(W)^2]}$ is less than the median of $\|W\|$. 
\end{lemma}

In the following, $C^\beta$ denotes the H\"older space of order $\beta$ on $\cX$.
\begin{example}[Mat\'ern processes]\label{ex:matern-gp}
Let $k_x$ be a Mat\'ern-$\alpha$ kernel, $\cX$ be as in Example~\ref{ex:matern}, $\underline\alpha<\alpha$ be any positive number. Then 
there exists a modification of $\mc{GP}(0, k_x)$ which always has finite $C^{\underline\alpha}$ norm \citep[p.~2104]{van_der_vaart_information_2011}.
\end{example}

\paragraph{Miscellaneous Results}
\begin{definition}[entropy number]\label{defn:entropy-number}
Let $H,J$ be Banach spaces, $A\subset J$ be a bounded set. For all $i\in\mb{N}$, the $i$-th entropy number is defined as 
$$
e_i(A, J) = \inf \{\epsilon>0: N(A, \|\cdot\|_J, \epsilon) \le 2^i \}, 
$$
where $N$ denotes the covering number. Further, let $T: H\to J$ be any bounded linear operator. Then the $i$-th entropy number of the operator $T$ is defined using the image of the unit-norm ball $H_1$ under $T$:
$$
e_i(T) = e_i(T(H_1), J).
$$
\end{definition}

The following singular value inequality will be frequently used, both to $s_j(AB)$ and the $j$-th largest eigenvalue $\lambda_j(ABB^\top A^\top) = s_j(AB)^2$:
\begin{lemma}[\citealp{bhatia2013matrix}, Problem III.6.2]
Let $A,B$ be any two operators, $\|\cdot\|$ denote the operator norm, and $s_j$ denote the $j$-th largest singular value. 
Then 
\begin{equation}\label{eq:trivial-sv-ineq}
s_j(AB) \le \min\{\|B\|s_j(A), \|A\|s_j(B)\}.
\end{equation}
\end{lemma}

\section{Deferred Proofs: Function Spaces}\label{app:proof-sec-iv-to-kl}

\newcommand{\Ereg}{E_r}

\begin{remark}[versions of $E$]\label{rmk:regular-cond-exp}
Conditional expectations are only defined up to $P(dz)$-null sets. 
As $\bx$ is supported on a bounded open subset of $\RR^{d_x}$, there exists a regular conditional probability $\mu$, which defines a version of conditional expectation \citep{kallenberg1997foundations}
$$
\text{for all square integrable } f,~~
\EE(f(\bx)\mid \bz) = \int \mu(dx,\bz) f(x) ~ a.s.~[P(dz)].
$$
Throughout the work, we work with the above version of conditional expectation.\footnote{The choice of $\mu$ is only unique up to a $P(dz)$-null set; we fix an arbitrary version to define $E_r$. What matters to us is the fact that $E_r$ is defined with a regular conditional probability, so that \eqref{eq:E-bounded-sup-norm} always holds.
} It represents a linear operator between spaces of functions, denoted as
$$
(\Ereg f)(z) := \int \mu(dx, z) f(x). 
$$
As $\mu(\cdot,z)$ is a probability measure for all $z\in\cZ$, we now have 
\begin{equation}\label{eq:E-bounded-sup-norm}
\|\Ereg f\|_\infty \le \|f\|_\infty.
\end{equation}
\end{remark}

Our focus in this work is in estimation; thus, in the main text and other sections of the appendix, we will abuse notation, and use $E$ to also refer to $\Ereg$ for readability. In this section, however, we make the distinction clear for full clarity. 

The following claim is well-known. Note that by requiring $[f]_\sim$ to be a Gaussian measure in $L_2(P(dx))$, we are requiring our Gaussian process to possess a possibly richer $\sigma$-algebra than e.g., the version returned by the Kolmogorov extension theorem. However, they will induce the same marginal distributions for $f(X)$, and the resulted estimators. See Definition~\ref{defn:gaussian-measure}, and e.g.~\citet{van2008reproducing,stuart2010inverse} for an accessible review of related issues.
\begin{claim}\label{lem:gp-gm}
Let $\cX$ be a bounded subset of $\RR^d$, $k_x$ be a reproducing kernel on $\cX$, s.t.~$\EE_{P(dx)} k_x(x,x)<\infty$. 
Let $[f]_\sim \sim N(0, C)$ be a Gaussian random element on $L_2(P(dx))$ (Definition~\ref{defn:gaussian-measure}), s.t.~the marginal distributions of $f$ match 
$\mc{GP}(0,k_x)$ s.t.~$[f]_\sim$. Then $C$ equals the integral operator of $k_x$. 
\end{claim}
\begin{proof}
\citet[p.~538]{stuart2010inverse}. 
\end{proof}
Note that for our $k_x$, the integral operator $T_x=\iota_x\iota_x^\top$ (Lemma~\ref{lem:mercer}). 

\begin{proof}[Proof for Lemma~\ref{lem:iv-to-kl-1}]
By definition of Gaussian measure (\ref{defn:gaussian-measure}) and Claim~\ref{lem:gp-gm}, we have $E[f]_\sim\sim N(0, E T_x E^\top)$, so it suffices to construct a $k_z$ with integral operator $E T_x E^\top$. 

By Lemma~\ref{lem:mercer} (i) and
boundedness of $E: L_2(P(dx))\to L_2(P(dz))$, the operator $E\iota_x: \cH\to L_2(P(dz))$ is Hilbert-Schmidt.
Thus, the operator $E\iota_x\iota_x^\top E^\top$ is trace-class,
and we can invoke \citet[Theorem 3.10]{steinwart_mercers_2012}, which shows the existence of an RKHS $\cI$ with a measurable reproducing kernel $k_0$, such that
\begin{enumerate}[leftmargin=*,label=\roman*.]
   \item The integral operator of $k_0$ equals $E\iota_x\iota_x^\top E^\top$.
   \item For appropriate choices of $e^{z}_i$ s.t.~$[e^z_i]_\sim$ diagonalizes $E\iota_x\iota_x^\top E^\top$, $\{\sqrt{\lambda_i^z}e_i^z: i\in\mb{N},\lambda_i^z>0\}$ form an ONB of $\cI$.
   \item $k_0(z,z') = \sum_{i\in I}\lambda^z_i e^z_i(z)e^z_i(z'), ~~\EE_{P(dz)}\: k_0(z,z)<\infty$.
\end{enumerate}
Combining (i, iii) and Claim~\ref{lem:gp-gm} above completes the proof.
\end{proof}

Observe the statement (ii) above shows that, the RKHS $\cI$ satisfies
\begin{equation}\label{eq:i-space-defn}
\cI = \Big\{\sum_{i\in I} b_i \sqrt{\lambda_i^z}e_i^z: ~ (b_i)\in \ell_2(I)\Big\},\quad%
\Big\|\sum_{i\in I} b_i \sqrt{\lambda_i^z}e_i^z\Big\|_\cI = \|(b_i)\|_{\ell_2(I)},
\end{equation}
where $I\subset\mb{N}$ denotes an index set which is at most countable.

\begin{proof}[Proof for Lemma~\ref{lem:optimal-I}]
The operator $\iota_x^\top E^\top E\iota_x$ is also trace-class.
Let $\{(\lambda^z_i,e_i^z): i\in I\}$ be defined as in the proof of Lemma~\ref{lem:iv-to-kl-1}, so that $\{e_i^x := (\lambda^z_i)^{-1/2}\iota_x^\top E^\top [e_i^z]_\sim: i\in I\}\subset \cH$ diagonalizes $\iota_x^\top E^\top E\iota_x$.
Then for any $f\in\cH$, it holds that
\begin{align*}
\infty &> \sum_{i \in I} \<f,e_i^x \>_\cH^2 =
 \sum_{i\in I} (\lambda_i^z)^{-2}\<f, \iota_x^\top E^\top E\iota_x e_i^x\>_\cH^2
=
\sum_{i\in I} (\lambda_i^z)^{-1}\<E\iota_x f,  [e_i^z]_\sim\>_2^2.
\end{align*}
Comparing with \eqref{eq:i-space-defn}, we can see that
for any
function $g$ s.t.~$[g]_\sim = E\iota_x f$, the RHS shows that $g\in\cI$, and equals $\|g\|_\cI^2$.

Conversely, for any $g\in\cI$, the sequence $\{(\lambda_i^z)^{-1/2} \<[g]_\sim, [e_i^z]_\sim\>_2:i\in I\}$ must be in $\ell_2$. Additionally, $\{e_i^x\}_{i \in I}$ is an ONS of $\cH$, so the limit
$$
\sum_{i\in I} (\lambda_i^z)^{-1/2} \<[g]_\sim, [e_i^z]_\sim\>_2 e_i^x
=: f
$$
must exists in $\cH$, and $\| f \|_\cH = \|g \|_\cI$ holds by the above display.
Similarly, it holds that
\[ E [f]_\sim
= \sum_{i \in I}(\lambda_i^z)^{-1/2}\<[g]_\sim, [e_i^z]_\sim\>_2 E\iota_xe_i^x
= \sum_{i \in I}\<[g]_\sim, [e_i^z]_\sim\>_2  [e_i^z]_\sim
 = [g]_\sim. \]
This completes the proof.
\end{proof}

We prove the following claim, made in Example~\ref{ex:informative-latents}. 
\begin{claim}\label{claim:ex-fs}
\itI Let $k_z,k_z'$ satisfy Lemma~\ref{lem:iv-to-kl-1}. Then $\EE_{\bz,\bz'\sim P} (k_z(\bz,\bz')-k_z'(\bz,\bz'))^2=0$. 
\itII Let $\bar k_z, k_z$ be defined as in the example. Then the non-zero Mercer eigenvalues of $\bar k_z, k_z$ coincide.
\end{claim}
\begin{proof}
\itI Both kernels are $L_2(P(dz)\otimes P(dz))$ bounded; the claim thus follows from the isometry between $L_2$-bounded kernels and their (Hilbert-Schmidt) integral operators, and the fact that both kernels have the same integral operator. 
\itII Let $\{(\lambda_i,[e_i^{\bar z}]_\sim)\}$ denote the eigendecomposition of the integral operator $T_{\bar k}$. Following the definitions we find that $\{[e_i^{\bar z}\circ\Phi]_\sim\}$ are eigenfunctions of $T_z$, with the same eigenvalues; and it is not possible for $T_z$ to have additional non-zero eigenpairs. 
\end{proof}

\subsection{Further Regularity Properties of $\cI$}\label{app:regularity}

\paragraph{Eigendecay} Recall the proof of Lemma~\ref{lem:iv-to-kl-1} invokes \cite{steinwart_mercers_2012}, and leads to the following results:
\begin{claim}
The kernel $k_z$ is measurable, satisfies $\EE_{P(dz)} k_z(z,z)<\infty$, and has integral operator equal to $E\iota_x\iota_x^\top E^\top$. 
\end{claim}
The last statement bounds the decay of the Mercer eigenvalues: using Assumption~\ref{ass:s2},~\ref{ass:mildly-ill-posed}, and \eqref{eq:trivial-sv-ineq}, 
we immediately find $\lambda^z_i = \lambda_i(E\iota_x\iota_x^\top E^\top) \lesssim i^{-\max\{b+1,2p\}}.$ We further have the following:
\begin{claim}
Under Assumption~\ref{ass:qb-1}, it holds that $\lambda^z_i \lesssim i^{-(b+2p+1)}$.
\end{claim}
\begin{proof}
Let $\{[\psi_i]_\sim:i\in\mb{N}\}$ be the Mercer eigenfunctions of $\cH$, s.t.~$\{\sqrt{\lambda_i(\iota_x\iota_x^\top)}\psi_i:i\in\mb{N}\}$ form an ONB of $\cH$ (Lemma~\ref{lem:mercer}). 
By the min-max theorem for eigenvalues \citep[e.g.,][Theorem 3.2.4]{simon2015operator}, 
\begin{align*}
\lambda_i(E\iota_x\iota_x^\top E^\top) = 
\lambda_i(\iota_x^\top E^\top E\iota_x) 
&= \inf_{V\subset \cH,\mrm{dim} V=i-1}\sup_{e\perp V,\|e\|_\cH=1} e^\top \iota_x^\top E^\top E\iota_x e
\\ & 
\le \|E\iota_x (\sqrt{\lambda_i(\iota_x\iota_x^\top)}\psi_i)\|_2^2 \lesssim i^{-(b+1)} \|E[\psi_i]_\sim\|_2^2 \lesssim i^{-(b+2p+1)}.
\end{align*}
The last inequality follows by the link condition.
\end{proof}

\paragraph{Bounded Kernel and GP Prior Draws} To establish boundedness of the kernel $k_z$ and (a version of) $g\sim\mc{GP}(0,k_z)$, we need the following additional assumptions:
\begin{enumerate}[label=(A.\Roman*), leftmargin=*]
    \item\label{it:A1} A version of $\mc{GP}(0,k_x)$ takes value on a separable subspace $\mb{B}\subset L_\infty(P(dx))$.
    \item\label{it:A2} The operator $E_r$ (Remark~\ref{rmk:regular-cond-exp}) maps $\mb{B}$ to a space of continuous, bounded functions on $\cZ$. 
\end{enumerate}
Note \ref{it:A1} will hold given 
our Assumption~\ref{ass:qb-n}, in which case we can take the subspace as a power RKHS $\cH^\alpha$; see \citet{steinwart_convergence_2019}. As discussed around that assumption, %
for some valid choice of $\alpha$, $\cH^\alpha$ should match the regularity of the second-stage RKHS assumed in previous work on kernelized IV, so such a boundedness assumption matches previous work. 

\ref{it:A2} will hold if we assume $E_r$ maps $f\in\cH^\alpha$ to another RKHS over $\cZ$, with a continuous, bounded reproducing kernel. Such an RKHS is often assumed in previous work; note that it does not have have the optimal regularity. 
Alternatively, the assumption can also be fulfilled by the assumption that $P(dx\times dz)$ have a continuous Lebesgue density and the marginal density $p(z)$ does not vanish. %

We now establish the following lemma. It shows the $\cI$ defined in Sec.~\ref{sec:iv-to-klearn} fulfills the conditions in Asm.~\ref{ass:s1-rkhs}. It also shows that by defining $k_z$ with $E_r$ as below, we can remove the null set indeterminancies in Sec.~\ref{sec:iv-to-klearn}: all possible $k_z$'s have the same integral operator (\ref{it:ikf-existence}) and are thus equivalent up to null sets (Claim~\ref{claim:ex-fs} {\em(i)}), yet they are shown to be continuous (\ref{it:ikf-continuity}).
\begin{lemma}[bounded kernel and GP draws]\label{lem:iv-to-kl-further}
Let $f$ be a Gaussian measure  with marginal distributions matching $\mc{GP}(0,k_x)$. 
Let $E_r$ be defined in Remark~\ref{rmk:regular-cond-exp}. Then under \ref{it:A1}, 
\begin{enumerate}[leftmargin=*,label=\roman*.]
    \item\label{it:ikf-existence} There exists a kernel $k_z$, s.t.~$E_r f\sim\mc{GP}(0,k_z)$, with integral operator equaling $E T_x E^\top$. 
    \item\label{it:ikf-bounded} $k_z$ is bounded, and there exists a version of $g\sim\mc{GP}(0,k_z)$ which always has a finite sup norm.
    \item\label{it:ikf-continuity} If additionally \ref{it:A2} holds, $k_z$ will be continuous, and its RKHS $\cI$ will also satisfy Lemma~\ref{lem:optimal-I}. 
\end{enumerate}
\end{lemma}
\begin{proof}
\ref{it:ikf-bounded}: Observe that by \eqref{eq:E-bounded-sup-norm}, for all $z_0\in\cZ$, the linear functional $e_{z_0}: f\mapsto (E_r f)(z_0)$ is bounded on $L_\infty$. As $k_x$ is a bounded kernel, we have $\|\cdot\|_\cH\ge (\sup_{x\in\cX} k_x(x,x))^{-1}\|\cdot\|_\infty$; thus, $e_{z_0}$ is bounded on $\cH$, and its Riesz representer $h_{z_0}\in\cH$ has norm 
$\|h_{z_0}\|_\cH \le \sup_{x\in\cX} k_x(x,x) =: \sigma_x$. Moreover, for any $\{z_1, z_1, \ldots, z_m\}\subset \cZ$ and $a\in \RR^m$, the linear map 
$
f\mapsto \sum_{j=1}^m a_j e_{z_j}(f)
$
is also bounded on $L_\infty$, and thus $\cH$; and its representer $h_{\{z_j, a_j\}}$ has $\cH$-norm bounded by $\|a\|_2 \sigma_x$. 
By our assumptions on $f\sim\mc{GP}(0,k_x)$, we can invoke 
\citet[Definition 11.12-11.13, Lemma 11.14]{ghosal2017fundamentals} which show that, for all $m,\{z_i\}\in \cZ^m,a\in\RR^m$ and $f\sim\mc{GP}(0,k_x)$,
$$
\sum_{j=1}^m a_j (E_r f)(z_j)\sim \cN(0, \|h_{\{z_j,a_j\}}\|_\cH^2),\quad\text{where}~ \|h_{\{z_j,a_j\}}\|_\cH\le \|a\|_2\sigma_x,
$$ 
meaning that $E_r f$ distributes as a GP. Its reproducing kernel \citep[Definition 11.12]{ghosal2017fundamentals} $k_z$ satisfies 
\begin{equation}\label{eq:bounded-kernel}
    \sup_{z\in\cZ} k_z(z,z) \le \sigma_x.
\end{equation}
We also have 
\begin{equation}\label{eq:bounded-gp-draw}
    \|E_r f\|_\infty \overset{\eqref{eq:E-bounded-sup-norm}}{\le} \|f\|_\infty <\infty.
\end{equation}
As $E_r f$ is a version of $\mc{GP}(0,k_z)$, \eqref{eq:bounded-kernel} and \eqref{eq:bounded-gp-draw} prove the second claim.

\ref{it:ikf-existence}: Let $k_z$ be defined as above, 
$T_x$ denote the integral operator of $k_x$. We claim 
$k_z$ has integral operator $E T_x E^\top$: this is because
by Claim~\ref{lem:gp-gm} applied to $f\sim\mc{GP}(0,k_x)$, we have $[f]_\sim\sim N(0, T_x)$; 
moreover, we have $E_r f \sim \mc{GP}(0,k_z)$, and %
$[E_r f]_\sim = E[f]_\sim\sim N(0, E T_x E^\top)$ by definition of Gaussian measure, and the boundedness of $E$. %
Thus, by Claim~\ref{lem:gp-gm} applied to $k_z$, its integral operator is $E T_x E^\top$. 

\ref{it:ikf-continuity}: The continuity of $k_z$ follows from \ref{it:A2}, and the fact that continuous GP samples must correspond to an RKHS with a continuous kernel \citep[Example 8.1]{van2008reproducing}. Now it remains to re-establish Lemma~\ref{lem:optimal-I}. 

Following the proofs for Lemma~\ref{lem:iv-to-kl-1}, \ref{lem:optimal-I}, 
let $\{(\lambda_i^z, [e_i^z]_\sim): i\in I\}$ be a set of eigenfunctions for $E T_x E^\top$. By Lemma~\ref{lem:mercer}, $\{[e_i^z]_\sim: i\in I\}$ then determine an ONS $\{\sqrt{\lambda_i^z}e_i^z: i\in I\}$ for $\cI$. It suffices to show this is an ONB, after which we can follow the proof of Lemma~\ref{lem:optimal-I}. But as $k_z$ is bounded and continuous, the inclusion operator $\iota_z:\cI\to L_2(P(dz))$ is now injective \citep[Exercise 4.6]{steinwart2008support}; thus, by \citet[Theorem 3.1]{steinwart_mercers_2012}, $\{\sqrt{\lambda_i^z}e_i^z: i\in I\}$ is an ONB. This completes the proof.
\end{proof}

\section{Deferred Proofs: Kernel Learning}

\subsection{Notations and Preliminary Observations}\label{app:proof-common}
Let $m'=[m/2]$, $\Proj{m'}{(\cdot)}$ denote the projection onto the top $m'$ Mercer basis $\psi_1,\ldots,\psi_{m'}$, and the respective Mercer eigenvalues be $\lambda_i$. Then there exists i.i.d.~normal rvs $\bar e_{i j}$ s.t.
$$
\begin{pmatrix}
    \Proj{m'}{\gpi{1}} \\ 
    \ldots, \\
    \Proj{m'}{\gpi{m}}
\end{pmatrix} = 
\begin{pmatrix}
    \bar e_{1 1} & \ldots & \bar e_{1 m'} \\ 
    \ldots & \ldots & \ldots \\ 
    \bar e_{m 1} & \ldots & \bar e_{m m'}
\end{pmatrix} 
\begin{pmatrix}
    \sqrt{\lambda_1} \psi_1 \\ 
    \ldots \\ 
    \sqrt{\lambda_{m'}} \psi_{m'}
\end{pmatrix}.
$$
Denote the $m\times m^\prime$ matrix as $\Xi$. 
Introduce the notation $\hat G := (\esti{1};\ldots;\esti{m})$,
and $\bar G, \Psi$ so that we can write the above as 
$$
\hat G + (\bar G - \hat G) = \Xi \Psi.
$$ 
\textbf{Note our slight abuse of notation}: throughout the proof, we will use $\hat G$ to refer to both the vector-valued function as in the main text, and a ``column of $m$ functions'', i.e., a linear map from $L_2(P(dx))$ (or other suitable function spaces) to $\RR^{m}$. We define the norm $$
\|\hat G\|_2^2 := \sum_{i=1}^m \|\esti{i}\|_2^2 = \int \sum_{i=1}^m(\esti{i}(z))^2 P(dz),
$$ and similarly for $\bar G,\Psi$; and use notations such as $\hat G_j$ to refer to the $j$-th row of this ``column vector of functions'', so e.g., $\hat G_j$ refers to $\esti{j}$. 

As $\Xi$ is a $m\times m'$ Gaussian random matrix, where $m' = [m/2]$, we have the high-probability singular value bounds 
\begin{equation}\label{eq:singular-value-raw}
c''\sqrt{m} \le s_{min}(\Xi) \le s_{max}(\Xi)=s_{max}(\Xi^\top)\le c'\sqrt{m},
\end{equation}
where $c'>c''>0$ are universal constants \citep{edelman1988eigenvalues}.\SkipNOTE{
    The smallest eigenvalue bound for $\Xi^\top \Xi$ is at the beginning of Ch6. The largest singular value bound can be derived from the $O(\sqrt{n})$ bound for the $n\times n$ matrix, among other ways. 
} Thus, 
\begin{equation}\label{eq:singular-value}
s_{max}((\Xi^\top \Xi)^{-1}\Xi^\top) \le (c'')^{-2}c' m^{-1/2} =: c_r m^{-1/2}.
\end{equation} 
And for all $j\le m'$, %
\begin{equation}\label{eq:middle-singular-value}
s_j((\Xi^\top \Xi)^{-1}\Xi^\top) \overset{\eqref{eq:trivial-sv-ineq}}{\ge} s_j(\Xi^\top)\|\Xi^\top \Xi\|^{-1} 
\ge s_{min}(\Xi) \|\Xi^\top \Xi\|^{-1} \overset{\eqref{eq:singular-value-raw}}{\ge}
c''\sqrt{m}\cdot (c'\sqrt{m})^{-2} =: c_r' m^{-1/2}.
\end{equation}
(By the other inequality in \eqref{eq:trivial-sv-ineq}, the remaining singular values are zero.)

We will condition on the event defined in \eqref{eq:singular-value-raw} throughout all proofs below. On this event, we have 
$\Psi = (\Xi^\top \Xi)^{-1} \Xi^\top \bar G$, and we can define the transformed feature map
$$
\hat\Psi := (\Xi^\top \Xi)^{-1}\Xi^\top \hat G.
$$

\subsection{Proof for Theroem~\ref{prop:approx}}\label{app:proof-prop-approx}

\newcommand{\etest}{\bar e_*}
Recall the notations and observations in Appendix~\ref{app:proof-common}. 

As we will work with the truncated estimators $\esti{j}$, we first show that the truncation does not affect the error rate. 
By Borell's inequality, we have, for any $B>1$
\begin{equation}\label{eq:borell-sup}
\PP(\max_{i\in [m]} \|\gpi{i}\|_\infty \ge B) \le 
m\,\PP(\|\gpi{i}\|_\infty \ge B) < C_1 m e^{-C_2 B^2}. 
\end{equation}
Thus the above event has high probability for $B=4 C_2^{-1/2} \sqrt{\log m}$. On this event, the truncated estimator will have the same $L_2$ error as the original estimators, leading to 
$$
\EE_{\data,G} \|\hat G - G\|_2^2 := \EE_{\data,G} \sum_{j=1}^m \|\esti{j} - \gpi{j}\|_2^2 = m\xi_n^2. 
$$ And by Markov's inequality we have, with high probability 
\begin{equation}\label{eq:G-mean-bound}
\|\hat G-G\|_2^2 \le m\xi_n^2 \log n. 
\end{equation}
We further restrict to the event on which the above holds. 
Now, 
for any $g^*\in L_2(P(dx))$, let $$
\etest  := (\lambda_1^{-1/2}\<g^*,\psi_1\>_2, \ldots, \lambda_{m'}^{-1/2}\<g^*,\psi_{m'}\>_2)\in\RR^{m'}
$$ so that $\|\etest\|_2 = \|\Proj{m'}{g^*}\|_\cI, \etest^\top\Psi=\Proj{m'}{g^*}$. Then for 
$\tilde g^* := \etest^\top \hat\Psi$, 
we have 
\begin{align*}
\|\tilde g^*\|_{\tilde\cI} &= m^{1/2}\|\Xi(\Xi^\top \Xi)^{-1} \etest\|_2 
\overset{\eqref{eq:singular-value}}{\le} c_r\|\etest\|_2  = c_r\|\Proj{m'}{g^*}\|_\cI, \numberthis\label{eq:rkhs-norm-bound-proof} \\
\|\tilde g^* - g^*\|_2 &\le \|\bar e_*^\top (\hat\Psi - \Psi)\|_2 + \|g^* - \Proj{m'}{g^*}\|_2 
\\
&\le c_r\|\Proj{m'}{g^*}\|_{\cI} (\xi_n + m^{-(\beff+1)/2})\sqrt{\log n} + \|g^* - \Proj{m'}{g^*}\|_2, \numberthis\label{eq:suboptimal}
\end{align*}
where \eqref{eq:suboptimal} holds on a high-probability event independent of $g^*$, by Lemma~\ref{lem:suboptimal} which we prove below. \hfill\qedsymbol

\begin{remark}
The failure probability comes from four sources: a Borell's inequality in \eqref{eq:borell-sup}, a Markov's inequality in \eqref{eq:G-mean-bound}, the singular value bound \eqref{eq:singular-value}, and a Markov's inequality in the above lemma, on $\bar G-G$. 
The failure probability of \eqref{eq:borell-sup} is $O(m^{-3})$, which can be easily improved by increasing $B$. 
The failure probability of \eqref{eq:singular-value} is exponentially small~\citep[Theorem 4.6.1]{vershynin2018high}. %
The Markov's inequality on $\bar G-G$ can be sharpened with another use of Borell-TIS inequality. %
Therefore, the main source of failure probability comes from our lack of further assumptions on the black-box learner, and can be improved given such assumptions. 
\end{remark}

\begin{lemma} For any $j\in[m]$,
\begin{equation}\label{eq:kl-trunc-error}
\EE_{\gpi{j}} \|\gpi{j} - \Proj{m'}{\gpi{j}}\|_2^2 = \sum_{j=m'+1}^\infty \lambda_j \asymp m^{-\beff}.
\end{equation}
\end{lemma}
\begin{proof}
By the $L_2$ convergence of K-L expansion, and the eigendecay assumption. 
\end{proof}

\begin{lemma}\label{lem:suboptimal}
On a high-probability event determined by $G$ and $\data$, we have, for any $\etest\in\RR^m$, 
$$
\|\etest^\top(\hat\Psi - \Psi)\|_2 \le c_r \|\etest\|_2(\xi_n + m^{-(\beff+1)/2})\sqrt{\log n}.
$$
\end{lemma}

\begin{proof}
Introduce the notation $\tilde e := \Xi (\Xi^\top \Xi)^{-1} \etest$, so that we can write the LHS as 
\begin{align*}
\|\etest^\top(\hat\Psi-\Psi)\|_2 &= \|\etest^\top (\Xi^\top \Xi)^{-1}\Xi^\top (\bar G-\hat G)\|_2  = \|\tilde e(\bar G - \hat G)\|_2
\le \|\tilde e(\bar G - G)\|_2 + \|\tilde e(\hat G-G)\|_2.
\end{align*}
We consider the two terms in turn. 
\begin{enumerate}[leftmargin=*]
    \item 
For the first term, observe 
\begin{equation}\label{eq:G-trunc-err-mean}
\EE_{G-\bar G} \<(G-\bar G)_i, (G-\bar G)_j\> = \mbf{1}_{\{i=j\}} \EE \|(G-\bar G)_1\|_2^2, ~~\text{where}~~ \EE \|(G-\bar G)_1\|_2^2\overset{\eqref{eq:kl-trunc-error}}{\asymp} m^{-\beff}. 
\end{equation}
Moreover, 
$\Xi$ and $G-\bar G$ depends on disjoint subsets of the generalized Fourier coefficients of the GP samples, and are thus independent. Thus, $\tilde e$ and $G-\bar G$ are also independent, and  
\begin{align*}
\EE_{G-\bar G} \|  (G-\bar G)^\top \tilde e\|_2^2 
&= \EE_{G-\bar G}\: \tilde e^\top (G-\bar G)(G-\bar G)^\top \tilde e \asymp m^{-\beff} \|\tilde e\|_2^2. 
\end{align*}
By the Markov inequality we have, w.h.p.~w.r.t.~$G-\bar G$, 
\begin{equation}\label{eq:lem-suboptimal-case-1}
 \| (G-\bar G)^\top \tilde e\|_2^2 \le  m^{-\beff} \log n \|\tilde e\|_2^2.
\end{equation}
\item For the second term we have
\begin{align*}
\|\tilde e^\top (G-\hat G)\|_2^2 &=
\int (\tilde e^\top (G-\hat G)(z))^2 P(dz) 
\le 
\int \|\tilde e\|_2^2 \|(G-\hat G)(z)\|^2 P(dz) \\ 
&= \|\tilde e\|_2^2 \|G-\hat G\|_2^2
\overset{\eqref{eq:G-mean-bound}}{\le} \|\tilde e\|_2^2 m\xi_n^2 \log n. \numberthis\label{eq:lem-suboptimal-case-2}
\end{align*}
In the above, the first inequality is the Cauchy-Schwarz inequality in $\RR^m$. 
\end{enumerate}
Combining these results and the observation that 
\begin{equation}\label{eq:tilde-e-norm}
\|\tilde e\|_2 = \|\Xi(\Xi^\top \Xi)^{-1}\etest\|_2 \overset{\eqref{eq:singular-value}}{\le} c_r m^{-1/2} \|\etest\|_2,
\end{equation}
we have, with high probability w.r.t.~$G$ and $\data$,
$$
\|\etest^\top(\hat\Psi-\Psi)\|_2^2 \le 
\|\tilde e\|_2^2 (m\xi_n^2 + m^{-\beff}) \log n \le c_r^2 \|\etest\|_2^2 (\xi_n^2 + m^{-(\beff+1)}) \log n. 
$$
\end{proof}

\subsection{Proof for Proposition~\ref{prop:est}}\label{app:estimation}

Recall the notations and observations in Appendix~\ref{app:proof-common}. In the following, the first two lemmas bound the {\em entropy number} of $\tilde\cI_1$, and the final proof uses it to bound the local Rademacher complexity.

\begin{lemma}\label{lem:larger-kernel}
There exists an RKHS $\bar\cI$ s.t.~$\tilde\cI_1\subset\bar\cI_1$, and, on the event in Theorem~\ref{prop:approx}, we have, for all $0\le i_s\le i_e$, the Mercer eigenvalue bound
\begin{equation}\label{eq:Cbar-eval-sum-bound}
\sum_{i=i_s}^{i_e}\lambda_i(\bar C)\lesssim i_s^{-\bar b} + \chi_n^2,
\end{equation}
where 
$
\chi_n := (m^{-\frac{\beff+1}{2}} +\xi_n)\sqrt{\log n}, 
$ and
the constant hidden in $\lesssim$ is independent of $j,m$ or $n$.
\end{lemma}
\begin{proof}
We first define $\bar \cI$. 
Define $
\Delta_1 G_n = \hat G_n - G_n,
\Delta_2 G_n = G_n - \bar G_n.
$
Then $$
\begin{aligned}
\tilde\cI_1 &= \{\theta^\top (\bar G_n+\Delta_1 G_n + \Delta_2 G_n)(\cdot): \|\theta\|_2^2\le m^{-1}\}  \\ 
&\subset \{\theta_1^\top \bar G_n(\cdot) + \theta_2^\top \Delta_1 G_n(\cdot)+ \theta_3^\top \Delta_2 G_n(\cdot)
: \|\theta_1\|_2^2+\|\theta_2\|_2^2+\|\theta_3\|_2^2\le 3 m^{-1}\}.
\end{aligned}
$$
The RHS is the unit-norm ball of an RKHS, denoted as $\bar\cI$. Its reproducing kernel and integral operator are (recall our notations in Appendix~\ref{app:proof-common})
\begin{align*}
\bar k(z,z') &= \frac{3}{ m}(\bar G_n(z)^\top \bar G_n(z') + \Delta_1 G_n(z)^\top \Delta_1 G_n(z') +  \Delta_2 G_n(z)^\top \Delta_2 G_n(z')
),\\
\bar C &= \frac{3}{m}(\bar G_n^\top \bar G_n + (\Delta_1 G_n)^\top \Delta_1 G_n + (\Delta_2 G_n)^\top \Delta_2 G_n
).
\end{align*}
By \citet[Theorem 2]{wielandt1955extremum}, 
we have, for any $0\le i_s\le i_e$, 
$$
\sum_{i=i_s}^{i_e} \lambda_i(\bar C) \le \frac{3}{m}\left(
\sum_{i=i_s}^{i_e} \lambda_i(\bar G_n^\top \bar G_n) + \Tr((\Delta_1 G_n)^\top \Delta_1 G_n) + \Tr((\Delta_2 G_n)^\top \Delta_2 G_n) 
\right).
$$ 
For the first part above, recall $\bar G_n = \Xi\Psi$; and on the event defined in Theorem~\ref{prop:approx}, we have, by \eqref{eq:singular-value-raw}, $
(c'')^2 m \le \lambda_i(\Xi^\top\Xi) \le (c')^2 m.
$ Thus, 
$$
\lambda_i(\bar G_n^\top \bar G_n) = \lambda_i(\Psi^\top \Xi^\top \Xi \Psi) \overset{\eqref{eq:trivial-sv-ineq}}{\asymp} m \lambda_i(\Psi^\top\Psi) \lesssim m i^{-(\beff+1)},
$$
The second term is bounded as
$$
\Tr((\Delta_1 G_n)^\top(\Delta_1 G_n)) = \|\Delta_1 G_n\|_2^2 
\overset{\eqref{eq:G-mean-bound}}{\lesssim} m\xi_n^2 \log n.
$$
For the third, by Markov's inequality on \eqref{eq:G-trunc-err-mean}, we have, w.h.p.~w.r.t.~$G_n-\bar G_n$,
$$
\Tr((\Delta_2 G_n)^\top(\Delta_2 G_n)) \lesssim  m^{-\beff} \log n.
$$
Note this event has been included in Theorem~\ref{prop:approx} via Lemma~\ref{lem:suboptimal}. Combining the above, we have, for all $0\le i_s\le i_e$, 
\begin{equation*}
\sum_{i=i_s}^{i_e} \lambda_i(\bar C) \lesssim \left(\sum_{i=i_s}^{i_e} i^{-(\beff+1)}\right) + (\xi_n^2 + m^{-(\beff+1)})\log n \lesssim 
i_s^{-\beff} + (\xi_n^2 + m^{-(\beff+1)})\log n \leq i_s^{-\beff} + \chi_n^2.
\end{equation*}
\end{proof}

\begin{lemma}\label{lem:entropy-number}
Let $\Ztest := \{z_1,\ldots, z_{n_2}\}$ be a sample of $n_2$ iid inputs independent of $\data$, and $L_2(\Ztest)$ denote the $L_2$ space defined by the respective empirical measure. Then, 
on the high-probability event defined in Theorem~\ref{prop:approx}, we have 
$$
\EE_{\Ztest} e_j(\mrm{id}: \tilde\cI \to L_2(\Ztest)) \lesssim
j^{-(\beff+1)}(
    \min\{j,n_2\}^{\frac{\beff+1}{2}} +  %
    \min\{j,n_2\}^{\beff+\frac{1}{2}} \chi_n
    )),
$$
where 
$
\chi_n := (m^{-\frac{\beff+1}{2}} +\xi_n)\sqrt{\log n}, 
$ and
the constant hidden in $\lesssim$ is independent of $j,m$ or $n$.
\end{lemma}
\begin{proof}
As $\tilde\cI_1\subset\bar\cI_1$, it suffices to establish the bound for $\bar\cI_1$. 
For this, we first invoke \citet[Theorem 7.30]{steinwart2008support}, but with the RHS of the last display on p.~275 replaced by \eqref{eq:Cbar-eval-sum-bound}. Following the proof we find, for all $p\in (0,1)$, 
$$
\EE_{\Ztest} e_j(S_{\bar k,D}^*) \lesssim 
j^{-1/p} \sum_{i=1}^{\min\{j,n_2\}}
i^{1/p-1} 
\sqrt{i^{-1}(i^{-\beff} + \chi_n^2)}.
$$
As stated in their Corollary 7.31, the LHS above is our desired empirical entropy number. Following the proof for that corollary, we set $p = (\beff+1)^{-1}$ above, leading to 
$$
\begin{aligned}
\EE_{\Ztest} e_j(\tilde\cI_1\to L_2(\Ztest)) &\lesssim 
j^{-(\beff+1)} \sum_{i=1}^{\min\{j,n_2\}} 
i^{\beff}(i^{-(\beff+1)/2} + i^{-1/2}\chi_n) 
\\ &\lesssim 
j^{-(\beff+1)}(
    \min\{j,n_2\}^{\frac{\beff+1}{2}} +  %
    \min\{j,n_2\}^{\beff+\frac{1}{2}} \chi_n).
\end{aligned}
$$
\end{proof}

We now prove Proposition~\ref{prop:est}, by plugging in 
our new entropy number bound to the chaining argument in 
\citet[Theorem 7.12 -- Theorem 7.16]{steinwart2008support}. 
\begin{proposition}[Proposition~\ref{prop:est}, restated]
Let $$
\bar R_{n_2}(\tilde\cI_1; \sigma) := \EE_{\Ztest,\varepsilon_i} \sup_{g\in\tilde\cI_1, \|g\|_2\le \sigma} 
\left|\frac{1}{n_2}\sum_{i=1}^{n_2} \varepsilon_i g(z_i)
\right|
$$ be the local Rademacher complexity, and $B := \sup_{g\in\tilde\cI_1}\|g\|_\infty$. Then its critical radius, defined as the solution $\delta_{n_2}$ to
\begin{equation}\label{eq:crit-rad-defn}
\bar R_{n_2}(\tilde\cI_1; \delta) \le \delta^2,
\end{equation}
is bounded as 
$$
\delta_{n_2}^2 \lesssim B^{\frac{\beff}{\beff+2}}n_2^{-\frac{\beff+1}{\beff+2}} + \frac{\log^2 n_2}{n_2} + \chi_n^2.
$$
\end{proposition}
\begin{proof}
Define $I_\sigma := \{g\in\tilde\cI_1, \|g\|_2\le \sigma\}$.  
By \citet[Lemma~7.14 and the last display on p.~254]{steinwart2008support}, we have 
\begin{align*}
e_i(I_\sigma, L_2(\Ztest)) &\le e_{i-1}(I_\sigma, L_2(\Ztest)), \\
\EE_{\Ztest} e_1(I_\sigma, L_2(\Ztest)) &\le
\EE_{\Ztest} \sup_{g\in I_\sigma} \|g\|_{L_2(\Ztest)} \le (\sigma^2 + 8B \bar R_{n_2}(\tilde\cI_1;\sigma))^{1/2} =: s_1.
\end{align*}
Taking expectation on both sides of \citet[Theorem 7.13]{steinwart2008support}, we find
\begin{align*}
\bar R_{n_2}(\tilde\cI; \sigma) &\le 
\sqrt{\frac{\ln 16}{n_2}} \EE_{\Ztest}\Big(
    \sum_{i=1}^\infty 2^{i/2} e_{2^i}(I_\sigma, L_2(\Ztest)) + 
    \underbrace{\sup_{g\in I_\sigma} \|g\|_{L_2(\Ztest)}}_{\le s_1}
\Big).
\end{align*}
Plugging in our Lemma~\ref{lem:entropy-number} to the first term in RHS above, we have, for $j := 2^i$ and $m_2 := [\log_2 n_2]$,
\begin{align*}
&{\phantom{=}}\EE_{\Ztest} \sum_{i=1}^\infty 2^{\frac{i}{2}} e_{2^i}(I_\sigma, L_2(\Ztest))  \\
&\le 
\sum_{i=1}^{\infty} 2^{\frac{i}{2}} \min\{s_1, j^{-\frac{\beff+1}{2}}\} + 
\sum_{i=1}^{m_2} \cancel{2^{\frac{i}{2}} \cdot j^{-\frac{1}{2}}}\chi_n + %
\sum_{i=m_2+1}^\infty 2^{\frac{i}{2}} n_2^{\beff+\frac{1}{2}} j^{-(\beff+1)} \chi_n 
\\ 
&=: S_1 + \chi_n\log n_2 + S_2.
\end{align*}
By \citet[Lemma 7.15]{steinwart2008support} with $p\gets (\beff+1)^{-1}$, we have 
$
S_1 \lesssim s_1^{\frac{\beff}{\beff+1}},
$
where the constant in $\lesssim$ is independent of all sample sizes. For $S_2$, we have 
\begin{align*}
S_2 &\le n_2^{\beff+\frac{1}{2}} \chi_n 
\int_{x=m_2}^\infty 2^{\frac{x}{2} - (\beff+1)x} dx
\lesssim  n_2^{\beff+\frac{1}{2}} \chi_n 2^{\frac{m_2}{2}-(\beff+1)m_2} \le \chi_n.
\end{align*}
Combining, we have 
\begin{align*}
\EE_{\Ztest} \sum_{i=1}^\infty 2^{\frac{i}{2}} e_{2^i}(I_\sigma, L_2(\Ztest)) &\lesssim s_1^{\frac{\beff}{\beff+1}} + \chi_n \log n_2, \\ 
\bar R_{n_2}(\tilde\cI;\sigma) 
&\le C n_2^{-\frac{1}{2}} \left(
(\sigma^2+B\bar R_{n_2}(\tilde\cI;\sigma))^{\frac{\beff}{2(\beff+1)}} + 
\chi_n\log n_2
\right),
\end{align*}
for some constant $C$ which does not depend on $\sigma$ or any sample sizes.
To find solutions to \eqref{eq:crit-rad-defn} 
using the above bound, we will restrict to 
\begin{equation}\label{eq:sigma-bound-1}
\sigma^2 \ge 
2C\Big(\frac{\log^2 n_2}{n_2} + \chi_n^2\Big) \ge
4 C n_2^{-\frac{1}{2}} \chi_n\log n_2.
\end{equation}
For such $\sigma$ we now have 
$$
\bar R_{n_2}(\tilde\cI;\sigma) \le 
C n_2^{-\frac{1}{2}} (\sigma^2+B\bar R_{n_2}(\tilde\cI;\sigma))^{\frac{\beff}{2(\beff+1)}} + \frac{1}{4}\sigma^2 \le 
\max\Big\{
    3 C n_2^{-\frac{1}{2}} \sigma^{\frac{\beff}{\beff+1}}, 
    3 C B^{\frac{\beff}{\beff+2}} n_2^{-\frac{\beff+1}{\beff+2}}, 
    \frac{1}{2}\sigma^2
\Big\}.
$$
For the RHS to be $\le \sigma^2$ it suffices to consider the first two terms, leading to $\sigma^2 \gtrsim B^{\frac{\beff}{\beff+2}} n_2^{-\frac{\beff+1}{\beff+2}}$. Combining with \eqref{eq:sigma-bound-1} complete the proof.
\end{proof}

\subsection{Average-Case Analysis}\label{app:proof-thm-ml}

\begin{corollary}[average-case approximation error]\label{thm:ml}
In the setting of Theorem~\ref{prop:approx}, 
with $\data$-probability $\to 1$
we have 
$$
\EE_{g^*\sim\mc{GP}}\|\tilde g^*-g^*\|_2 \le \EE_{g^*\sim\mc{GP}}\epsilon_{n_1,m}(g^*),
$$
where $\epsilon_{n_1,m}$ is defined by replacing the $\xi_{n_1}$ in \eqref{eq:approx-bound} with $m^{-1/2}\xi_{n_1}$, and $\|\tilde g^*\|_{\tilde\cI}$ satisfies \eqref{eq:approx-norm-bound}. 
In particular, for $m = [n_1^{1/(\beff+1)}]$, we have \begin{equation}\label{eq:avg-approx-bound}
\EE_{g^*\sim\mc{GP}}\|\tilde g^*-g^*\|_2 = \tilde O\big(\xi_{n_1}+n_1^{-\frac{\beff}{2(\beff+1)}}\big).
\end{equation}
\end{corollary}
While we will not use \eqref{eq:avg-approx-bound} for IV, it hints at the possibility of improvement, through more careful analyses and/or additional assumptions. Remark~\ref{rmk:suboptimality} discusses this in more detail.

\begin{proof}
Recall the proof of Theorem~\ref{prop:approx}, and the notations therein. 
We can easily check that all statements in the proposition still hold true, with the high-probability events now defined w.r.t.~($G$ and) the combined training samples. 
Thus, \eqref{eq:rkhs-norm-bound-proof} will still hold, and it remains to prove the two approximation error bounds. We will improve Lemma~\ref{lem:suboptimal} and plug into \eqref{eq:suboptimal}. We can check the proof of the lemma also remains valid, and \eqref{eq:lem-suboptimal-case-1} is still good enough; thus, it suffices to improve \eqref{eq:lem-suboptimal-case-2} about the term $\|\tilde e^\top (G-\hat G)\|_2$.

The proof of Theorem~\ref{prop:approx} only invokes Lemma~\ref{lem:suboptimal} with 
$$
\etest = (\lambda_1^{-1/2}\<g^*,\psi_1\>_2, \ldots, \lambda_{m'}^{-1/2}\<g^*,\psi_{m'}\>_2),
$$ where $\{(\lambda_i,\psi_i)\}$ are the Mercer decomposition of $k_x$. 
For $g^*\sim\mc{GP}(0,k_x)$, $\etest$ will distribute as $N(0, I)$. We will thus 
modify Lemma~\ref{lem:suboptimal} into: 
\begin{equation}\label{eq:lem-suboptimal-avg}
\EE_{\etest\sim N(0,I)}\|\etest(\hat\Psi-\Psi)\|_2 \le c_r(m^{-1/2}\xi_n+m^{-(\beff+1)/2})\sqrt{\log n}\cdot \EE_{\etest\sim N(0,I)} \|\etest\|_2.
\end{equation}
(The statement is still restricted to a high-probability event w.r.t.~training samples and $G$.) To prove the above, note that \eqref{eq:lem-suboptimal-case-1} still holds for any fixed $\etest$, so the second term above remains correct. \\ 
It remains to deal with the first term. 
Introduce the notation 
$S := \Xi(\Xi^\top \Xi)^{-1}$, so that $\tilde e = S \etest$. 
We have,
for \emph{any fixed $(\data,G)$,}
\begin{align*}
\EE_{\etest}(\|\tilde e^\top (G-\hat G)\|_2^2) &= 
\EE_{\etest} \Tr(\tilde e\tilde e^\top (G-\hat G)(G-\hat G)^\top) \\ 
&=\EE_{\etest}\Tr(S\etest\etest^\top S^\top (G-\hat G)(G-\hat G)^\top) \\ 
&=\Tr(S\cdot\EE_{\etest}(\etest\etest^\top)\cdot S^\top (G-\hat G)(G-\hat G)^\top) \\ 
&=
\Tr(S S^\top (G-\hat G)(G-\hat G)^\top) \\
&\le \|S S^\top\| \Tr((G-\hat G)(G-\hat G)^\top) = 
\|S S^\top\| \sum_{j=1}^m \|\gpi{j}-\esti{j}\|_2^2.
\end{align*}
The inequality above is the von Neumann inequality $\Tr(AB)\le \|A\|_{op}\Tr(B)$. 
Conditioned on the $(\data,G)$-measurable event on which \eqref{eq:singular-value} and \eqref{eq:G-mean-bound} hold, the first term above is bounded by $c_r^2 m^{-1}$, and the second by $m\xi_n^2\log n$. Thus we have, with $(\data,G)$-probability $\to 1$,
$$
\EE_{g^*}\|\tilde e^\top (G-\hat G)\|_2^2 \le c_r^2\xi_n^2\log n \lesssim
c_r^2\xi_n^2\log n \frac{(\EE_{\etest}\|\bar e_*\|_2)^2}{m},
$$
where the last inequality follows from~\citet[Theorem 3.1.1]{vershynin2018high}.
Since $\EE X^2\ge (\EE X)^2$, we complete the new bound for the first term in \eqref{eq:lem-suboptimal-avg}, and subsequently \eqref{eq:lem-suboptimal-avg}. Following the original proof, we can see that \eqref{eq:suboptimal} becomes 
$$
\EE_{g^*\sim\mc{GP}} \|\tilde g^*-g^*\|_2 \le \EE_{g^*\sim\mc{GP}}\left[
c_r \|\Proj{m'}{g^*}\|_\cI(m^{-1/2}\xi_n + m^{-(\beff+1)/2})\sqrt{\log n}+\|g^*-\Proj{m'}{g^*}\|_2\right],
$$
which proves the first claim.

For the second claim, from Jensen's inequality, we know
$\EE_{g^* \sim \mc{GP}} \| \Proj{m'}{g^*} \|_{\cI} 
 = \EE_{\bar e_*} \| \bar e_* \|_2 \leq \sqrt{\EE_{\bar e_*} \| \bar e_* \|_2^2} = m'$, and thus the first term in the above display is $\tilde O( \xi_n + m^{-\bar b / 2})$.
 Similarly, we have
 \[ \EE_{g^*\sim\mc{GP}}\|g^*-\Proj{m'}{g^*}\|_2
 \leq \left ( \EE_{g^*\sim\mc{GP}}\|g^*-\Proj{m'}{g^*}\|_2^2 \right )^{1/2} 
 \overset{\eqref{eq:kl-trunc-error}}{\leq} O(m^{-\bar b / 2}), \]
 which completes the proof.\
\end{proof}

\begin{remark}\label{rmk:suboptimality}
The only change in this proof is a new bound for the estimation error from the oracle. The Cauchy-Schwarz inequality used in the previous bound \eqref{eq:lem-suboptimal-case-2} was is typically loose: it requires the vector $\tilde e$, determined by the test function, to be parallel with $(\gpi{1}-\esti{1},\ldots,\gpi{m}-\esti{m})$, the estimation residuals. %
This is at odds with the intuition that residual functions for each independently drawn GP samples to be in some sense uncorrelated. 

For an extreme example, suppose our target RKHS $\cI=\mrm{span}\{\psi_1\}$ is one-dimensional, so that the standard GP prior draw can be written as $\epsilon \psi_1$ for some $\epsilon\sim\cN(0,1)$. Then, independent draws from $\mc{GP}(0,k_z)$ should have independent signs, %
and it seems very strange if a useful regression algorithm often returns residual functions with correlated signs determined by $\tilde e$, for such $\gpi{\cdot}$. 
The fact that such oracles are allowed by our assumption suggest there is room for improvement, although we do not pursue this path, since the suboptimality is relatively mild (see the end of Appendix~\ref{app:derivation-example-cont}). Empirically, our estimator works well for $n_1 = n_2$. 

Note that while similar problems have been studied in multi-task learning, such works often assume independent inputs for the different tasks, which corresponds to the different GP prior draws in our setting. Corollary~\ref{thm:ml} appears new in this aspect, at least among works that established fast rates.
\end{remark}

\subsection{Deferred Derivations and Additional Discussions}\label{app:main-instantiations}

\subsubsection{Derivation for Example~\ref{ex:dnn-oracles}} \label{app:derivation-examples}

We first derive the expression of $\xi_n$. For regression functions of the form $g = \bar g\circ\Phi$, where $\Phi$ is defined as in the text, and $\bar g$ is $\underline{\beta}_2$ H\"older regular and have $d_l$ inputs, 
\citet{schmidt-hieber_nonparametric_2020} establishes the convergence rate 
\begin{equation}\label{eq:s-h-raw-rate}
\xi_n \lesssim (n^{-\frac{\beta_1}{2\beta_1+d_z}} + n^{-\frac{\underline{\beta}_2}{\underline{\beta_2}+d_l}})\|\bar g\|_{C^{\underline{\beta}_2}}\log^{3/2} n + \epsilon_{opt}.
\end{equation}
(We view $\|\Phi\|_{C^{\beta_1}}$ as a constant.)
The result holds uniformly for all $\bar g\in C^{\underline{\beta}_2}$ with uniformly bounded norm. 

Recall Example~\ref{ex:matern-gp}: for any $\epsilon>0$, there exists a version\footnote{We can work with any version of the GP prior since the regression oracle only accesses its evaluation on a finite number of points, which have the same distribution among all versions.} of the Mat\'ern GP $\Pi_z$ for $\bar g$ s.t.~$\bar g\in C^{\beta_2-\epsilon}$ a.s. We set $\underline{\beta}_2 = \beta_2-\epsilon$. 
By Lemma~\ref{lem:borell-tis}, the random variable $\|\bar g\|_{C^{\underline{\beta}_2}}$ has a subgaussian tail. Thus for any $p>1$, we can choose some $C_1>0$ which depend on $p$ and $\epsilon$, leading to 
$$
\Pi_z(E) := 
\Pi_z(\{\bar g: \|\bar g\|_{C^{\underline{\beta}_2}}^2 \le C_1 \log n\}) \ge 1 - n^{-p}.
$$
Let the NN model in \citet{schmidt-hieber_nonparametric_2020} be constructed for $g=\bar g\circ\Phi$ satisfying the above norm bound, and $\hat g_n$ denote the resulted estimator. Consider 
\begin{align}
\EE_{\Pi_z} \|\hat g_n - g\|_2^2 &\le  
\EE_{\Pi_z} [\mbf{1}_{E} \|\hat g_n - g\|_2^2] + 
2\EE_{\Pi_z} [\mbf{1}_{E^c}(\|\hat g\|_\infty^2+\|g\|_2^2)]  \nonumber \\ &\le 
\EE_{\Pi_z} [\mbf{1}_{E} \|\hat g_n - g\|_2^2] + 
\frac{2\EE_{\Pi_z}[\|\hat g\|_\infty^2\mid E^c]}{n_2^p} + 2\EE_{\Pi_z} [\mbf{1}_{E^c}\|g\|_2^2].\label{eq:subg-argument}
\end{align}
The first term clearly has the desired bound, with two extra logarithms. For the second term, note the estimator in \citet{schmidt-hieber_nonparametric_2020} has sup norm bounded as $\|\hat g_n\|_\infty\le C_1\log n$, so it also has the desired bound. For the last term, another application of Lemma~\ref{lem:borell-tis} shows that $\|\bar g\|_2$ norm also has a subgaussian tail. Let $\Phi_2$ denote its CDF. Then we have 
$$
\EE(\mbf{1}_{E^c}\|g\|_2^2) \le \int_{\Phi_2^{-1}(1-n_2^{-p})}^{+\infty} x^2 d\Phi_2(x) \lesssim \int_{C_2\sqrt{p\log n_2}}^\infty x^2 e^{-C_3 x^2} dx \lesssim n_2^{-C_4 p} \sqrt{p\log n_2}.
$$
In the above, $C_2,C_3,C_4$ are determined by $\Pi_z$, so we can choose a sufficiently large $p$ so that the RHS is $\lesssim n^{-1}$. This completes the derivation for $\xi_n$. \qed

We now turn to the regression error using a fixed-form kernel. The Mat\'ern process $\bar g\sim\mc{GP}(0, \bar k_z)$ takes value in the H\"older space $C^{\underline{\beta}_2}(\bar\cZ)$ w.p.1, and the order cannot be made larger (Example~\ref{ex:matern-gp}). Thus, the random function $g = \bar g\circ\Phi$ can only take value in $C^{\min\{\underline{\beta}_2,\beta_1\}}(\cZ)$ w.p.1. 
Our claimed regression rate follows from \cite{van_der_vaart_information_2011} (for Mat\'ern kernels), or \cite{van_der_vaart_adaptive_2009} (for Gaussian/RBF kernels with an adaptive bandwidth).\footnote{Note that we did not rule out the possibility of a better rate being attainable: we did not prove a lower bound, instead only presented the best known upper rate given our H\"older regularity condition of $g$.  %
However, improvement is most likely impossible: 
as discussed in introduction, separability results have been established in similar feature learning settings for certain fixed-form kernels; in the setting of this example, \cite{schmidt-hieber_nonparametric_2020} established a lower bound for a wavelet model with a similar order.}
\qed

\subsubsection{Derivations for Example~\ref{ex:dnn-oracles-cont}}\label{app:derivation-example-cont}

By properties of the Mat\'ern RKHS (Example~\ref{ex:matern}), the latent-space kernel $\bar k_z$ has eigendecay $\lambda_i \lesssim i^{-(1+2\beta_2/d_l)}$. Thus we have $\beff+1 = 1+\nicefrac{2\beta_2}{d_l}$, and 
$$
\xi_{n_1} = \tilde\cO\Big(\epsilon_{fea,n_1} + n_1^{-\frac{\beff-\epsilon}{2(\beff-\epsilon+1)}}
\Big) + \epsilon_{opt},
$$ 
for all $\epsilon>0$. We first establish the three approximation error bounds.
\begin{enumerate}[leftmargin=*,nosep]
    \item 
The first claim (Case (i) in the text) follows by observing $m^{-\beff} = n_1^{-\beff/(\beff+1)} \ll \xi_{n_1}$. 

\item 
For Case (ii), let $\tilde g_0\in\tilde\cI$ be the approximation returned by Corollary~\ref{thm:ml}, then we have 
\begin{equation}\label{eq:avg-g-approx-1}
\EE_g \|\tilde g_0 - g\|_2 \overset{\eqref{eq:avg-approx-bound}}{=} \tilde O(\xi_{n_1}+n_1^{-\frac{\beff}{2(\beff+1)}}) = \tilde O(\xi_{n_1}). 
\end{equation}
It remains to provide an RKHS norm bound. 
By convergence of the Karhunen-Lo\`eve expansion \cite[Thm.~3.1]{steinwart_convergence_2019} we have $\Proj{m}{g} = \sum_{i=1}^m \epsilon_i (\lambda^z_i)^{1/2} \psi^z_i$, where $\{(\lambda^z_i,\psi^z_i)\}$ denote the Mercer representation of $k_z$, and $\epsilon_i\sim\cN(0,1)$ are iid. As $ (\lambda^z_i)^{1/2} \psi^z_i$ is an ONB of $\cH$ (Lem.~\ref{lem:mercer}), we have 
$
\|\Proj{m}{g}\|_\cH^2 = \sum_{i=1}^m \epsilon_i^2.
$
Thus, 
$\chi^2$-concentration bounds \cite[Example 2.11]{wainwright2019high} yields 
$$
\PP(\|\Proj{m}{g}\|_\cH^2>2\sqrt{2}m)\le e^{-m} = \exp\big(-n^{\frac{1}{\beff+1}}\big). 
$$
Thus, we let $\tilde g = 0$ on the above event, and $\tilde g_0$ otherwise. Repeating the argument of \eqref{eq:subg-argument} we can show that the additional $L_2$ error is negligible, and $\tilde g$ satisfies the same bound of \eqref{eq:avg-g-approx-1}. In this case, the average-case $L_2$ rate of DNN is $\xi_{n_1}$, by definition. 

\item 
Additionally, we claim that in Case (ii), for $m=[n^{\frac{\beff}{(\beff+1)^2}}]$, 
there exists $\tilde g^*\in\tilde\cI$ s.t.
$$
\PP_{g^*\sim\mc{GP}}\big(
\|\tilde g^*\|_{\tilde\cI}\le c_r n_1^{1/2(\beff+1)}, ~
\|\tilde g^*-g^*\|_2 = \tilde \cO(m^{1/2}\xi_{n_1}+
    n_1^{-\beff^2/2(\beff+1)^2})
\big) \ge 1 - e^{-m}\to 1.
$$
To show this, we shall apply Theorem~\ref{prop:approx} to a different target function $g^*$. Let us denote by $g$ the GP prior draw to be approximated. 
Then we have the standard prior mass bound for the ``sieve set'' \cite{van_der_vaart_rates_2008}: for some $C>0$,
\begin{equation}\label{eq:basic-sieve}
\forall\bar n\in\mb{N}, 
\PP_{g\sim\mc{GP}(0,k_z)}\big(\exists \gApprox[\bar n], 
\|\gApprox[\bar n]\|_\cI \le C \bar n^{1/2(\beff+1)}, 
\|\gApprox[\bar n] - g\|_2 \le C \bar n^{-\beff/2(\beff+1)}
\big) \ge 1 - \exp\big(-\bar n^{\frac{1}{\beff+1}}\big). 
\end{equation}
(This is proved by combining the Borell inequality \cite[e.g.,][Prop.~11.17]{ghosal2017fundamentals} and the $L_2$ small-ball probability bound \cite[e.g.,][Cor.~4.9, with $\beta=1$]{steinwart_convergence_2019}.)  
We set $\bar n := [n_1^{\beff/(\beff+1)}]$ and restrict to $g$ in the above event. 
Invoking Lemma~\ref{lem:trunc-err} with $g\gets\gApprox[\bar n],m\gets [\bar n^{\frac{1}{b+1}}]$ leads to 
$$
\|\Proj{m}{\gApprox[\bar n]}\|_\cI \le \|\gApprox[\bar n]\|_\cI \lesssim \bar n^{\frac{1}{2(\beff+1)}}, ~~
\|\Proj{m}{\gApprox[\bar n]} - \gApprox[\bar n]\|_2^2 \lesssim \bar n^{-\frac{\beff}{2(\beff+1)}}. 
$$
Invoking Theorem~\ref{prop:approx} with $m\gets 2[\bar n^{\frac{1}{\beff+1}}], g^*\gets \Proj{m}{\gApprox[\bar n]}$ yields 
$$
\|\tilde g\|_{\tilde\cI} \le c_r \bar n^{\frac{1}{2(\beff+1)}}, ~~
\|\tilde g - g^*\|_2 \le c_r \bar n^{\frac{1}{2(\beff+1)}}(\xi_{n_1} + \bar n^{-1/2})\log n_1 + 0. %
$$
Two applications of triangle inequality yield
$$
\|\tilde g - g\|_2 \lesssim 
 \bar n^{\nicefrac{\beff}{2(\beff+1)^2}}(\xi_{n_1}+\bar n^{-\nicefrac{1}{2}})\log n + 
 \bar n^{-\nicefrac{\beff}{2(\beff+1)}}  = \tilde\cO\Big(
     \epsilon_{fea,n_1} + n_1^{-\nicefrac{(\beff^2-\epsilon)}{2(\beff+1)^2}}\Big),
$$
for all $\epsilon>0$. 
This result is most useful for functions satisfying the conditions in \eqref{eq:basic-sieve}, which may or may not be random GP prior draws. 
\end{enumerate}

In Case (ii) the average-case DNN rate is $\xi_{n_1}$, by definition. Case (iii) can thus be suboptimal when feature learning is easy, as the exponent of the second term above is worse by a multiplicative factor of $(\beff-\epsilon)/(\beff+1)$. Note the suboptimality can also be removed if we can take $n_1\gtrsim n_2^{\nicefrac{(\beff+1)}{\beff}}$, at which point estimation error starts to dominate the combined regression error. The difference in order diminishes as $\beff\to\infty$. 
The situation in Case (i) is less clear, since it is unclear if the DNN can take advantage of the improved regularity of $g^*\in\cI$, see below.\qed

\begin{remark}[optimal rate for $g^*\in\cI$]\label{rmk:xi-optimality-rkhs}
We discuss whether the DNN oracle in \cite{schmidt-hieber_nonparametric_2020} may estimate $g^*\in\cI$ with a rate better than $\xi_n$. 

The condition that $g^*\in\cI$, or equivalently $\bar g^*\in\bar\cI$, implies that $\bar g^*$ now possesses a better Sobolev regularity comparing with typical GP samples  (Example~\ref{ex:matern}-\ref{ex:matern-gp}), but it provides no additional information on H\"older regularity, which is needed in \cite{schmidt-hieber_nonparametric_2020}. If we additionally assume the H\"older regularity is also improved by the same amount (i.e., $\bar g^*\in C^{\underline{\beta}_2+(d_l/2)}(\bar\cZ)$), 
the DNN rate will improve to $\tilde\cO(\epsilon_{fea,n_1}+n_1^{-\nicefrac{\beff+1}{2(\beff+2)}})$, by \eqref{eq:s-h-raw-rate}; this is smaller than $\xi_{n_1}$ if feature learning is sufficiently easy. Without the extra H\"older regularity, we can only invoke the Sobolev embedding theorem to obtain 
$\bar g^*\in C^{\underline{\beta}_2}(\bar\cZ)$; the derived DNN rate can only be infinitesimally better than $\xi_{n_1}$. 
\end{remark}

\subsubsection{Alternative Kernel Learning Schemes}\label{rmk:iv-alternative} 

Let us compare the proposed algorithm with a more obvious alternative. For simplicity, we assume the feature learning term in $\xi_{n_1}$ does not dominate, and restrict to Case (i) in Example~\ref{ex:dnn-oracles-cont} ($g^*\in\cI$).

The established bound 
requires $O(n_1^{\beff/(\beff+1)^2})$ calls to the regression oracle.\footnote{Or equivalently, solving a vector-valued regression problem with a similar output dimension.}  
In the IV setting, we have $\bar b \ge \max\{2p-1,b\}$, and possibly $\bar b = b+2p$ (App.~\ref{app:regularity}). 
An alternative instrument learning procedure is to  %
construct a $d$-dimensional approximation to $\cH$, %
invoke the oracle on its basis functions, and use the estimates to construct an alternative $\tilde\cI$. 
For an approximation error of $n_1^{-(b+1)/2(b+2)}$, we need %
$d\asymp n_1^{1/(b+2)}$ bases \citep{rudi2015less,rudi2016generalization}, which becomes much worse than $n_1^{1/(\beff+2)}$ when $p$ is moderately large. Moreover, finding such an approximation is difficult: the same references note that intuitive choices, such as uniformly sampled random feature, or Nystr\"om inducing points, actually require $d\asymp n^{\frac{b}{b+1}}\to n$. 
It may be also possible that such an algorithm actually has worse performance, as we have not checked if its dependency on the oracle estimation error remains unchanged. %

Our approximation gracefully adapts to the ill-posedness of the problem, while also being simpler to implement. Intuitively, this is because it directly focuses on the approximation of the optimal $\cI$ as opposed to $\cH$. In other words, it works with a basis $\{\varphi_i\}\subset\cH$ that is not necessarily optimal for the approximation of $\cH$, but is designed so that $\{E\varphi_i\}$ approximate $\cI$ well. Interestingly, we do not require the knowledge of such a basis, instead automatically adapting to it. {\em This is made possible by the isotropy of the $\mc{GP}(0,k_x)$ prior}: the GP can be viewed as being defined by an aribtrary basis, convergence technicalities notwithstanding.

\section{Approximation and Estimation Results for Gaussian Process Regression}\label{app:gp-fixed-design-regr}

This section includes various technical results for using $\tilde\cI$ in the ``GP regime'', i.e., when the regression function does not live in $\tilde\cI$, but only satisfies an approximability condition like Asm.~\ref{ass:s2}~\ref{it:gp-scheme}. Regularization becomes weaker in the GP scheme than the kernel scheme; to see this, consider a regression task with data $\{(z_i, \bar y_i=g_0(z_i)+\epsilon_i):i\in [\bar n]\}$. Observe that both the KRR estimate and the GP posterior mean estimator can have the form of 
$$
\hat g_n := \arg\min_{g\in\tilde\cI}\frac{1}{\bar n}\sum_{i=1}^{\bar n} (g(z_i) - \bar y_i)^2 + \bar\nu \|g\|_{\tilde\cI}^2, 
$$
but the kernel scheme uses $\bar\nu\gg \bar n^{-1}$ (unless the RKHS has fixed finite dimensionality), to dominate the critical radius of the RKHS \cite{wainwright2019high}, whereas the GP scheme uses $\bar \nu\asymp \bar n^{-1}$. See \cite{van_der_vaart_information_2011,kanagawa_2018_gaussian} for more discussions. It is thus understandable that the GP scheme requires additional assumptions. 

We first impose the following assumption on the \emph{true RKHS} $\cI$: 
\begin{assumption}\label{ass:emb-general}
Let the Mercer eigenvalues of $\cI$ be $\lambda_i \asymp i^{-(\beff+1)}$. Then \eqref{eq:emb} holds for $\gamma = \frac{\beff-1}{\beff+1}$. 
\end{assumption}
Recall the ``top-level RKHS'' $\bar\cI$ and $\cI$ satisfy embedding properties of the same order. As discussed in Example~\ref{ex:matern}, this assumption holds when $\bar\cI$ is a Mat\'ern RKHS with $\beff> 2$, and its kernel is supported on a bounded, open subset of $\RR^{d_l}$ with smooth boundaries. 

The above assumption controls estimation error using the true RKHS $\cI$. To see this, observe Lemma~\ref{lem:emb} now implies 
\begin{equation}\label{eq:interpolation}
\|g\|_\infty \lesssim \|g\|_\cI^{\frac{\beff-1}{\beff+1}} \|g\|_2^{\frac{2}{\beff+1}},\quad\forall g\in\cI.
\end{equation}
Fast-rate convergence analyses typically require sup norm bounds on a ``relevant subset'' of the hypothesis space, such as $\{g-g_0: g\in\cI, \|g-g_0\|_2\text{ is small}\}$. 
The interpolation inequality \eqref{eq:interpolation} provides such a bound if we were using the true model $\cI$. 
Note that such assumptions are not needed in the kernel regularization scheme, because as discussed above, in that case the ridge regularizer has a greater magnitude, thus providing more control for the RKHS norm on the relevant subset.

However, we will use the approximate model $\tilde\cI$ for estimation, so to control similar sup norm quantities, we also need some crude bound on the sup-norm approximation error of $\tilde\cI$. This, in turn, requires the following crude sup-norm error bound for the regression oracle.
\begin{assumption}\label{ass:oracle-sup-norm-err}
On a high-probability event determined by $\data$, we have, for $m\asymp n^{\nicefrac{1}{\beff+1}}$, 
$
\sum_{j=1}^m \|\gpi{j} - \estiRaw{j}\|_\infty^2 \lesssim 1. 
$
\end{assumption}
The assumption requires sup norm error to scale at $n^{-\nicefrac{1}{2(\beff+1)}}$. We expect it to be mild, at least when $\beff$ is reasonably large (and feature selection is not too hard). This is because
in the classical kernel literature, the sup norm error rate for Sobolev kernels is $O(n^{-\nicefrac{(\beff-1)}{2(\beff+1)}})$ \citep{fischer2020sobolev}. Moreover, if the GP samples are $\nicefrac{(\beff-\epsilon) d}{2}$-H\"older regular as in the Mat\'ern example, error rates in $L_2$ and sup norm will both be $O(n^{-\nicefrac{(\beff-\epsilon)}{2(\beff+1-\epsilon)}})$. 

\subsection{Sup Norm Approximation}

All results in this subsection assume Assumption~\ref{ass:emb-general}. 

\begin{lemma}\label{lem:basic-sieve}
Let $\Pi_g$ denote the standard GP prior defined by $k_z$, and $\tau_n = n^{-\frac{\beff/2}{\beff+1}}$. There exist constants $C,C'>0$ s.t.~(i) for all $n\in\mb{N}$, 
\begin{equation}\label{eq:sieve}
\Pi_g(\Theta_n) := \Pi_g(\{g = g_h + g_e: \|g_h\|_\cI^2 \le C n\tau_n^2, \|g_e\|_2^2\le C\tau_n^2, \|g_e\|_\infty^2\le C n^{-\frac{1}{\beff+1}}\}) \ge 1 - e^{-C' n\tau_n^2}.
\end{equation}
(ii) In the above display, we can always choose $g_h = \Proj{m}{g}$, for $m=[n^{\frac{1}{\beff+1}}]$.
\end{lemma}
\begin{proof}
(i) is a consequence of \citet{steinwart_convergence_2019} and is proved in \citet[Corollary 13, with $b\gets \beff,b'\gets b-1$]{wang2021quasibayesian}. To show (ii), it suffices to verify that $$
\tilde g_h := \Proj{m}{g_h}+\Proj{m}{g_e}, ~~
\tilde g_e := \Proj{>m}{g_h}+\Proj{>m}{g_e}
$$ still satisfy the above display. 
For $\tilde g_h$, let $\lambda_m\asymp m^{-(\beff+1)}$ be the $m$-th Mercer eigenvalue for $k_z$. Then 
\begin{align*}
\|\tilde g_h\|_\cI &\le \|g_h\|_\cI + \|\Proj{m}{g_e}\|_\cI 
\le \sqrt{Cn\tau_n^2} + \lambda_m^{-1/2} \|\Proj{m}{g_e}\|_2 
\lesssim n^{\frac{1/2}{\beff+1}} = \sqrt{n\tau_n^2}. \\ 
\|\Proj{>m}{g_h}\|_2 &\le \lambda_m^{1/2}\|g_h\|_\cI  \lesssim \tau_n. 
\numberthis\label{eq:sieve-intermediate}\\ 
\|\tilde g_e\|_2 &\le \|g_e\|_2 + \|\Proj{>m}{g_h}\|_2 
\le C\tau_n + \lambda_m^{1/2}\|g_h\|_\cI \lesssim \tau_n. \\ 
\|\tilde g_e\|_\infty&\le \|g_e\|_\infty + \|\Proj{>m}{g_h}\|_\infty 
\overset{\eqref{eq:interpolation}}{\le} C n^{-\frac{1/2}{\beff+1}} + 
(\|\Proj{>m}{g_h}\|_\cI)^{\frac{\beff-1}{\beff+1}} (\|\Proj{>m}{g_h}\|_2)^{\frac{2}{\beff+1}}  \\ 
&\le
 Cn^{-\frac{1/2}{\beff+1}} + 
n^{\frac{1/2}{\beff+1}\cdot\frac{\beff-1}{\beff+1}} (\|\Proj{>m}{g_h}\|_2)^{\frac{2}{\beff+1}} \overset{\eqref{eq:sieve-intermediate}}{\lesssim} n^{-\frac{1}{2(\beff+1)}}.
\end{align*}
\end{proof}

\begin{corollary}\label{corr:sup-trunc-err-bound}
Let $\gpi{1},\ldots,\gpi{m}\sim\Pi_g$ be i.i.d.~samples from the GP prior. Then 
\begin{equation}\label{eq:pi-sup-trunc-error-bound}
\Pi_g\Bigg(\sum_{i=1}^m \|\Proj{>m}{\gpi{i}}\|_\infty^2 \ge C 
\Bigg) \le m e^{-C' m} \to 0.
\end{equation}
\end{corollary}
\begin{proof}
By Lemma~\ref{lem:basic-sieve} with $n := m^{\beff+1}$, and union bound.
\end{proof}

\begin{proposition}\label{prop:sup-norm-approx}
Suppose the conditions in Theorem~\ref{prop:approx} hold, and Assumption~\ref{ass:emb-general},~\ref{ass:oracle-sup-norm-err} hold. %
Then on a high-probability event determined by $\data$, 
for any $g^*\in\cI$, the approximation $\tilde g^*\in\tilde\cI$ constructed in Theorem~\ref{prop:approx} will also satisfy 
$$
\|\tilde g^* - g^*\|_\infty \le 
\|g^*-\Proj{m'}{g^*}\|_\infty + 
Cm^{-1/2}\|\Proj{m'}{g^*}\|_\cI,
$$
where the constant $C$ is independent of $g^*$.
\end{proposition}
\begin{proof}
The proof parallels those of Theorem~\ref{prop:approx} and Lemma~\ref{lem:suboptimal}. Recall that on the event defined in Theorem~\ref{prop:approx}, we have $\|\gpi{j}\|_\infty\lesssim\sqrt{\log n}$, and thus $\|\esti{j}-\gpi{j}\|_\infty \le \|\estiRaw{j}-\gpi{j}\|_\infty$. 
Also 
recall the definitions of $\etest,\tilde e$ therein, and $\tilde g^* = \etest^\top \hat\Psi$. 

Now we have 
\begin{align*}
\|\tilde g^* - g^*\|_\infty &\le \|\etest^\top(\hat\Psi-\Psi)\|_\infty + \|g^*-\Proj{m'}{g^*}\|_\infty 
= \|\tilde e(\bar G-\hat G)\|_\infty + \|g^*-\Proj{m'}{g^*}\|_\infty  \\
&\le \|\tilde e(\bar G-G)\|_\infty + \|\tilde e(G-\hat G)\|_\infty +
\|g^*-\Proj{m'}{g^*}\|_\infty. \\ 
\|\tilde e(G-\hat G)\|_\infty &= \sup_{z\in\cZ} \<\tilde e, (G-\hat G)(z)\>_2 
\le \|\tilde e\|_2 \Big(\sum_{j=1}^m \|\gpi{j}-\esti{j}\|_\infty\Big)^{1/2} \lesssim \|\tilde e\|_2.\quad\text{\color{gray}(Asm.~\ref{ass:oracle-sup-norm-err})} \\ 
\|\tilde e(G-\bar G)\|_\infty &
\le 
\|\tilde e\|_2 \Big(\sum_{j=1}^m \|\gpi{j}-\Proj{m}{\gpi{j}}\|_\infty\Big)^{1/2} \overset{\eqref{eq:pi-sup-trunc-error-bound}}\lesssim \|\tilde e\|_2. 
\end{align*}
Combining the above displays, \eqref{eq:tilde-e-norm}, and the fact that $\|\etest\|_2 = \|\Proj{m'}{g^*}\|_\cI$ (see proof for Theorem~\ref{prop:approx}) completes the proof.
\end{proof}
 
\subsection{Estimation in Fixed-Design Regression}

\subsubsection{Small-Ball Probability Bound} 

It is known that certain well-behaving entropy number bounds imply a small-ball probability bound \citep{li_approximation_1999}. For our purpose, however, we need to modify their proof, as our entropy number bound in Lemma~\ref{lem:entropy-number} is somewhat less regular. 

\begin{lemma}\label{lem:small-ball-prob}
Let $\tilde\cI$ be a finite-dimensional RKHS with a bounded reproducing kernel, supported on a bounded subset of $\RR^{d_z}$.  
$\iota: \cI\to L_2(P(dz))$ be the natural inclusion operator,
$\Pi = N(0, \iota\iota^*)$ be the standard GP prior (see Claim~\ref{lem:gp-gm}).
Suppose the average entropy number bound in Lemma~\ref{lem:entropy-number} hold.
Then there exists some constant $C>0$, such that for 
$$
\epsilon_{n_2} := 
C (n_2^{-\frac{\beff/2}{\beff+1}} \log n_2 + \chi_{n_1} \log^2 n_2) \lesssim 
(n_2^{-\frac{\beff/2}{\beff+1}} \log n_2 + (m^{-\frac{\beff+1}{2}}+\xi_{n_1})\sqrt{\log n_1} \log^2 n_2), 
$$
we have, on a high-probability event determined by $\data$ and $\Ztest$, 
$$
-\log \Pi(\{g: \|g\|_{n_2}\ge \epsilon_{n_2}\}) \le n_2\epsilon_{n_2}^2.
$$
\end{lemma}
\begin{proof}
Let $\iota_{n_2}: \tilde \cI \to L_2(Z^{n_2})$ be the inclusion operator, 
\SkipNOTE{The proof follows \citet{li_approximation_1999}, with their entropy number bound replaced, and Lemma~2.1 replaced by \citet[Theorem 3.2]{pajor_volume_1989}.}
$l_k(\iota_{n_2})$ be the $l$-approximation number. %
\citet[Lemma 2.1]{li_approximation_1999} states that for some univeral $c_1,c_2$,\footnote{We dropped the logarithm factor therein because $L_2(\Ztest)$ is $K$-convex %
\citep{pisier1979espaces}. \SkipNOTE{this ref shows the the convexity constant only depends on the Banach-Mazur distance of $L_2(P(\Ztest))$ to a Hilbert space of the same dim (see also \citep[end of section 1]{maurey03type}); the constant is thus 1.}}
$$
l_k(\iota_{n_2}) \le c_1\sum_{j\ge c_2 k} e_j(\iota_{n_2}^\top) j^{-1/2}\quad\forall k\in\mb{N}.
$$
\citet[Theorem 3.2]{pajor_volume_1989} states that, for some universal $c_3>0$, $$
e_{[c_3 j]}(\iota_{n_2}^\top) \le 2 e_j(\iota_{n_2})\quad\forall j\in\mb{N}.
$$
Combining the two results, and 
taking expectation w.r.t.~$\Ztest$, we have 
\begin{align*}
\EE_{\Ztest} l_k(\iota_{n_2}) &\le c_1\sum_{j\ge c_2' k} j^{-1/2} \EE_{Z^{n_2}} e_j(\iota_{n_2}) ,\quad \forall k\in\mb{N}.
\end{align*}
Plugging in Lemma~\ref{lem:entropy-number}, we can see that for $\chi_{n_1} = (m^{-\frac{\beff+1}{2}}+\xi_{n_1})\sqrt{\log n_1}$, 
\begin{align}
\EE_{\Ztest} l_k(\iota_{n_2}) 
&\lesssim \sum_{j\ge c_2' k} j^{-(\beff+\frac{3}{2})} 
(\min\{j, n_2\}^{\frac{\beff+1}{2}} + \min\{j,n_2\}^{\beff+\frac{1}{2}} \chi_{n_1}) 
\nonumber \\  &
= \sum_{j=c_2' k}^{n_2} j^{-(\beff+\frac{3}{2})} 
(j^{\frac{\beff+1}{2}} + j^{\beff+\frac{1}{2}} \chi_{n_1}) + 
\sum_{j> n_2} j^{-(\beff+\frac{3}{2})} 
(n_2^{\frac{\beff+1}{2}} + n_2^{\beff+\frac{1}{2}} \chi_{n_1})
\nonumber \\  &
\lesssim \Bigg(\sum_{j=c_2' k}^\infty j^{-\frac{\beff}{2}-1}
+\sum_{j=c_2' k}^{n_2} j^{-1} \chi_{n_1}\Bigg) + 
(n_2^{\frac{\beff+1}{2}} + n_2^{\beff+\frac{1}{2}}\chi_{n_1})n_2^{-(\beff+\frac{1}{2})} 
\nonumber \\  &
\lesssim k^{-\frac{\beff}{2}} + n_2^{-\frac{1}{2}} + \chi_{n_1}(1+\log n_2),\quad\forall k\in\mb{N} \label{eq:sbp-intermediate}
\end{align}
invoking Markov's inequality on \eqref{eq:sbp-intermediate},\footnote{
We note that the concentration should be much sharper: our entropy number bounds are based on tail sums of Gram matrix eigenvalues, which are $O(n_1^{-1})-$subgaussian 
\citep{shawe-taylor_eigenspectrum_2002}. For simplicity, however, we do not optimize the failure probability.
} with $k=[n_2]^{\frac{1}{\beff+1}}$, we have, for some $C>0$ and on a high-probability event determined by $\Ztest$, 
$$
\ell_k(\iota_{n_2}) \le C(n_2^{-\frac{\beff/2}{\beff+1}}\log n_2+\chi_{n_1}\log^2 n_2) = \epsilon_{n,n_2}.
$$
On this event we have 
\begin{equation}\label{eq:valid-sbp-event}
n(\epsilon_{n,n_2}) \le k \le n_2^{\frac{1}{b+1}},\quad\text{where}\quad
n(\epsilon) := \max\{k: 4 l_k(\iota_{n_2}) \ge \epsilon\}
\end{equation}
Our conditions about $\tilde\cI$ imply the GP prior has a Karhunen-Lo\`eve expansion; 
as $\tilde\cI$ is finite-dimensional, the Karhunen-Lo\`eve expansion always converge. Therefore, by \citet[Lemma 2.3, Proposition 2.3]{li_approximation_1999}, we have, for any $\epsilon>0$,  
$$
-\log \Pi(\{g: \|g\|_{n_2}\ge \epsilon\}) \le n(\epsilon) \log \frac{n(\epsilon)}{\epsilon}.
$$
Plugging \eqref{eq:valid-sbp-event} to the above completes the proof.
\end{proof}

\subsubsection{Regression with Fixed Design}

Consider the regression problem with fixed design:
\begin{equation}
    \bar Y_i = g_0(Z_i) + \varepsilon_i, \quad i = 1,2, \dots, n,
\end{equation}
where $Z_1, \dots, Z_n$ are fixed, $\varepsilon_1, \dots, \varepsilon_n$ are independent $1$-subgaussian random variables.
Define $p_0$ to be the distribution of $\bar Y := (\bar Y_1, \dots, \bar Y_n)$ and $p_g := \cN(g(Z), I_n)$ for any function $g$.
Let $\Pi_n$ be a GP prior and $\Theta$ be a parameter space such that $\Pi_n(\Theta) = 1$, we consider the fractional posterior for $\alpha \in (0, 1)$:
\begin{equation}
    \Pi_{n, \alpha}(A \mid Z) := \frac{\int_A [p_g(\bar Y)]^\alpha \Pi_n(dg)}{\int_\Theta [p_g(\bar Y)]^\alpha \Pi_n(dg)}.
\end{equation}
Define $r_n(g, g^\dagger) := \log p_{g^\dagger}(\bar Y) - \log p_g(\bar Y)$ and $B_n(g^\dagger) := \{ g \in \Theta : \EE_{p_0} r_n(g, g^\dagger) \leq n\epsilon^2_n, \mathrm{Var}_{p_0} r_n(g, g^\dagger) \leq n\epsilon^2_n \}$ and $D_\alpha^{(n)}(g, g^\dagger) := -\frac{1}{1 - \alpha} \log \EE_{p_0} \left( \frac{p_g}{p_{g^\dagger}} \right)^\alpha$. 
By invoking the contraction theorem of the fractional posterior under the misspecified setting (since $p_0$ is non-Gaussian), we have the following theorem.\footnote{In the original statement of \citet[Corollary 3.7]{bhattacharya_bayesian_2019}, $g^\dagger$ is the best KL-approximation of $p_0$ in $\Theta$. Following the same line of their proof, this corollary also holds for an arbitrary function $g^\dagger$ in our setting.}

\begin{theorem}[\citealp{bhattacharya_bayesian_2019}, Corollary 3.7]
    \label{thm:fp-contraction}
For any $n\in\mb{N}$, 
let $\Pi_n$ be an $n$-dependent prior, $\alpha\in(0,1)$ be arbitrary, $\epsilon_n \in (0,1)$ be such that $n\epsilon_n^2 >2$, $g^\dagger$ be a function (not necessarily in $\Theta$) such that
$$
-\log\Pi_n(B_n(g^\dagger)) \le n\epsilon_n^2.
$$
Then with $p_0$ probability at least $1-2/(n\epsilon_n^2)$, we have 
$$
\int\left\{\frac{1}{n} D_\alpha^{(n)}(g, g^\dagger)\right\} \Pi_{n,\alpha}(dg \mid Z) \le \frac{2\alpha+1}{1-\alpha}\epsilon_n^2.
$$
\end{theorem}

Let $\Pi_n$ be the GP determined by $\tilde \cI$, then the postrior mean is
\begin{equation}\label{eq:gpr-mean-estimator}
g^* := \int_\Theta g \Pi_{n, \alpha}(dg \mid Z) =  \argmin_{g\in\tilde\cI} \frac{1}{2\alpha n}\sum_{i=1}^{n} (\bar Y_i - g(Z_i))^2 + \frac{1}{n}\|g\|_{\tilde\cI}^2,
\end{equation}
Since $p_g$ is Gaussian and $p_0$ is subgaussian, we can explicitly compute $B_n(g^\dagger)$ and $D_\alpha^{(n)}$ in the above theorem and obtain the following corollary.

\begin{corollary}\label{corr:fixed-design-gpr-mean}
Let $g^*$ be the GPR posterior mean estimator as in \eqref{eq:gpr-mean-estimator}, for $\alpha=\frac{1}{2}$; and $\bar Y_i = g_0(Z_i) + \varepsilon_i$ where $\EE(\varepsilon_i\mid Z_i)=0$, and $\varepsilon_i$ is $1$-subgaussian. 
Let $\Pi_n$ be the standard GP prior determined by $\tilde\cI$, be arbitrary, $\epsilon_n \in (0,1)$ be such that $n\epsilon_n^2 >2$, and 
$$
-\log\Pi_n(\{g: \|g-g_0\|_{n}\lesssim \epsilon_n\}) \le n\epsilon_n^2.
$$
Then we have, for some universal constant $c>0$, 
$$
\PP(\|g^* - g_0\|_{n} \ge c\epsilon_n \mid Z) \le \frac{1}{2n\epsilon_n^2}.
$$
\end{corollary}
\begin{proof}
    By the subgaussian properties, $\EE_{p_0} \left ( \frac{p_g}{p_{g^\dagger}} \right )^\alpha$ can be bounded as follows:
    \[ \begin{aligned}
      &\phantom{{}={}}  
      \EE_{p_0} \exp\left( -\frac{n \alpha}{2} \left( \| \bar Y - g \|_{n}^2 -   \| \bar Y - g^\dagger \|_{n}^2 \right)\right) \\
      &= \exp\left( -\frac{n \alpha}{2} \left( \| g \|_{n}^2 -   \| g^\dagger \|_{n}^2 - 2 \< g_0, g - g^\dagger \>_{n} \right)\right) 
      \EE_{\varepsilon_1, \dots, \varepsilon_n} e^{\alpha\sum_{k=1}^n \varepsilon_k (g(Z_k) - g^\dagger(Z_k))} \\
      &\leq \exp\left( -\frac{n \alpha}{2} \left( \| g \|_{n}^2 -   \| g^\dagger \|_{n}^2 - 2 \< g_0, g - g^\dagger \>_{n} \right)\right) 
      e^{\frac{n\alpha^2}{2} \| g - g^\dagger \|_{n}^2 } \\
      &= \exp\left( -\frac{n \alpha}{2} \left( \left( 1 - \alpha \right)\| g - g^\dagger\|_{n}^2  - 2 \< g_0- g^\dagger g- g^\dagger\>_{n} \right)\right),
    \end{aligned} \]
    then it holds that
    \begin{align}
        D^{(n)}_{\alpha}(g, g^\dagger) &:= -\frac{1}{1 - \alpha} \log \EE_{p_0} \left ( \frac{p_g}{p_{g^\dagger}} \right )^\alpha \nonumber\\
        &\phantom{:}\geq 
        \frac{n \alpha}{2(1 - \alpha)}\left( \left( 1 - \alpha \right)\| g- g^\dagger\|_{n}^2  - 2 \< g_0- g^\dagger g- g^\dagger\>_{n} \right).
        \label{eqn:fp-alpha-divergence}
    \end{align}
    Similarly, since $r_n(g, g^\dagger) := \log p_{g^\dagger}(\bar Y) - \log p_g(\bar Y)$, then 
    \begin{align*}
        \EE_{p_0} r_n(g, g^\dagger) 
        &= \frac{n }{2} \left(\| g- g^\dagger\|_{n}^2  - 2 \< g_0- g^\dagger, g- g^\dagger\>_{n} \right), \quad
        \Var_{p_0} r_n(g, g^\dagger) 
        =  n \| g- g^\dagger\|^2_{n}.
    \end{align*}
    Setting $g^\dagger = g_0$, the set $B_n = \{ g \in \Theta : \EE_{p_0} r_n(g, g_0) \leq n\epsilon^2_n, \mathrm{Var}_{p_0} r_n(g, g_0) \leq n \epsilon_n^2 \}$ reduces to $\{ g \in \Theta : \| g - g^\dagger \|_{n} \leq \epsilon^2_n \}$, and thus by the assumption we know $-\log \Pi_n(B_n) \lesssim n\epsilon_n^2$, which fulfills the requirement in Theorem~\ref{thm:fp-contraction}.
    Therefore, the following holds with probability $1 - \frac{2}{n\epsilon^2_n}$
    \begin{equation}
        \int \| g - g_0 \|_{n}^2 \Pi_{n, \alpha}(d g \mid Z)
        \overset{\eqref{eqn:fp-alpha-divergence}}{\lesssim} \int D_\alpha^{(n)}(g, g_0) \Pi_{n, \alpha}(d g \mid Z)
        \lesssim \epsilon_n^2. 
    \end{equation}
    Finally, Jensen's inequality yields that 
    \begin{equation}
        \| g^* - g_0 \|_{n}^2 =
        \left \| \int g\Pi_{n, \alpha}(d g \mid Z) - g_0 \right \|_{n}^2 \leq
        \int \| g - g_0 \|_{n}^2 \Pi_{n, \alpha}(d g \mid Z)
         \lesssim \epsilon_n^2,
    \end{equation}
    which completes the proof.
\end{proof}
\section{Deferred Proofs: IV Regression}

\subsection{Proof for Proposition~\ref{prop:dikkala-regime}}\label{app:dikkala-proof}

We will use a slightly modified version of Theorem~1 in \citet{dikkala_minimax_2020}, which we state below with changed notations.

\newcommand{\hlchange}[1]{#1}

\begin{theorem}[\citealt{dikkala_minimax_2020}, adapted]\label{thm:dikkala}
Let $\cH,\cI$ be normed function spaces, such that functions in 
$\cH_{B_x}$ and $\cI_{3U}$ has bounded ranges in $[-1,1]$. 
Consider the estimator 
\begin{equation}\label{eq:dkiv-dikkala-estimator}
\hat f_{n} := \argmin_{f\in\cH} \max_{g\in\cI} \frac{1}{{n}}\sum_{i=1}^{n} (y_i - f(x_i))g(z_i) - \lambda\left(\|g\|_\cI^2+\frac{U}{\delta^2}\|g\|_{2,n}^2\right) + \mu\|f\|_\cH^2.
\end{equation}
Assume $f_0\in\cH$, and 
\begin{equation}\label{eq:dikkala-approx-requirement}
\forall f\in\cH: ~~ \min_{g\in\cI, \|g\|_\cI\le L\|f-f_0\|_\cH} \|g - E(f-f_0)\|_2 \le \eta_{n}\hlchange{\|f-f_0\|_\cH}. 
\end{equation}
Let $\delta = c_1 \delta_{n} + c_2 \sqrt{\frac{\log (c_3/\zeta)}{{n}}}$, where $c_1,c_2,c_3>0$ are universal constants, and $\delta_{n}$ is an upper bound for the critical radii of $\cI_{3U}$ and the function space\footnote{We dropped a scaling in its definition since our $\cH$ and $\cI$ are star-shaped.} $$
\cG := \{(x,z)\mapsto (f-f_0)(x) g_f(z): f-f_0\in\cH_{B_x}, 
g_f = \arg\min_{g\in\cI_{L^2 B_x}} \|g_f-E(f-f_0)\|_2
\}.
$$
If $\lambda> \delta^2/U, \mu > 2\lambda(\hlchange{1+}4L^2+27U/B_x)$,
we will have, w.p.~$1-3\zeta$: 
$$
\|E(\hat f_{n} - f_0)\|_2 \le \left(1025\delta+\hlchange{\eta_{n}}+\frac{(3U+54B_x^{-1}U+8L^2+\hlchange{2})\lambda+\mu}{\delta}\right)\max\{1,\|f_0\|_\cH^2\}. 
$$
\end{theorem}

\paragraph{Proof for the modified theorem}
The only change we make is in \eqref{eq:dikkala-approx-requirement}, where we added the factor $\|f-f_0\|_\cH$. We can see that the only use of its original version in \citet{dikkala_minimax_2020} is on $f\gets \hat f_n$ (p.~54 therein), so their proof will continue to hold after replacing $\eta_n$ with $\eta_n\|\hat f_n-f_0\|_\cH$. 
Therefore, by the last display on their p.~54:\footnote{We make the constant hidden in their big-O notation explicit.}
$$
\begin{aligned}
\frac{\delta}{2}\|E(\hat f_n-f_0)\|_2 &\le 
1025\delta^2 + 
\delta\eta_n\hlchange{\|\hat f_n-f_0\|_\cH} + 
3\lambda U 
+ 2\sup_{g\in\cI}\Psi^{\bar\nu/2}(f_0, g) \\ 
&\phantom{=} + 27\delta^2\frac{\|\hat f_n-f_0\|_\cH^2}{B_x} + 
4\lambda L^2    \|\hat f_n - f_0\|_\cH^2 + \mu(\|f_0\|_\cH^2-\|\hat f_n\|_\cH^2).
\end{aligned}
$$
Since $\lambda>\max\{\delta^2/U\hlchange{,\delta\eta_n}\}$, the sum of \hlchange{the second} and last three terms in RHS are bounded by 
$$
\begin{aligned}
&\hlchange{\delta\eta_n} + \lambda\left(\hlchange{1+}\frac{27}{B}+4L^2\right)\|\hat f_n - f_0\|_\cH^2 + \mu(\|f_0\|_\cH^2 - \|\hat f_n\|_\cH^2) \\
\le &\hlchange{\delta\eta_n}+2\lambda\left(\hlchange{1+}\frac{27}{B}+4L^2\right)(\|\hat f_n\|_\cH^2 + \|f_0\|_\cH^2) 
+ \mu(\|f_0\|_\cH^2 - \|\hat f_n\|_\cH^2).
\end{aligned}
$$
Since $\mu \ge 2\lambda(\hlchange{1+}\frac{27U}{B_x}+4L^2)$, the latter is bounded by 
$\hlchange{\delta\eta_n+}
(2\lambda(27B_x^{-1} U+4L^2+\hlchange{1})+\mu)\|f_0\|_\cH^2.
$
Following the next two displays on their p.~55, we have 
$$
\frac{\delta}{2}\|E(\hat f_n-f_0)\|_2 \le 2\sup_{g\in\cI}\Psi^{\bar\nu/2}(f_0,g) + 
1025\delta^2 + 3\lambda U + 
\hlchange{\delta\eta_n+}
(2\lambda(27B_x^{-1} U+4L^2+\hlchange{1})+\mu)\|f_0\|_\cH^2.
$$
By our assumption that $f_0\in\cH$, their subsequent upper bound for $\Psi^{\bar\nu/2}$ becomes $\sup_{g\in\cI}\Psi^{\bar\nu/2}(f_0, g)\le 0$. Thus 
$$
\|E(\hat f_n-f_0)\|_2 \le 
1025\delta + 3\frac{\lambda U}{\delta} + 
\hlchange{\eta_n+}
\frac{2\lambda(27B_x^{-1} U+4L^2+\hlchange{1})+\mu}{\delta}\|f_0\|_\cH^2.
$$
This completes the proof.\hfill\qedsymbol

\paragraph{Proof for Proposition~\ref{prop:dikkala-regime}}
Let us set $n \gets n_2, \cI \gets \tilde\cI, \delta\asymp n^{-\frac{b+1}{b+2}}$, $B_x \asymp 1, U\asymp (\log n_1)^{-1}$, $
\lambda \gets \delta^2/U, \mu \gets 2\lambda(1+4L^2+27U/B_x), 
$ where the constants hidden in $\asymp$ are determined by Assumption~\ref{ass:s2}, \ref{ass:s1-rkhs} and Theorem~\ref{prop:approx}, and are independent of any sample size.
Then the estimator \eqref{eq:dkiv-dikkala-estimator} becomes
\begin{equation}\label{eq:dkiv-dikkala-estimator-instantiated}
\hat f_{n_2} := \arg\min_{f\in\cH} \max_{g\in\tilde\cI} \frac{1}{n_2}\sum_{i=1}^{n_2}
(y_i-f(x_i)-g(z_i))g(z_i) - \lambda \|g\|_\cI^2 + \mu \|f\|_\cH^2.
\end{equation}
We claim that with $\delta, \eta_{n_2}, L$ set as below, the conditions in Theorem~\ref{thm:dikkala} will now hold in our setting, 
conditioned on the high-probability events in Theorem~\ref{prop:approx}. This is because:
\begin{enumerate}[leftmargin=*]
\item %
By Proposition~\ref{prop:est}, the squared critical radius for $\tilde\cI_{3U}$ is bounded by 
$\tilde O(n_2^{-\frac{\beff+1}{\beff+2}} + \xi_{n_1} + m^{-(\beff+1)})
$, for $\cG$ we consider its entropy number in the empirical $L_2$ norm (Definition~\ref{defn:entropy-number}). Denote by $L_2(\Ztest)$ the empirical $L_2$ space. Observe 
$$
e_{j+1}(\cG, L_2(\Ztest)) \le U^{-1} L^2 B_x\cdot e_j(\cH_{B_x}, L_2(\Ztest)) + e_j(\tilde\cI_{L^2 B_x}, L_2(\Ztest)), 
$$
since we can combine the coverings in the RHS to obtain a covering for $\cG$. By \citet[Exercise 7.7]{steinwart2008support}, we have 
$$
\EE_{\Ztest}\; e_j(\cH_{B_x}, L_2(\Ztest)) \lesssim B_x j^{-(b+1)} \min\{j,n_2\}^{\frac{b+1}{2}},
$$
Combining the above, Lemma~\ref{lem:entropy-number}, and the fact that $\beff\ge b$, we have 
$$
\EE_{\Ztest}\; e_j(\cG, L_2(\Ztest)) \lesssim
j^{-(b+1)} \min\{j,n_2\}^{\frac{b+1}{2}}\log n_1 + 
j^{-(\beff+1)}\log n_1
    \min\{j,n_2\}^{\beff+\frac{1}{2}} \chi_{n_1}
    ,
$$
where $\chi_{n_1}$ is defined therein. 
As the above bound has a similar structure to Lemma~\ref{lem:entropy-number}, we can 
repeat the proof for Proposition~\ref{prop:est} %
and find that
$
\delta_{n_2}^2 = \tilde O(n_2^{-\frac{b+1}{b+2}} + \xi_{n_1} + m^{-(\beff+1)}
)
$
\item The boundedness condition for $\cH_{B_x}$ is satisfied by Assumption~\ref{ass:s2}; that for $\cI_{3U}$ is verified in Section~\ref{sec:main}.
\item It remains to determine $L$ and $\eta_n$ in \eqref{eq:dikkala-approx-requirement}. We claim \eqref{eq:dikkala-approx-requirement} will hold by setting
$$
L = c_\tau, ~~
\eta_{n_2} \asymp (\xi_{n_1} + m^{-\frac{\beff+1}{2}})\log n_1.
$$
To prove this, observe that the true $\cI$ defined in Lemma~\ref{lem:iv-to-kl-1} would satisfy \eqref{eq:dikkala-approx-requirement} with $L=1,\eta_{n_2}\equiv 0$; for our approximation $\tilde\cI$, by Theorem~\ref{prop:approx} applied to $g^*=E(f-f_0)$, we know there exists some $\tilde g$
\begin{align*}
\|\tilde g - E(f-f_0)\|_{\tilde\cI} &\le c_\tau \|E(f-f_0)\|_\cI \le c_\tau \|f-f_0\|_\cH, \\
\|\tilde g - E(f-f_0)\|_2 &\le c_\tau \|g^*\|_\cI (\xi_{n_1} + m^{-\frac{\beff+1}{2}})\log n + \|g^* - \Proj{m}{g^*}\|_2  \\ 
&\overset{\eqref{eq:trunc-err}}{\lesssim}
\|g^*\|_\cI [(\xi_{n_1} + m^{-\frac{\beff+1}{2}})\log n_1 +  m^{-\frac{\beff+1}{2}}] 
\\ &
\le 2 
\|f-f_0\|_\cH (\xi_{n_1} + m^{-\frac{\beff+1}{2}})\log n_1.
\end{align*}
\end{enumerate} 
Now all conditions in the theorem are fulfilled, and for $n_1\ge n_2, m\ge n_2^{\frac{1}{b+2}}$, we get the convergence rate of 
$$
\begin{aligned}
\|E(\hat f_n - f_0)\|_2 &\le 
 \left(1025\delta+{\eta_{n_2}}+\frac{(3U+54B_x^{-1}U+8L^2+{2})\lambda+\mu}{\delta}\right)\max\{1,\|f_0\|_\cH^2\}  
 \\ &
= \tilde O\Big(
    (n_2^{-\frac{b+1}{2(b+2)}} + \xi_{n_1}) \max\{1, \|f_0\|_\cH^2\}  
\Big),
\end{aligned}
$$
as claimed.\hfill\qedsymbol
\subsection{Assumptions and Setup in Theorem~\ref{prop:qb-regime}}\label{app:qb-ass}

We have imposed Assumptions~\ref{ass:emb-general}, \ref{ass:oracle-sup-norm-err}. Asm.~\ref{ass:emb-general} accounts for our different assumption about $\cI$, comparing with \cite{wang2021quasibayesian} (see App.~\ref{app:qb-proof} below), and is satisfied by Mat\'ern kernels with suitable orders; Asm.~\ref{ass:oracle-sup-norm-err} requires the regression oracle to satisfy a mild sup norm error bound, which is used to control the sup norm approximation error of $\tilde\cI$. App.~\ref{app:gp-fixed-design-regr} discusses these assumptions in detail. 
We also introduce the following two assumptions, both taken from \cite{wang2021quasibayesian}: 

The following assumption is widely used in the NPIV literature \citep{hall_nonparametric_2005,blundell_semi-nonparametric_2007,horowitz_applied_2011}. It connects estimation error for $Ef_0$ to that for $f_0$.
\begin{assumption}[link condition]\label{ass:qb-1}
Let $\{\bar\varphi_i: i\in\mb{N}\}$ denote the Mercer eigenfunctions of $k_x$. Then we have, for all $f\in L_2(P(dx))$ and $j\in\mb{N}$, $$
j^{-2p}\|\Proj{j}f\|_2^2 \lesssim
\|E f\|_2^2 \lesssim \sum_{i=1}^\infty i^{-2p}\<f,\bar\varphi_i\>_2^2.
$$
\end{assumption}

The following assumption first appears in a literature that analyzes kernel ridge regression in a ``hard learning'' scenario \cite{steinwart2009optimal,fischer2020sobolev}. As discussed in \cite{wang2021quasibayesian}, it accounts for the different regularity between ``typical'' GP prior draws and the RKHS \citep{van_der_vaart_information_2011}, which makes GP modeling to fall into this scenario. If it is known that $f_0\in\cH_0$ for some RKHS $\cH_0$ with a bounded kernel and eigendecay $\lambda_i(T_{\cH_0})\lesssim i^{-b}$, in GP modeling we should specify as $\cH$ the power RKHS $\cH_0^{\nicefrac{b+1+\epsilon}{b}}$ (Defn.~\ref{defn:power-spaces}), so that both (an infinitesimally deteriorated version of) the GP scheme of Asm.~\ref{ass:s2} and this assumption can hold \cite{wang2021quasibayesian}.\footnote{The deterioration by $\epsilon$ can be removed if $[\cH_0]^{1-\epsilon}$ can be embedded into $L_\infty(P(dx))$.} 
As discussed in Example~\ref{ex:matern}, this assumption is satisfied by all Mat\'ern kernels satisfying Asm.~\ref{ass:s2} \ref{it:gp-scheme}, given the requirement for $P(dx)$ therein.

\begin{assumption}[embedding property]\label{ass:qb-n} For some (arbitrarily small) $\epsilon>0$, \eqref{eq:emb} holds for $\cH$ with $\gamma = \frac{b-\epsilon}{b+1}$. 
\end{assumption}

The quasi-posterior is defined by a standard GP prior $\Pi = \mc{GP}(0,k_x)$, and its Radon-Nikodym derivative w.r.t.~the prior:
\begin{equation}\label{eq:quasi-posterior}
\frac{\Pi(df\mid\dataSII)}{\Pi(df)} \propto e^{-\frac{n_2}{\lambda}\ell_{n_2}(f)},~~\text{where}~~ 
\ell_{n_2}(f) := \sup_{g\in\tilde\cI}\Big(\sum_{i=1}^{n_2}2(f(x_i)-y_i)g(z_i)-g(z_i)^2\Big)-\nu\|g\|_{\tilde\cI}^2.
\end{equation}
In the above, $\lambda,\nu\asymp 1$ are scaled ridge regularizers, and $\tilde\cI$ is constructed as in Algorithm~\ref{alg:main}. 
As the log quasi-likelihood $\ell_{n_2}$ is quadratic in $f$, we can check the corresponding posterior mean estimator has a similar form to \eqref{eq:dkiv-dikkala-estimator}, with $\lambda,\delta^2\gets (2n_2)^{-1}\nu, U\gets \frac{1}{2}$ and $\mu\gets (2n_2)^{-1}\lambda$. Note this is an invalid choice for Proposition~\ref{prop:dikkala-regime} which requires $\delta^2,\lambda,\mu$ to be $\tilde \Theta(n_2^{-\nicefrac{b+1}{b+2}})$, demonstrating the difference in regularization scale.

\subsection{Proof for Theorem~\ref{prop:qb-regime}}\label{app:qb-proof}

We will modify the proof for Theorem~3 in \cite{wang2021quasibayesian} to allow for our different assumptions about $\cI$, and account for the approximation error in $\tilde\cI$. For the former, note that in \cite{wang2021quasibayesian} $\cH$ is in the GP scheme, but $\cI$ is in the kernel scheme: it has eigendecay $\lambda_i\asymp i^{-(b+2p)}$ and contains the image of the power RKHS $\cH^{\nicefrac{b}{b+1}}$ under $E$ (Assumption 7 therein). In contrast, our $\cI$ is also in the GP scheme, having the eigendecay of $i^{-(b+2p+1)}$ and only containing the image of $\cH$.

Still, all assumptions in \cite{wang2021quasibayesian} about $\cH$ and $E$ are equivalent to ours, so their technical lemmas that do not involve $\cI$ continue to hold. 
This leaves us with the final proofs of their Proposition 20 and the Theorem 3, which include the only occurrences of $\cI$. We will address them in turn. 

Throughout the proof, the denotation of the constants ($C,C',\ldots$) may change from line to line.

\subsubsection{Replaced Denominator Bound}

This subsection replaces Proposition 20 in \cite{wang2021quasibayesian}. For all $n\in\mb{N}$, define $\delta_n := n^{-\frac{b+2p}{2(b+2p+1)}}$. 

\begin{lemma}[Local sub-Gaussian complexity]\label{lem:gaussian-complexity}
Let $\bar V^{n_2} := \{\bar v_i: i\in [n_2]\}$ be a set of $1$-bounded rvs which are conditionally independent given $Z^{n_2}$, and have zero conditional mean. 
Let $\cG(\tilde\cI_1;\delta)=\EE_{\bar V^{n_2}}\sup_{g\in\tilde\cI_1,\|g\|_{n_2}\le\delta}\left|\frac{1}{n_2}\sum_{i=1}^{n_2}g(z_i)\bar v_i\right|$. %
Then there exists $C_1>0$ s.t.~for any $\delta\ge n_2^{-1/2}$, we have 
$$
\PP_{\Ztest}(\cG(\tilde\cI_1;\delta) \le C_1 n^{-\frac{1}{2}} \delta^{\frac{b+2p}{b+2p+1}}) \ge 1 - e^{-n_2^{\frac{1}{b+2p+1}}/\log^4 n_1}.
$$
\end{lemma}
\begin{proof}
Recall $\bar\cI$ as defined in Lemma~\ref{lem:larger-kernel}. As $\tilde\cI_1\subset\bar\cI_1$, it suffices to establish the above for $\bar\cI_1$. 

Let $\{\hat\lambda_j\},\{\lambda_j\}$ denote the eigenvalues of the integral operators w.r.t.~the empirical and population measure. 
Let $m := [(\delta^2)^{-\frac{1}{b+2p+1}}]$. By Lemma~\ref{lem:larger-kernel} and our choice of $n_1$, we have 
$$
\sum_{j=m}^{n_2}\lambda_j \lesssim m^{-(b+2p)} + \chi_{n_1}^2 \le (\delta^2)^{\frac{b+2p}{b+2p+1}} + n_2^{-1}.
$$
 By 
\citet[Theorem 5, Proposition 2]{shawe-taylor_eigenspectrum_2002}, 
$$
\PP_{\Ztest}\Bigg(
\sum_{j=m}^{n_2} \hat\lambda_j \ge \sum_{j=m}^{n_2} \lambda_j + \epsilon
\Bigg) \le \exp(-2n_2\epsilon^2/\bar R^4),
$$  
where $\bar R:=\sup_z \bar k(z,z)$. 
By the construction of $\bar k$, \SkipNOTE{we need the m components of G-hat G, G-bar G and G to have log n sup norm}
boundedness of $\{\hat g^{(i)}_{n_1}\},\{g^{(i)}_{n_1}\}$ (see \eqref{eq:borell-sup}), and Corollary~\ref{corr:sup-trunc-err-bound},  
we can see that on the high $\PP_{\dataSI}$-probability event the theorem conditioned on, 
$\bar R \lesssim \log n_1$.   
Combining the two displays above, and recalling $\delta^{2}\ge n_2^{-1}$, we have, for some $C_0>2$,  
$$
\begin{aligned}
\PP_{\Ztest}\Biggl(
\sum_{j=m}^{n_2} \hat\lambda_j \le C_0 (\delta^2)^{\frac{b+2p}{b+2p+1}}
\Biggr) &\ge 
\PP_{\Ztest}\Biggl(
\sum_{j=m}^{n_2} \hat\lambda_j \le \sum_{j=m}^{n_2}\lambda_j + (C_0-2) (\delta^2)^{\frac{b+2p}{b+2p+1}}
\Biggr) 
\\ &\ge 
1 - e^{-n_2^{\frac{1}{b+2p+1}}/\log^4 n_1}.
\end{aligned}
$$
Plugging to \citet[Theorem 13.22; the proof holds for all bounded rvs]{wainwright2019high}, we have, on the above event, 
\begin{align*}
\cG(\bar\cI_1;\delta) &\le \sqrt{\frac{2}{n_2}}\sqrt{\sum_{j=1}^{n_2} \min\{\delta^2, \hat\lambda_j\}} 
\le \sqrt{\frac{2}{n_2}}\sqrt{m \delta^2 + 
\sum_{j=m}^{n_2}  \hat\lambda_j} 
\le C_1 n_2^{-\frac{1}{2}} \delta^{\frac{b+2p}{b+2p+1}}.
\end{align*}
\end{proof}

\begin{lemma}[KRR norm bound]\label{lem:krr-norm-bound}
Let $V^{n_2} = \{v_i: i\in [n_2]\}$ be a set of $B$-bounded rvs which are conditionally independent given $Z^{n_2}$, and have zero conditional mean. Let $\hat g$ be the KRR estimate: 
$
\hat g = \arg\min_{g\in\tilde\cI} \sum_{i=1}^{n_2}(g_0(z_i)+v_i - g(z_i))^2 + \nu \|g\|_{\tilde\cI}^2. 
$
Let $E_n(Z^{n_2},V^{n_2})$ denotes the event on which
\begin{enumerate}[leftmargin=*]
    \item There exists 
$\tilde g\in\tilde\cI$ {\em determined by $Z^{n_2}$}, s.t.~$\|\tilde g\|_{\tilde\cI}^2\lesssim n_2^{\frac{1}{b+2p+1}}\log n_2, \|\tilde g-g_0\|_{n_2}^2\lesssim \delta_{n_2}^2\log n_2$. 
\item $\hat g$ satisfies $\|\hat g-g_0\|_{n_2}^2 \lesssim \delta_{n_2}^2\log n_2$. 
\end{enumerate}
Then, on the intersection of $E_n$ and an event with probability $1-\gamma_{n_2} \to 1$, we have
$$
\|\hat g\|_{\tilde \cI}^2 \lesssim B^{\frac{b+2p+1}{b+2p+1/2}} n_2^{\frac{1}{b+2p+1}}\log n_2.
$$
\end{lemma}
\begin{proof}
With an abuse of notation, we denote $\<g, \epsilon\>_{n_2} := \frac{1}{n_2}\sum_{i=1}^{n_2} g(z_i) \epsilon_i$. We now proceed in two steps:

\underline{(Step 1)} We build a peeling argument. By \citet[Theorem 3.24]{wainwright2019high}, we have, for some $C_1,C_2>0$ and any $\delta>0$, 
\begin{align*}
\PP_{V^{n_2}}\biggl(\sup_{g\in\tilde\cI_1,\|g\|_{n_2}\le \delta} |\<g, \epsilon\>_{n_2}| 
&\ge \cG(\tilde\cI_1;\delta) + \frac{1}{2}C_1 B n_2^{-\frac{1}{2}} \delta^{\frac{b+2p}{b+2p+1}} 
\biggr)  \le \exp(-C_2 \delta^{-\frac{2}{b+2p+1}}).
\end{align*}
Combining with Lemma~\ref{lem:gaussian-complexity}, we have, for any $\delta\in (n_2^{-1/2}, n_2^{-\frac{b+2p}{2(b+2p+1)}})$, 
$$
\begin{aligned}
&\phantom{=}\PP_{V^{n_2}}\biggl(\sup_{g\in\tilde\cI_1,\|g\|_{n_2}\le \delta} |\<g, \epsilon\>_{n_2}| \ge C_1 B n_2^{-\frac{1}{2}} \delta^{\frac{b+2p}{b+2p+1}} 
\biggr) \\ &
\le \exp\Big(- (\log^{-4} n_1) n_2^{\frac{1}{b+2p+1}}\Big) + \exp\Big(
-C_2 n_2^{\frac{b+2p}{(b+2p+1)^2}}\Big) =: \eta_{n_2}^{(0)}.
\end{aligned}
$$
With an union bound over $\{\delta = e^j n_2^{-1/2}: 0\le j\le \log n_2^{\frac{1/2}{b+2p+1}}\}$, we have, with probability $\ge 1-\eta_{n_2}^{(0)} \log n_2 \to 1$, 
$$
\sup\Big\{|\<g,\epsilon\>_{n_2}|: g\in\tilde\cI_1, \|g\|_{n_2} \le n_2^{-\frac{b+2p}{2(b+2p+1)}}\Big\} 
\le e C_1 B n_2^{-\frac{1}{2}} \max\{n_2^{-\frac{1}{2}},\|g\|_{n_2}\}^{\frac{b+2p}{b+2p+1}}.
$$
Applying the above to $\frac{1}{\max\{\|g\|_{\tilde\cI},1\}} g$:
\begin{equation}\label{eq:g-norm-1}
    \begin{split}
\sup\Big\{|\<g,\epsilon\>_{n_2}|&: g\in\tilde\cI, \|g\|_{n_2} \le n_2^{-\frac{b+2p}{2(b+2p+1)}}\Big\} 
\le \\
&e C_1 B \max\{n_2^{-\frac{1}{2}}\|g\|_{n_2}^{\frac{b+2p}{b+2p+1}}\|g\|_{\tilde\cI}^{\frac{1}{b+2p+1}}, 
n_2^{-\frac{2(b+2p)+1}{2(b+2p+1)}} \|g\|_{\tilde\cI}\}.
    \end{split}
\end{equation}
In particular, %
this applies to $(C\log n_2)^{-1}(\tilde g-\hat g)$ for some $C>0$.

\underline{(Step 2)} As $\hat g$ minimizes the empirical loss, we have
$$
\nu \|\hat g\|_{\tilde\cI}^2 \le 
n_2 \|\hat g-g_0\|_{n_2}^2 + \nu \|\hat g\|_{\tilde\cI}^2 \le 
n_2\|\tilde g-g_0\|_{n_2}^2 + \nu\|\tilde g\|_{\tilde\cI}^2 + 
2n_2\<\epsilon, \hat g-\tilde g\>_{n_2}. 
$$
Plugging in our conditions, we have
\begin{align*}
\|\hat g\|_{\tilde\cI}^2 &\le 
C_3 n_2^{\frac{1}{b+2p+1}}\log n_2 + 2n_2\<\epsilon, \hat g-\tilde g\>_{n_2} \\ 
&\overset{\eqref{eq:g-norm-1}}{\le} 
C_3'(n_2^{\frac{1}{b+2p+1}}\log n_2 + n_2 B \max\{n_2^{-\frac{1}{2}}\|\hat g-\tilde g\|_{n_2}^{\frac{b+2p}{b+2p+1}}\|\hat g-\tilde g\|^{\frac{1}{b+2p+1}}_{\tilde\cI}, 
n_2^{-\frac{2(b+2p)+1}{2(b+2p+1)}}\|\hat g-\tilde g\|_{\tilde\cI}\}, \\ 
\|\hat g-\tilde g\|_{\tilde\cI}^2 &\le 2(\|\hat g\|_{\tilde\cI}^2 + \|\tilde g\|_{\tilde\cI}^2) \\ 
&\le C_3''(n_2^{\frac{1}{b+2p+1}}\log n_2 + n_2 B \max\{n_2^{-\frac{1}{2}}\|\hat g-\tilde g\|_{n_2}^{\frac{b+2p}{b+2p+1}}\|\hat g-\tilde g\|^{\frac{1}{b+2p+1}}_{\tilde\cI}, 
n_2^{-\frac{2(b+2p)+1}{2(b+2p+1)}}\|\hat g-\tilde g\|_{\tilde\cI}\}.
\end{align*}
By enumerating the dominating term and taking the maximum of the implied bounds, we conclude that 
$$
\|\hat g-\tilde g\|_{\tilde\cI}^2\le 2C_3'' B^{\frac{b+2p+1}{b+2p+1/2}} n_2^{\frac{1}{b+2p+1}} \log n_2.
$$
Another triangular inequality completes the proof.
\end{proof}

\begin{lemma}[Bernstein's inequality]\label{lem:bernstein-emp-norm}
For any %
function $g$ with $\|g\|_\infty\le B,\|g\|_2\le s$, %
$$
\PP_{\Ztest}(\|g\|_{n_2}^2 \le 5 s^2) \ge 1-e^{-n_2s^2/B^2}.
$$
\end{lemma}
\begin{proof}
The random variable $g(\bz)^2$ is $\|g\|_\infty^2$ bounded and has mean $\|g\|_2^2$ and variance $\le \EE g(\bz)^4 \le \|g\|_\infty^2 \|g\|_2^2$. The claim follows by Bernstein's inequality \citep[e.g.,][Theorem 6.12, with $\xi_i\gets g(z_i)^2-\|g\|_2^2$]{steinwart2008support}.  
\end{proof}

\begin{lemma}[replaces Lemma 19 in \citealp{wang2021quasibayesian}]\label{lem:replaced-den-fset}
There exist constants $C_1,C_2>0$, so that for all $n\in\mb{N}$, there exists a function set $\Theta_{d0}$ with prior mass 
$
\Pi(\Theta_{d0}) \ge e^{-C_1 n\delta_n^2}
$, s.t.~for all $f\in\Theta_{d0}$, we have 
\begin{equation}\label{eq:den-f-defn}
\|E(f-f_0)\|_2^2 \le C_2\delta_n^2, ~~
\|f-f_0\|_\infty^2 \le C_2.
\end{equation}
Moreover, for such $f$ we have $f-f_0=f_h+f_e$, where
\begin{equation}\label{eq:den-f-impl}
\|f_h\|_\cH^2 \le C_2 n\delta_n^2, ~~
\|E f_e\|_2^2 \le C_2 \delta_n^2, ~~ \|f_e\|_\infty^2 \le C_2.
\end{equation}
\end{lemma}
\begin{proof}
Let $\Theta_{d0}$ be defined as \citet[proof for Lemma 19]{wang2021quasibayesian}. As stated in that lemma, \eqref{eq:den-f-defn} holds. For such $f$, let $f_h'+f_e'=f-f_0$ be defined as in their proof, so that 
$\|f_h'\|_{\cH}^2 \lesssim n\delta_n^2, \|f_e'\|_2^2 \lesssim n^{-\frac{b}{b+2p+1}}, \|f_e'\|_\infty\lesssim 1$; applying their Lemma 16 to $f_h'$ and $f_e'$ shows the existence of $f_h+f_e=f-f_0$ satisfying \eqref{eq:den-f-impl}. 
\end{proof}

\begin{proposition}[replaces Proposition 20 in \citealp{wang2021quasibayesian}]\label{prop:qb-den}
For some $C_{den}>0,\eta'_{n_2}\to 0$ we have 
\begin{equation}\label{eq:qb-den}
\PP_{\dataSII}\left(
\int e^{-n_2\ell_{n_2}(f)}\Pi(df) \ge \exp\left(-C_{den} \delta_{n_2}^2 n_2\log^2 n_2\right)
\right) \ge 1-\eta'_{n_2} \to 1.
\end{equation}
\end{proposition}
\begin{proof} We proceed in two steps. 

\underline{(Step 1)} This step replaces 
Step 2 of the proof %
in \citet{wang2021quasibayesian}; it will establish that  
for some $C_5>0$ and a sequence $\eta_n\to 0$, 
\begin{equation}\label{eq:qb-den-s2-goal}
\inf_{f\in\Theta_{d0}}\PP_{\dataSII}(\{\ell_{n_2}(f) \le C_5\|E(f-f_0)\|_2^2 + C_5\delta_{n_2}^2\log^2 n_2\}) \ge 1-\eta_{n_2}.
\end{equation}
Recall
$
\ell_{n_2}(f) = \sup_{g\in\tilde\cI} \frac{1}{n_2}\sum_{i=1}^{n_2}(2(f(x_i)-y_i)g(z_i) - g(z_i)^2) - n_2^{-1}\nu\|g\|_{\tilde\cI}^2
$
and $\nu\asymp 1$. {With some algebra,\footnote{See e.g.~the beginning of Appendix C.1.3 in \citet{wang2021quasibayesian} for a similar argument.} we find 
\begin{equation}\label{eq:ell-algebra}
\ell_{n_2}(f) = \|\hat g_f\|_{n_2}^2 + n_2^{-1}\nu \|\hat g_f\|_{\tilde\cI}^2,
\end{equation}
}
where $\hat g_f$ is the optima for $\ell_{n_2}(f)$ and equals the KRR solution \eqref{eq:gpr-mean-estimator}. 
We will use Corollary~\ref{corr:fixed-design-gpr-mean} to bound the $\|\hat g_f\|_{n_2}$, { 
and Lemma~\ref{lem:krr-norm-bound} to bound $\|\hat g_f\|_{\tilde\cI}$.}

First 
note the residual $\be^f := f(\bx)-\by-E(f-f_0)(\bz)$ is bounded by $2(B+C_2)$, and thus subgaussian: this is because 
$|y_i-f_0(x_i)|\le B$ by Assumption~\ref{ass:npiv}, and $|f(x_i)-f_0(x_i)|\le C_2$ by \eqref{eq:den-f-defn}. 
Moreover, the true regression function 
$g_f := E(f-f_0)$ can be approximated in the true RKHS $\cI$ using $\bar g_f := E(f_h)$, where $f_h$ is defined in Lemma~\ref{lem:replaced-den-fset}, so that by that Lemma and Lemma~\ref{lem:optimal-I}, we have 
$$
\|\bar g_f\|_\cI^2 = \|f_h\|_\cH^2 \le C_2 n_2\delta_{n_2}^2, ~
\|\bar g_f - g_f\|_2^2 \le C_2\delta_{n_2}^2, ~
\|\bar g_f-g_f\|_\infty^2 \le C_2.
$$
Combining the above with \eqref{eq:den-f-defn} and the inequality $\|a\|^2 \le 2\|a\|^2 + 2\|a-b\|^2$ 
yields 
$$
\|\bar g_f\|_2^2 \le 4C_2\delta_{n_2}^2, ~ \|\bar g_f\|_\infty^2 \le 4C_2.
$$
Applying Theorem~\ref{prop:approx} and Proposition~\ref{prop:sup-norm-approx} to $\bar g_f\in\cI$, and recalling $\xi_{n_1}\lesssim n_2^{-1}$, we can see that 
there exists $\tilde g_f \in\tilde\cI$ s.t.~
\begin{equation}\label{eq:den-gapprox-cond}
\|\tilde g_f\|_{\tilde\cI}^2 \lesssim \|\bar g_f\|_\cI^2 \lesssim n_2\delta_{n_2}^2, ~~
\|\tilde g_f\|_2^2 \le 2\|\bar g_f\|_2^2 + 2\|\tilde g_f-\bar g_f\|_2^2 
\lesssim \delta_{n_2}^2 \log n_1, ~~
\|\tilde g_f\|_\infty^2 \lesssim 1.
\end{equation}
Now, a few applications of 
Lemma~\ref{lem:bernstein-emp-norm} shows that, for some constant $C_3>0$, we have, for any $f\in\Theta_{d0}$,  
\begin{equation}\label{eq:qb-den-e0}
\PP_{\Ztest}(\max\{\|g_f\|_{n_2}^2, \|\tilde g_f\|_{n_2}^2, \|\tilde g_f-g_f\|_{n_2}^2\} \le C \delta_{n_2}^2\log n_2
 ) \ge 1-3e^{-C_3 n_2\delta_{n_2}^2} \to 1.
\end{equation}
Therefore,  for some constant $C_3'>0$ and 
$\Pi_{g,n_1}$ the standard GP prior defined by $\tilde\cI$, we have %
\begin{align*}
\Pi_{g,n_1}(\{g: \|g-g_f\|_{n_2}\le C_3' \delta_{n_2}\log n_2) 
&\ge 
\Pi_{g,n_1}(\{g: \|g-\tilde g_f\|_{n_2}\le C_3 \delta_{n_2}\log n_2) 
\\&
\overset{(i)}{\ge} 
e^{-\|\tilde g_f\|_{\tilde\cI}^2} 
\Pi_{g,n_1}(\{g: \|g\|_{n_2}\le C_3 \delta_{n_2}\log n_2)  
\\&
\overset{(ii)}{\ge} 
\exp(-
(\|\tilde g_f\|_{\tilde\cI}^2 + n_2 C_3 \delta_{n_2}^2\log^2 n_2)) \numberthis\label{eq:qb-den-e1}
\\&
\overset{\eqref{eq:den-gapprox-cond}}{\ge} 
\exp(-2 n_2 C_3 \delta_{n_2}^2\log^2 n_2),
\end{align*}
where (i) can be found in \citet[Lemma I.28]{ghosal2017fundamentals}, and (ii) is by Lemma~\ref{lem:small-ball-prob} and holds on a high-probability event determined by $\Ztest$.

The above display fulfills the %
condition in 
Corollary~\ref{corr:fixed-design-gpr-mean}, which now shows that, when $\nu\asymp 1$ is sufficiently large (determined by $B$ and $C_2$), 
for some $C_4>0$ and $\eta_{n_2}^{(1)}\to 0$, both independent of $f\in\Theta_{d0}$, the maximizer $\hat g^f$ of $\ell_{n_2}(f)$ above satisfies 
\begin{equation}\label{eq:qb-den-e2}
\PP_{\dataSII}\left(\|\hat g_f - g_f\|_{n_2} \le C_4 \delta_{n_2}\log n_2 \mid \Ztest\right) \ge 1-\eta^{(1)}_{n_2} \to 1.
\end{equation}
{ \eqref{eq:den-gapprox-cond}, \eqref{eq:qb-den-e0} and \eqref{eq:qb-den-e2} fulfills the conditions in Lemma~\ref{lem:krr-norm-bound}, which now implies that 
$$
\PP_{\dataSII}\left(\|\hat g_f\|_{\tilde\cI}^2 \le C_5 n_2^{\frac{1}{b+2p+1}}\log n_2\right) \ge 1-\eta_{n_2}^{(2)} \to 1
$$
for some $C_5>0$. 
Plugging the two displays above into \eqref{eq:ell-algebra} %
yields \eqref{eq:qb-den-s2-goal}. %
}

\underline{(Step 2)} By \citet[proof for Proposition 20, step 1]{wang2021quasibayesian} applied to \eqref{eq:qb-den-s2-goal}, we find 
\begin{align*}
\PP_{\dataSII}(\Pi\{f\in\Theta_{d0}: \ell_{n_2}(f) \le C_5\|E(f-f_0)\|_2^2 + C_5\delta_{n_2}^2\log^2 n_2\}) \ge 1-4\eta_{n_2}.
\end{align*}
Combining the definition of $\Theta_{d0}$ and its prior mass bound in Lemma~\ref{lem:replaced-den-fset}, we have, on the above event, 
\begin{align*}
\int e^{-n_2\ell_{n_2}(f)}\Pi(df) \ge \exp\left(-3 C_5 n_2\delta_{n_2}^2\log^2 n_2\right).
\end{align*}
This completes the proof.
\end{proof}

\subsubsection{Replaced Numerator Bound}\label{app:qb-proof-main}

\newcommand{\err}[1][]{\mrm{err}_{n_2#1}}

Following \citet{wang2021quasibayesian}, the proof proceeds in two steps. Only the second step contains essential changes.

\paragraph{Step 1.}
Recall the notations in \citet[proof for Theorem 3]{wang2021quasibayesian}: in particular, 
$
\epsilon_n := n^{-\frac{b/2}{b+2p+1}}, ~~
\delta_n := n^{-\frac{(b+2p)/2}{b+2p+1}}, 
$
and $\{\Theta_n: n\in\mb{N}\}$ as defined by their Corollary 13. 
By that corollary, there exists some $C_{gp}>0$ so that 
\begin{equation}\label{eq:sieve-prior-mass}
\Pi(\Theta_r^c) \le \exp\Big(-2 C_{den} n_2^{\frac{1}{b+2p+1}} \log^2 n_2\Big), ~~\text{where}~~
r := \Big[\big(C_{gp}^{-1}C_{den}n_2^{\frac{1}{b+2p+1}}\log^2 n_2\big)^{b+1}
\Big].
\end{equation}
    
By their Lemma 21,\footnote{
We invoke the lemma with $n \gets [r^{\frac{b+2p+1}{b+1}}] \asymp n_2\log^{2(b+2p+1)} n_2$, so the lemma actually provides a stronger statement; for any $\alpha>0$ our claim immediately follows.
} when $M>0$ is sufficiently large, for $\alpha := \frac{2b+1}{b}$, the two sets of functions 
\begin{align*}
\err[,1] &:= \{f\in\Theta_r: \|f-f_0\|_2 \ge M\epsilon_{n_2}\log^\alpha n_2\}, \\
\err[,2] &:= \{f\in\Theta_r: \|E(f-f_0)\|_2 \ge C M\delta_{n_2}\log^\alpha n_2\}
\end{align*}
are ``equivalent up to constants''; to be precise,
we can have $\err[,1]\subset\err[,2]$ or $\err[,2]\subset\err[,1]$ by choosing $C>0$ appropriately, independent of $M$ or $n_2$. Therefore, we do not distinguish between them below, and use $\err$ to refer to both of them.

The analysis of (quasi)-posterior contraction rates is based on certain decompositions of the average posterior mass. We follow the decomposition in 
\citet{wang2021quasibayesian}: let $E_{den}$ denote the event defined in \eqref{eq:qb-den}, and $A(f,\dataSII)$ be an event to be defined below satisfying 
\begin{equation}\label{eq:num-test-prob-req}
\sup_{f\in\Theta_r\cap\err}\PP(A^c(f,\dataSII)) \le \exp\Big(-2 C_{den} n_2^{\frac{1}{b+2p+1}} \log^2 n_2\Big).
\end{equation}
Then
\begin{align*}
&\phantom{{}={}}\EE_{\dataSII}(
    \Pi(\err\mid\dataSII)
) \\&= 
\EE_{\dataSII} \frac{\int_{\err}\Pi(df)\exp(-n_2\ell_{n_2}(f))}{\int\Pi(df)\exp(-n_2\ell_{n_2}(f))}
\\ &\le 
\PP_{\dataSII} E_{den}^c + \EE_{\dataSII}\left(\mbf{1}_{E_{den}} 
\frac{\int_{\err} \Pi(df) \exp(-n_2\ell_{n_2}(f))}{\int \Pi(df)\exp(-n_2\ell_{n_2}(f))}
\right) \\ 
&\overset{\eqref{eq:qb-den}}{\le} 
o_{n_2}(1) + 
\exp\Big(C_{den} n_2^{\frac{1}{b+2p+1}}\log^2 n_2\Big) \EE_{\dataSII}
\int_{\err} \Pi(df) \exp(-n_2\ell_{n_2}(f)) 
\\ 
&\le 
o_{n_2}(1) + 
\exp\Big(C_{den} n_2^{\frac{1}{b+2p+1}}\log^2 n_2\Big) \EE_{\dataSII}
\Big[\Pi(\Theta_r^c)+\int_{\err\cap\Theta_r} \Pi(df) \exp(-n_2\ell_{n_2}(f))\Big]
\\
&\overset{\eqref{eq:sieve-prior-mass}}{\le} 
o_{n_2}(1) + 
\exp\Big(C_{den} n_2^{\frac{1}{b+2p+1}}\log^2 n_2\Big) \EE_{\dataSII}
\int_{\err\cap \Theta_r} \Pi(df) \exp(-n_2\ell_{n_2}(f))  \\ 
&\overset{\eqref{eq:num-test-prob-req}}{\le} 
o_{n_2}(1) + 
\exp\Big(C_{den} n_2^{\frac{1}{b+2p+1}}\log^2 n_2\Big) \EE_{\dataSII}\Big[\mbf{1}_{A(f,\dataSII)}
\int_{\err\cap \Theta_r} \Pi(df) \exp(-n_2\ell_{n_2}(f))\Big].
\end{align*}
Therefore, it suffices to show that, for all $f\in\err\cap\Theta_r$, on the event $A(f,\data)$, 
\begin{equation}\label{eq:num-goal}
\ell_{n_2}(f) \ge 2C_{den} \delta_{n_2}^2 \log^2 n_2.
\end{equation}
Following \citet{wang2021quasibayesian}, we define 
$$ \begin{aligned}
\Psi_{n_2}(f,g) &:= \frac{1}{n_2}\sum_{i=1}^{n_2} (f(x_i)-y_i)g(z_i), ~
\Psi(f,g) := \EE_{\dataSII}\Psi_{n_2}(f,g), ~ \\
A(f,\dataSII) &:= \left\{\Psi_{n_2}(f,g)-\|\tilde g\|_{n_2}^2 \ge 
\Psi(f,g)-\frac{3\|\tilde g\|_2^2}{2}\right\},
\end{aligned}
$$
where $\tilde g\in\tilde\cI$ is a (replaced) approximation to $E(f-f_0) =: g$ to be defined below; we can still invoke Bernstein's inequality as in \citet[end of p.~34]{wang2021quasibayesian}, finding 
\begin{equation}\label{eq:num-test-prob-defn}
\PP(A^c(f,\dataSII)) \le \exp\left(
    -\frac{n_2\|\tilde g\|_2^2}{16(\|f-f_0\|_\infty+B+\|\tilde g\|_\infty)^2}
\right).
\end{equation}
We will bound the RHS below.

\paragraph{Step 2.} We restrict to $f\in\Theta_r\cap\err$ and the event above.  

\underline{We first define our approximation $\tilde g$.} 
Let $\Delta f = f-f_0,
\widetilde{\Delta f}\in\cH,
g=E\Delta f,
g_j = E(\Proj{j}{\widetilde{\Delta f}})$ be defined as in the proof of~\citet[Lemma 21]{wang2021quasibayesian}.\footnote{
As in Step 1, we invoke the lemma with $n\gets [r^{\frac{b+2p+1}{b+1}}]=n_2 (\log^2 n_2)^{b+2p+1}$, so the various $L_2$ error terms there are smaller ($\epsilon_n\ll \epsilon_{n_2},\delta_n\ll \delta_{n_2}$), while the RKHS norm terms in the scale of $n^{\frac{1/2}{b+2p+1}}$ now becomes $n_2^{\frac{1/2}{b+2p+1}}\log n_2$. 
}
By their Eq.~40 and Eq.~42 in the proof of Lemma 21, we have 
\begin{equation}\label{eq:eq40-lemma21}
    \begin{aligned}
\|\widetilde{\Delta f} - \Delta f\|_\infty &\lesssim 1, ~~
\|\widetilde{\Delta f} - \Delta f\|_2 \lesssim \epsilon_{n_2}, ~~
\|\widetilde{\Delta f}\|_\cH \lesssim n_2^{\frac{1/2}{b+2p+1}}\log n_2, \\ 
\|E(\widetilde{\Delta f}-\Delta f)\|_2 &\lesssim \delta_{n_2}, ~~
\|g-g_j\|_2\lesssim \delta_{n_2}. 
    \end{aligned}
\end{equation}
Combining the last two inequalities, we have
$$
\|g_j - E(\widetilde{\Delta f})\|_2 \le \|g-g_j\|_2 + \|E(\widetilde{\Delta f}-\Delta f)\|_2 \lesssim \delta_{n_2}.
$$

By Lemma~\ref{lem:optimal-I} we have $g_j,E(\widetilde{\Delta f}) \in \cI$, and 
$
\max\{\|E(\widetilde{\Delta f})\|_\cI^2, \|g_j\|_\cI^2\} %
 \lesssim 
n_2\delta_{n_2}^2\log^2 n_2.
$ Thus their difference has the same $\cI$-norm bound, and 
by \eqref{eq:interpolation} we have 
\begin{align*}
\|g_j - E(\widetilde{\Delta f})\|_\infty &\lesssim 
\|g_j-E(\widetilde{\Delta f})\|_2^{\frac{1}{b+2p+1}} \|g_j-E(\widetilde{\Delta f})\|_\cI^{\frac{b+2p}{b+2p+1}} \le \log n_2. \\ 
\|g_j - g\|_\infty &\le \|g_j-E(\widetilde{\Delta f})\|_\infty 
+ \|\cancel{E(}\widetilde{\Delta f}-\Delta f)\|_\infty \lesssim \log n_2. \numberthis\label{eq:gj-sup-norm}
\end{align*}
(We used the inequality $\|E(\cdot)\|_\infty\le\|\cdot\|_\infty$ as we are working with a version of conditional expectation that is bounded under the sup norm.)

Now, by %
Theorem~\ref{prop:approx} and Proposition~\ref{prop:sup-norm-approx}, there exists $\tilde g\in\tilde\cI$ s.t.
\begin{equation*}%
\|\tilde g\|_{\tilde\cI}^2\lesssim %
n_2^{\frac{1}{b+2p+1}}\log^2 n_2, ~~
\|\tilde g - g_j\|_2^2 \lesssim \SkipNOTE{\|g_j\|_\cI^2 (m^{-(b+2p+1)} + xi_{n_1}^2) log n_2 \lesssim \|g_j\|_\cI^2 n_2^{-1}\le}
n_2^{-\frac{b+2p}{b+2p+1}} \log^2 n_2, ~~ 
\|\tilde g - g_j\|_\infty^2 \lesssim 
\SkipNOTE{\|\Proj{>m}(g_j)\|_\infty^2 + m^{-1}\|\Proj{m}(g_j)\|_\cI^2 }
\SkipNOTE{for the first term, note Proj{<m}(g_j) \lesssim 1 by interpolation inequality on I}
\log^2 n_2.
\end{equation*}
Combining with \eqref{eq:eq40-lemma21}, \eqref{eq:gj-sup-norm}:
\begin{equation}\label{eq:g-approx-properties}
\|\tilde g\|_{\tilde\cI}\lesssim %
n_2^{\frac{1/2}{b+2p+1}}\log n_2, ~~
\|\tilde g-g\|_2\lesssim \delta_n\log n_2, ~~
\|\tilde g-g\|_\infty\lesssim \log n_2.
\end{equation}

\underline{Now we check the probability bound \eqref{eq:num-test-prob-defn} satisfies \eqref{eq:num-test-prob-req}.}
Define $
U := \frac{\|f-f_0\|_2}{\epsilon_{n_2}\log^\alpha n_2}, 
$ so that $U\ge \sqrt{M}$, and 
by \citet[Lemma 21]{wang2021quasibayesian},
\begin{align*}
U \delta_{n_2}\log^\alpha n_2 &\lesssim 
\|g\|_2\le \|\Delta f\|_2 \le U\epsilon_{n_2}\log^\alpha n_2. \numberthis\label{eq:num-g-norm-bound}\\ 
\|\widetilde{\Delta f}\|_2 &\le \|\Delta f-\widetilde{\Delta f}\|_2+\|\Delta f\|_2 
\overset{\eqref{eq:eq40-lemma21}}{\lesssim} U\epsilon_{n_2}\log^\alpha n_2.
\end{align*}
By the embedding property assumption: %
\begin{align*}
\|f-f_0\|_\infty &\le \|\widetilde{\Delta f}\|_\infty + \|\Delta f-\widetilde{\Delta f}\|_\infty \overset{\eqref{eq:eq40-lemma21}}{\lesssim} 
\|\widetilde{\Delta f}\|_\infty + 1 
\overset{\eqref{eq:interpolation-general}}{\lesssim} \|\widetilde{\Delta f}\|_\cH^{\frac{b}{b+1}} \|\widetilde{\Delta f}\|_2^{\frac{1}{b+1}} 
\lesssim  U^{\frac{1}{b+1}}(\log n_2)^{\frac{b+\alpha}{b+1}}.
\end{align*}
Using the triangle inequality and the inequality $\|g\|_\infty=\|E(f-f_0)\|_\infty\le\|f-f_0\|_\infty$, 
we bound \eqref{eq:num-test-prob-defn} as 
\begin{align*}
\PP(A^c(f,\data)) &\le \exp\left(
    -\frac{n_2(\|g\|_2-\|\tilde g-g\|_2)^2}{16(\|f-f_0\|_\infty+B+\|\tilde g-g\|_\infty+\|g\|_\infty)^2}
\right) 
\\ &
\le 
\exp\left(
    -\frac{n_2(\|g\|_2-\|\tilde g-g\|_2)^2}{16(2\|f-f_0\|_\infty+B+\|\tilde g-g\|_\infty)^2}
\right) 
\\ &
\le 
\exp\left(-\frac{
        n_2\delta_{n_2}^2(\log n_2)^{2\alpha} (U^2-1)
    }{
        C( U^{\frac{1}{b+1}}(\log n_2)^{\frac{b+\alpha}{b+1}} + B + \log n_2)^2
}\right).
\end{align*}
For $\alpha\ge \frac{2b+1}{b}$ and $U$ sufficiently large, the RHS is $\le \exp(-C U^{\frac{2b}{b+1}} n_2\delta_{n_2}^2\log^2 n_2)$ and thus satisfies \eqref{eq:num-test-prob-req}. 

\underline{It remains to check \eqref{eq:num-goal}, for $f\in\Theta_r\cap\err$ and on the event $A(f,\dataSII)$.} As $\tilde g\in\tilde\cI$, the manipulations in \citet[Eq.~50]{wang2021quasibayesian} continue to hold, with $g_j$ replaced by $\tilde g$; it leads to 
$$
\ell_{n_2}(f) \ge c_1\|g\|_2^2 - C_2\delta_n^2 -  n_2^{-1}\nu\|\tilde g\|_\cI^2,
$$
where $c_1,C_2>0$ are constants, and we recall $\nu\asymp 1$.  
Plugging in \eqref{eq:num-g-norm-bound} and \eqref{eq:g-approx-properties}, we have 
$$
\ell_{n_2}(f) \ge
c_1 U^2\delta_{n_2}^2\log^{2\alpha} n_2 - C_2\delta_n^2 - n_2^{-1}\nu\|\tilde g\|_\cI^2 
\ge 
c_1 U^2\delta_{n_2}^2\log^{2\alpha} n_2 - C_2\delta_n^2 - C_3\SkipNOTE{+\nu}\delta_n^2 \log^2 n_2.
$$
Since $\alpha>1$, \eqref{eq:num-goal} holds for sufficiently large $U$. 
This completes the proof for Theorem~\ref{prop:qb-regime}.\qed

\subsection{Proof for Corollary~\ref{lem:qb-mlh}}\label{app:proof-qb-mlh}

The idea here is not new (see e.g., \cite{rousseau_asymptotic_2017}). By Proposition~\ref{prop:qb-den}, the upper bound holds on an event $E_{den}$ s.t.~$\PP_{\dataSII} E_{den}\to 1$. For the lower bound of $-\Pi(\dataSII)$, observe 
\begin{align*}
\Pi(\dataSII) &= \int_{\err} \Pi(df) e^{-n_2 \ell_{n_2}(f)} + \int_{\err^c}\Pi(df)e^{-n_2\ell_{n_2}(f)} \\ & \le  \int_{\err} \Pi(df) e^{-n_2 \ell_{n_2}(f)} + \Pi(\err^c).
\end{align*}
where $\err$ is defined in Appendix~\ref{app:qb-proof-main}. 
For the first part, we have (see Step 1 in Appendix~\ref{app:qb-proof-main})
\begin{align*}
\EE_{\dataSII} \mbf{1}_{E_{den}} \int_{\err} \Pi(df) e^{-n_2 \ell_{n_2}(f)}  \le e^{-2 C_{den} n_2^{\frac{1}{b+2p+1}}\log^2 n_2}. 
\end{align*}
As $\PP(E_{den})\to 1$, it must hold that 
$$
\PP_{\dataSII}\Big(\int_{\err} \Pi(df) e^{-n_2 \ell_{n_2}(f)} \le e^{-C_{den} n_2^{\frac{1}{b+2p+1}}\log^2 n_2}\Big) \to 1.
$$
For the second part, recall that by definition we have $$
\err^c\subset \{f: \|f-f_0\|_2 \le M\epsilon_{n_2}\log^{\frac{2b+1}{b}} n_2\}
$$ 
for some constant $M>0$, and $\epsilon_{n_2}^2\asymp n_2^{-\frac{b}{b+2p+1}}$; for $j=[n_2^{\frac{b+1}{b+2p+1}}]$ and $f^\dagger_j$ defined in Assumption~\ref{ass:s2}~\ref{it:gp-scheme}, we have 
\begin{align*}
\Pi(\err^c) &\le \Pi(\{f: \|f-f_0\|_2 \le M\epsilon_{n_2}\log^{\frac{2b+1}{b}} n_2\}\}\\ &\overset{(i)}{\le} \Pi\left(\left\{f: \|f-f^\dagger_{j}\|_2 \le \frac{M}{2}\epsilon_{n_2}\log^2{\frac{2b+1}{b}}n_2\right\}\right) \\ 
&\overset{(ii)}{\le} e^{\|f^\dagger_{j}\|_\cH^2}  \Pi\left(\left\{f: \|f\|_2 \le \frac{M}{2}\epsilon_{n_2}\log^2{\frac{2b+1}{b}}n_2\right\}\right)  \\ 
&\overset{(i)}{\le} e^{n_2^{\frac{1}{b+2p+1}}} 
\Pi\left(\left\{f: \|f\|_2 \le \frac{M}{2}\epsilon_{n_2}\log^2{\frac{2b+1}{b}}n_2\right\}\right).
\end{align*}
In the above, (i) follows by the assumption on $f^\dagger_j$, and (ii) can be found in \citet[Lemma I.28]{ghosal2017fundamentals}. 
Recall the $L_2$ small-ball probability bound: for $\cH$ satisfying Assumption~\ref{ass:s2} and sufficiently small $\epsilon$ it holds that \citep[see e.g.,][Corollary 4.9]{steinwart_convergence_2019}
$$
-\log\Pi(\{\|f\|_2\le \epsilon\}) \asymp \epsilon^{-\frac{2}{b}}.
$$
Combining these results complete the proof.\qed
\section{Implementation Details}\label{app:algo}

This section describes the implementation of the proposed method. Further experiment details, such as network architectures and the range of hyperparameters, are in Appendix~\ref{app:simulations}. 

\paragraph{The Algorithm}
Our algorithm is described in Algorithm~\ref{alg:main}. In our implementation, we combine the $m$ scalar regression tasks to a single vector-valued regression task, for which we train a neural network model with square loss. We draw approximate GP prior samples using random Fourier features.\footnote{More sophisticated sampling schemes exist, e.g., exact Mat\'ern GP samples can be obtained using its state-space represetation \cite{sarkka2013state}. But this is unnecessary for an one-off operation, where we can simply set the number of random features to be sufficiently large.}

Our implementation makes a small modification to $\tilde k_z$: we extend it with a single linear feature $\hat h(\cdot)$, defined as the regression estimate of $\EE(\by\mid\bz)$, with output truncated by $C\log n_1$. Thus, the kernel becomes 
$
\tilde k_z(z,z') = \frac{1}{m}\sum_{i=1}^m \esti{j}(z)\esti{j}(z') + \hat h(z)\hat h(z'). 
$
Clearly, the extension does not affect any approximation or estimation guarantee,\footnote{for the latter, note the truncation of $\hat h$, which ensures that $\|g\|_\infty\lesssim \|g\|_{\tilde\cI}\log n_1$ continue to hold.} but it guard against potential misspecification of $\cH$: when Assumption~\ref{ass:s2} (iii) fails to hold, the unmodified $\tilde\cI$ can only estimate $E f$ for $f\in\cH$ (or $f\sim\mc{GP}(0,k_x)$), but not $\EE(f(\bx)-\by\mid\bz=\cdot)$; the modified $\tilde k_z$ fixes this. This issue does not arise when the assumption holds.

The kernelized IV quasi-posterior admits the following closed-form expression \cite{dikkala_minimax_2020,wang2021quasibayesian}, which we restate for the reader's convenience. The point estimator \eqref{eq:minimax-estimator} equals the posterior mean below, with $\lambda,\nu$ adjusted appropriately (see Appendix~\ref{app:dikkala-proof}, or \cite[Appendix~E]{dikkala_minimax_2020}). 
\begin{align*}\label{eq:kiv}
\Pi(f(x_*)\mid\dataSII) &= \cN(K_{*x} \Lambda Y, K_{**} - K_{*x}\Lambda K_{x*}) \\ 
\text{where}\quad L &= (K_{zz}+\nu I)^{-1} K_{zz}, ~ \Lambda = (\lambda I + L K_{xx})^{-1} L.
\end{align*}
In the above, $K_{(\cdot)}$ denotes the appropriate Gram matrix. As our $K_{zz}$ admits a known low-rank structure: $K_{zz} = \Phi\Phi^\top$, where $\Phi=(\esti{1}(Z),\ldots,\esti{m}(Z),\hat h(Z))\in\RR^{n_2\times (m+1)}$, we use Woodbury identity to obtain $L=\Phi (\Phi^\top\Phi+\nu I)^{-1}\Phi^\top$ which can be computed in $\cO(m n^2)$. Another application of Woodbury identity allows the computation of $\Lambda$ in $\cO(mn^2)$ time \citep[Appendix D.3]{wang2021quasibayesian}. 
Applying Nystr\"om approximation to $k_x$ would improve the complexity to $\cO(m^2 n)$ \citep{dikkala_minimax_2020}, but we find the above procedure sufficient for our purposes. 

\paragraph{Hyperparameter Selection for Instrument Learning} In principle, the regression oracle may conduct hyperparameter selection by further splitting its observed dataset, and perform cross validation. We can also compare different oracles on the said validation set. 
In practice we construct a heldout set using $\dataSII$ and $\gpi{1},\ldots,\gpi{m}$ for the NN-based oracle, and use it for early stopping. For selecting the best oracle (which includes the selection of NN activation functions, optimizers, etc., in our setting), however, 
we find it slightly better to 
draw additional GP samples $\gpi{1}_v,\ldots,\gpi{m_v}_v$ and evaluate the different oracles by using the resulted kernel $\tilde k_z$ on the dataset 
$\{(z_i, \gpi{j}_v): i\in [n_2], j\in [m_v]\}$. This can be viewed as a task generalization loss in the terminologies of multi-task learning (Sec.~\ref{sec:related-work}), and is more directly connected to the role of $\tilde\cI$ in IV regression. It also allows data-dependent selection for $m$, although we use fixed choices of $m$ for simplicity. 

Following the discussion in the last paragraph, 
we further guard against potential misspecification of $\cH$, by adding to the validation statistics the regression error on $\{(z_i, y_i): i\in [n_2]\}$. In summary, our first stage validation statistics is 
\begin{equation}\label{eq:s1-val-stats} \begin{aligned}
    &m_v^{-1}\sum_{j=1}^{m_v} \msf{KRRTestMSE}(\{(z_i, \gpi{j}_v): i\in [n_2]\}; \tilde k_z) + %
    \msf{KRRTestMSE}(\{(z_i, y_i): i\in [n_2]\}; \tilde k_z). \end{aligned}
\end{equation}

\paragraph{Second-Stage Model Selection} %
We provide the following expression for the log marginal quasi-likelihood:
\begin{align}\label{eq:s2-log-qlh}
\log \int \Pi(df) \exp(-\lambda^{-1} n_2\ell_{n_2}(f)) &=
\log \int \Pi(df) \exp\left(-\frac{1}{2\lambda} (Y-f(X))^\top L (Y-f(X))\right) \nonumber \\ &=
-\frac{1}{2}\left(
Y^\top \Lambda Y + \log | \lambda^{-1} (L^{1/2} K_{xx} L^{1/2} + \lambda I) |
\right) + \mrm{const} \\ 
\text{where}\quad
L &= (K_{zz}+\nu I)^{-1} K_{zz} \nonumber \\ 
\text{and}\quad
\Lambda &= (\lambda I + L K_{xx})^{-1} L. \nonumber
\end{align}
In the above, $K_{xx}, K_{zz}$ denote the Gram matrices using $\cH$, $\tilde\cI$, respectively, and 
the first equality above can be found in \cite{wang2021quasibayesian}. Similar to \cite{wang2021quasibayesian} we exploit low-rank structures in $\Lambda$ and $L$ to accelerate computation, but with Nystr\"om approximation replaced by the feature representation of $\tilde\cI$. 

Note that we cannot use the above marginal quasi-likelihood cannot to select $\lambda$; this is evident from the first expression $\int \Pi(df) e^{-\lambda^{-1}n_2 \ell_{n_2}(f)}$, as the log quasi-likelihood $\ell_n$ does not contain $\lambda$. This is a small price we pay for the simplified analysis: observe 
$
\ell_n(f) = \sup_{g\in\tilde\cI}\sum_{i=1}^{n_2} 2(f(x_i)-y_i)g(z_i)-g(z_i)^2 - \nu \|g\|_{\tilde\cI}^2
$
is equivalent to $
\inf_{g\in\tilde\cI}\sum_{i=1}^{n_2} (f(x_i)-y_i-g(z_i))^2 + \nu \|g\|_{\tilde\cI}^2
$
up to a Gaussian complexity term. The latter allows for a more complete quasi-Bayesian data generating process as in \citep{florens_nonparametric_2012}, to which we can add a hyperprior for $\lambda$ as in \cite{florens_nonparametric_2012}. 
However, bounding the extra term would involve additional technicalities. We consider it beyond the scope of this work. 

Once $\lambda\asymp 1$ is fixed, however, both versions of marginal likelihood should allow for principled selection of other hyperparameters, as both choices for $\ell_{n_2}$ can be viewed as approximating $\lambda^{-1} n_2 \|E(f-f_0)\|_2^2$; although in our setting, theoretical guarantees have only been established for the former choice. As for $\lambda$, we fix $\lambda = \frac{1}{n_2}\sum_{i=1}^{n_2}(y_i - \hat h(z_i))^2$, where $\hat h$ is an estimate for $\EE(\by\mid\bz=\cdot)$; this can be viewed as an {\em empirical Bayes} counterpart for the choice in \cite{florens_nonparametric_2012}. Across all settings we find the selected value to be fairly stable. This is different from other hyperparameters such as kernel bandwidth or variance, for which careful likelihood-based selection is more needed.

\section{Extension for High-Dimensional Exogenous Covariates}\label{app:exo-algo}

In many applications we have access to additional exogenous covariates $\bw$, which are independent with the unobservable confounder $\bu$ \cite{horowitz_applied_2011,hartford2017deep}. Denote the original instrument and treatment as $\bz_o,\bx_o$, the true outcome function then satisfies $\EE(f_0(\bx_o,\bw)-\by\mid \bz_o,\bw)=0$, meaning that such $\bw$ can then be appended to both $\bx$ and $\bz$ in the formulation \eqref{eq:npiv}. 

When both $\bx_o$ and $\bw$ have moderate dimensions, the prior knowledge about $f_0$ can often be characterized by a product kernel $k_x(x,x') := k_{x_o}(x_o,x_o')k_w(w,w')$, where $k_w$ can be predetermined. In this case, Algorithm~\ref{alg:main} will continue to apply. This connects to the classical tensor product basis model \cite[e.g.,][]{chen_optimal_2015}, as the Mercer basis of $k_x$ equals $\{\psi^x_{(i,j)} = \psi^{x_o}_i\otimes \psi^w_j\}$. 
The inclusion of $\bw$ breaks our assumption by making $E$ non-compact. However, 
it is common in theoretical works to ignore it for brevity \cite{chen_estimation_2012,horowitz_applied_2011,horowitz_specification_2012,kato_quasi-bayesian_2013}, also because 
$\bw$ does not suffer from the ill-posedness issue, and thus is intuitively easier to handle.

When $\bw$ is high-dimensional, it can be difficult to prescribe a correctly specified $k_w$ with low complexity. However, it is often realistic that there exist some low-dimensional informative features $\bar w=\Phi_w(w)$,\footnote{Otherwise it would be unclear if efficient estimation is still possible at all.} in which case it is natural to consider kernels of the form $k_{x_o}(x_o,x_o')k_w(\Phi_w(w),\Phi_w(w'))$, where $k_{x_o}$ is still available {\em a priori}, but $k_w$ and $\Phi_w$ needs to be learned. In this section we describe a simple algorithm, inspired by \cite{xu_learning_2020}, which implicitly models such a $k_w$ with flexible NN models. For brevity, we will focus on intuition, and will not provide any formal error analysis. The algorithm will be evaluated in Appendix~\ref{app:exp-exo}.

To motivate the algorithm, suppose {\em for the moment} that we have access to $k_w$ and its Mercer basis. Then any $f\in\cH$ can be expressed as 
$
f = \sum_{i=1}^\infty\sum_{j=1}^\infty a_{ij} \psi^{x_o}_i\otimes \psi^w_j,
$ and the coefficients $a_{ij}$ should satisfy 
a fast decay, as determined by the assumed eigendecay of $k_{x_o}$ and the unknown, true $k_w$. It is thus possible to truncate the outer sum at $i=m$, and the truncation error will be vanishing fast for some slowly increasing $m$. 

The proof of Theorem~\ref{prop:approx} relied on the fact that $2m$ random GP prior draws approximate the top $m$ Mercer basis well. Thus, we can draw $2m$ samples $f^{x_o}_i\sim\mc{GP}(0,k_{x_o})$, and perform an ``approximate change of basis'':
$$
f =  \sum_{i=1}^m \psi^{x_o}_i\otimes \Big(\sum_{j=1}^\infty a_{ij} \psi^w_j\Big) + \Delta_m(f) 
= \sum_{i=1}^m f^{x_o}_i\otimes f^w_i + \Delta'_m(f),
$$
In the above, $f^w_i\in L_2(P(dw))$ can be obtained by rotating the functions $\{\sum_{j=1}^\infty a_{ij} \psi^w_j\}$,\footnote{
A properly scaled version of $\{\psi_i^{x_o}\}$ is well approximated by a rotated version of $\{f^{x_o}_i\}$ (App.~\ref{app:proof-prop-approx}). Thus, $f^w_i$ can be obtained by inverting the rotation on  $\{\sum_{j=1}^\infty a_{ij} \psi^w_j\}$.
} and should have good regularity determined by $k_{x_o}$ and $k_w$. Now, observe that 
$$
E f = \EE\Big(\sum_{i=1}^m f^{x_o}_i\otimes f^w_i\mid\bz_o,\bw\Big) + E\Delta'_m(f) = \sum_{i=1}^m (Ef^{x_o}_i)\otimes f^w_i + E\Delta'_m(f).
$$
The approximation error $\|E\Delta'_m(f)\|_2\le \|\Delta'_m(f)\|_2$ should continue to be small. We can efficiently approximate $E f_i^{x_o}$ using the regression oracle, and {\em parameterize} $E f$ as 
$$
(E f)(z_o,w) \approx \sum_{i=1}^m \esti{i}(z_o) f^w_i(w;\theta),
$$
where $f^w_i(w;\theta)$ denotes the $i$-th output of a DNN with parameter $\theta$. 
Since the true $f_0$ should satisfy $(E f_0)(\bz_o,\bw)=\EE(\by\mid\bz_o,\bw)$, we can use the RHS of the above to regress $\by$, after which the estimate for $f_0$ can be read out by replacing $\esti{i}(z_o)$ with $f^{x_o}_i(x_o)$. 

In summary, we constructed an algorithm which {\em implicitly} models a tensor product kernel. The full algorithm is summarized below:

\begin{algorithm}[h]
\caption{IV regression with learned instruments and exogenous covariates.}\label{alg:exo}
\begin{algorithmic}[1]
\REQUIRE $\dataSI,\dataSII$; regression algorithm $\msf{Regress}$; treatment kernel $k_{x_o}$; $m\in\mb{N}$
\FOR{$j \gets 1$ to $m$}
\STATE Draw $f^{x_o}_j\sim \mc{GP}(0, k_{x_o})$
\STATE 
$\estiRaw{j}\gets \msf{Regress}(\{((\tilde z_{o,i},\tilde w_i), f^{x_o}_j(\tilde x_{o,i})): i\in [n_1]\})$
\STATE Define
 $\esti{j} := \min\{\estiRaw{j}(\cdot), C\log m\}$
\ENDFOR
\STATE Let $g(z_o,w;\theta) := \sum_{i=1}^m \esti{i}(z_o,w) f^w_i(w;\theta)$, where $(f^w_i)$ denote a multi-output DNN with parameter $\theta$. Optimize $\theta$ using (a possibly regularized variant of) the objective
$$
\ell'_{n_2}(\theta;\dataSII) := \sum_{i=1}^{n_2} (g(z_{o,i},w_i;\theta)-y_i)^2.
$$
\STATE {\bf return} $\hat f_{n_2}(x_o,w) := \sum_{i=1}^m f^{x_o}_j(x_o) f^w_i(w;\theta)$.
\end{algorithmic}
\end{algorithm}

While we will not present a full analysis for brevity, it should be easy to provide faster-rate guarantees on $\|E(\hat f_{n_2} -f_0)\|_2$, for the NN model in \cite{schmidt-hieber_nonparametric_2020}: it suffices to control the M-estimation error about $\ell'_{n_2}$ with local Rademacher analysis, and combine it with the regression error $\|\esti{j} - Ef^{x_o}_j\|_2$. 
In principle, uncertainty quantification can be conducted by viewing $\ell'_{n_2}$ as a log quasi-likelihood, and modeling $f$ with a Bayesian neural network (BNN); investigation of posterior consistency may then follow \cite{liang2018bayesian}. 

Our algorithm is inspired by \cite{xu_learning_2020} which also employs a tensor product model, but additionally learns the treatment representation. From the theoretical perspective, their added flexibility comes with a hefty price: as noted in \cite{xu_learning_2020}, it is unclear if their alternating optimization procedure minimizes the empirical risk, and only slow rate convergence has been established for the hypothetical empirical risk minimizer. Intuitively, our method avoids these issues by disentangling (the important parts of) the two stages, and carefully reducing them to regression-like problems. This is also the reason we are able to reduce -- in principle -- uncetainty quantification to a standard BNN inference problem, which is not possible with the formulations of \cite{xu_learning_2020,dikkala_minimax_2020}.
\section{Simulations}\label{app:simulations}

Code for all experiments can be found at \url{https://github.com/meta-inf/fil}.

\subsection{Evaluation of Hyperparameter Selection}\label{sec:exp-s1-randgp}

We first evaluate the instrument learning procedure under a correctly specified second stage. 

\paragraph{Setup}
We set $f_0\sim\mc{GP}(0,k_x)$ and generate $\bz$ using NNs, with 
$N\in\{500,2500,5000\},D\in\{40,100\}$. $k_x$ is set to a RBF kernel with bandwidth determined set to the median distance between inputs (the ``median trick''). 

For our method, we train DNNs using the square loss and the AdamW optimizer, and perform early stopping based on validation loss. We apply dropout with rate varying in $\{0, 0.05, 0.1, 0.2, 0.4\}$. We vary activation functions in relu, swish and tanh; hidden layers in [100], [100, 100], [100, 100, 100] and [100, 100, 100, 100]; and learning rate in $\{5\times 10^{-3},10^{-3},5\times 10^{-4},10^{-4}\}$. 

We also include two baselines to verify the need for flexibily learned first stage:
\begin{enumerate}[leftmargin=*]
    \item A kernelized first stage using the RBF kernel, where the bandwidth is set to $\{1,2,4,8\}$ times the median distance.\footnote{We find that on this dataset, a bandwidth smaller than the median distance is never optimal.}
    \item A linear first stage, with basis set to the output of randomly initialized NNs with the same architecture. %
\end{enumerate} 
Both baselines use the same correctly specified second stage. 

We evaluate each setup on 5 independently generated datasets. For our method, training takes a total of 3.6 hours on 8 RTX 3090 GPUs. On average, each single experiment takes 11 seconds, and (at least) 6 experiments can be run in parallel on a single GPU.

\paragraph{Results} We plot the validation statistics against test error $\|\hat f_n-f_0\|_2^2$ 
in Figure~\ref{fig:s1-scatter}. We can see that 
learnable NNs lead to the best performance across all sample size, and the first-stage validation statistics correlates well with test performance. 
\begin{figure}[h]
    \centering
    \includegraphics[width=0.49\linewidth]{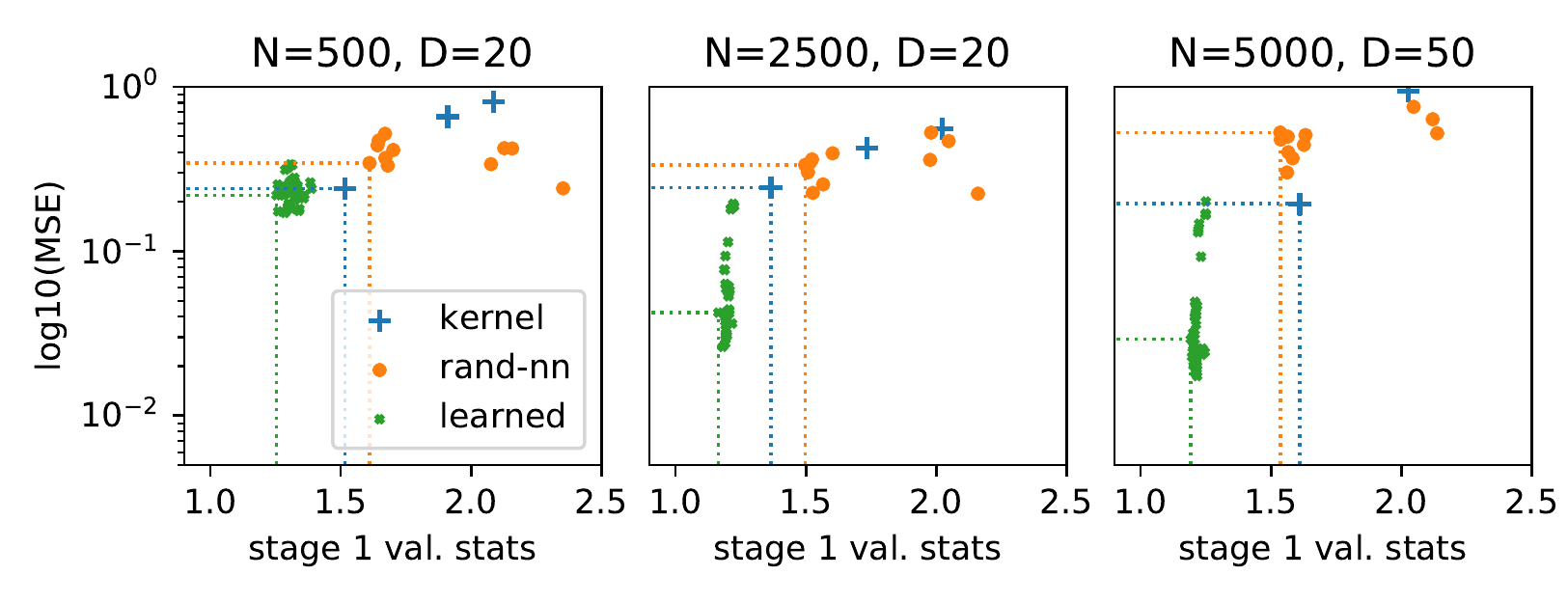}
    \includegraphics[width=0.49\linewidth]{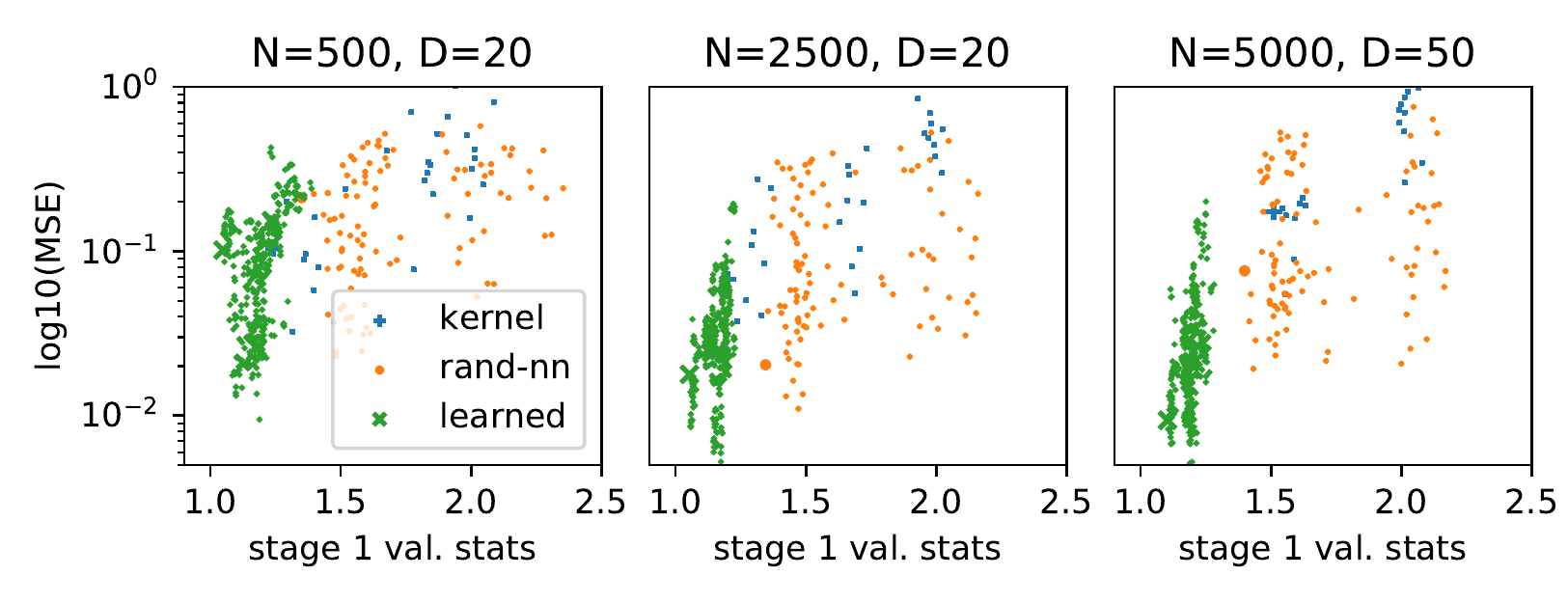}
    \caption{First-stage validation statistics vs.~counterfactual MSE for all methods and choices of hyperparameters. Left: visualization of a single run; within each method the model with the best validation statistics is highlighted. Right: aggregated plots for 5 independent runs.}\label{fig:s1-scatter}
\end{figure}
Figure~\ref{fig:s1-by-hps} provides further information on the influence of various NN hyperparameters, 
by varying one hyperparameters (network activation, architecture, learning rate or dropout rate), and plotting the test and validation statistics with the others chosen to optimize the validation statistics. We can see that network depth and dropout rate have a larger impact. 

\begin{figure}[ht]
\centering 
\subfigure[activation, test MSE]{\includegraphics[width=0.24\linewidth,clip,trim={0.4cm 0 0.25cm 0}]{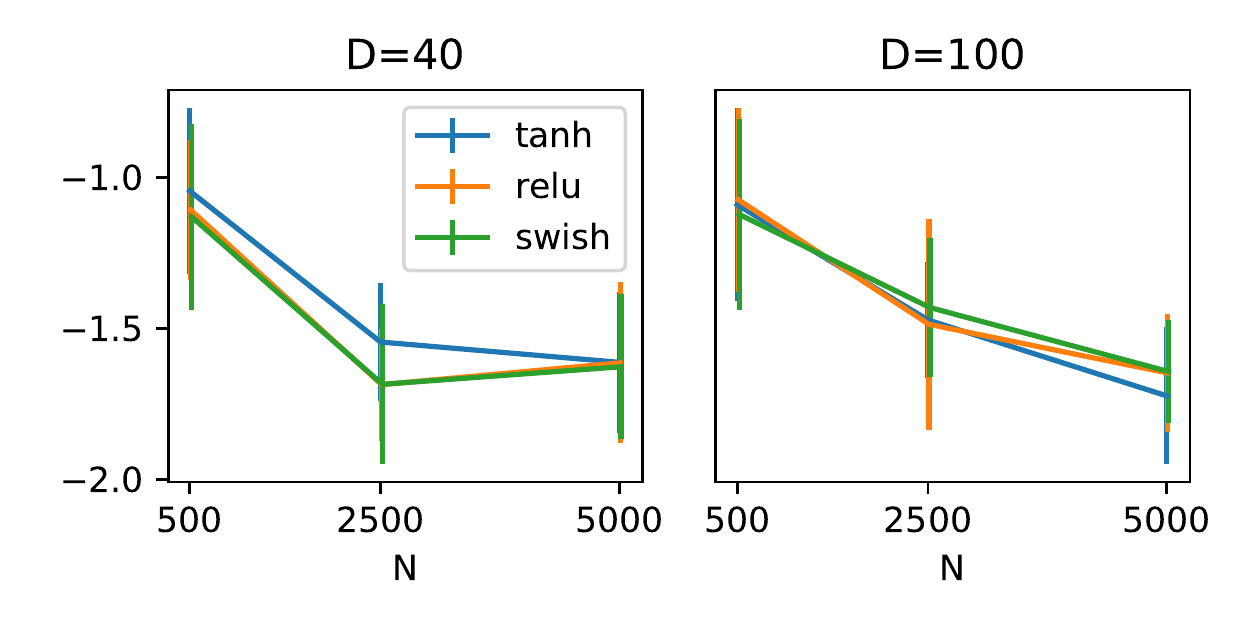}}
\subfigure[activation, val.~stats.]{\includegraphics[width=0.24\linewidth,clip,trim={0.4cm 0 0.25cm 0}]{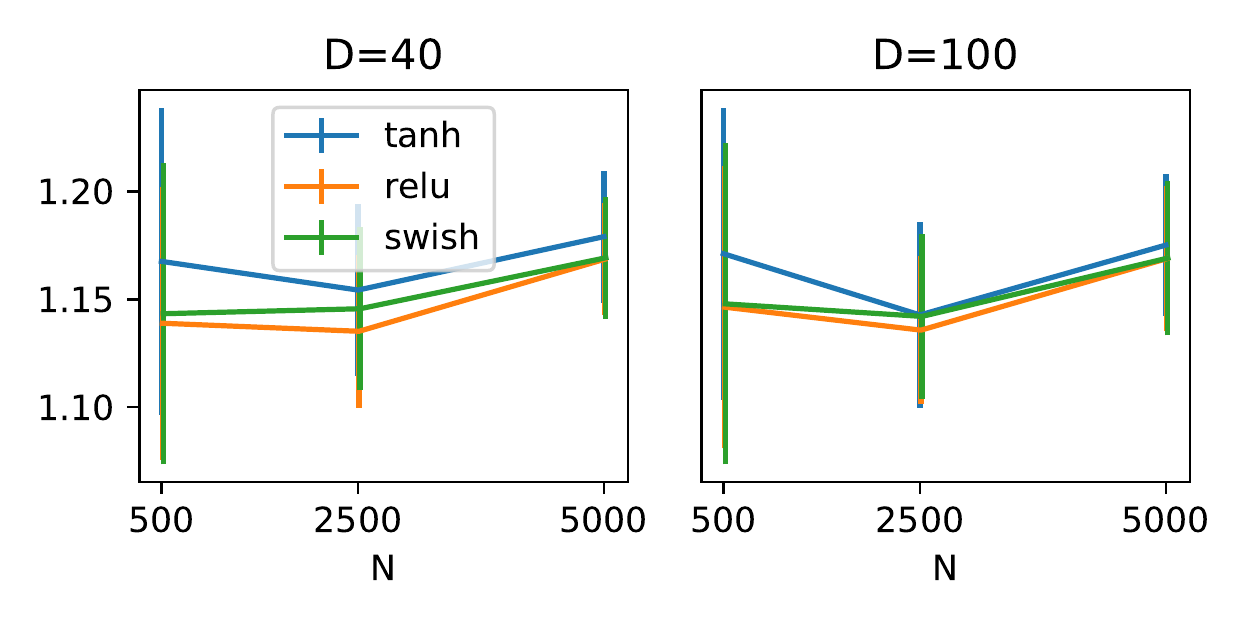}}
\subfigure[learning rate, test MSE]{\includegraphics[width=0.24\linewidth,clip,trim={0.4cm 0 0.25cm 0}]{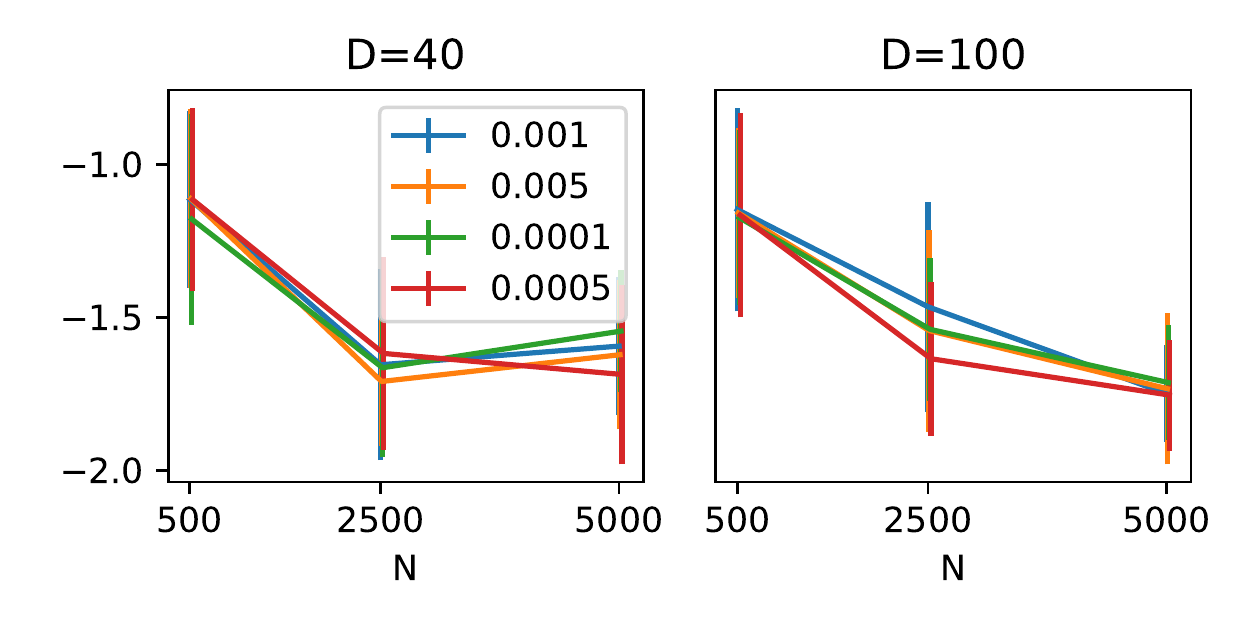}}
\subfigure[learning rate, val.~stats.]{\includegraphics[width=0.24\linewidth,clip,trim={0.4cm 0 0.25cm 0}]{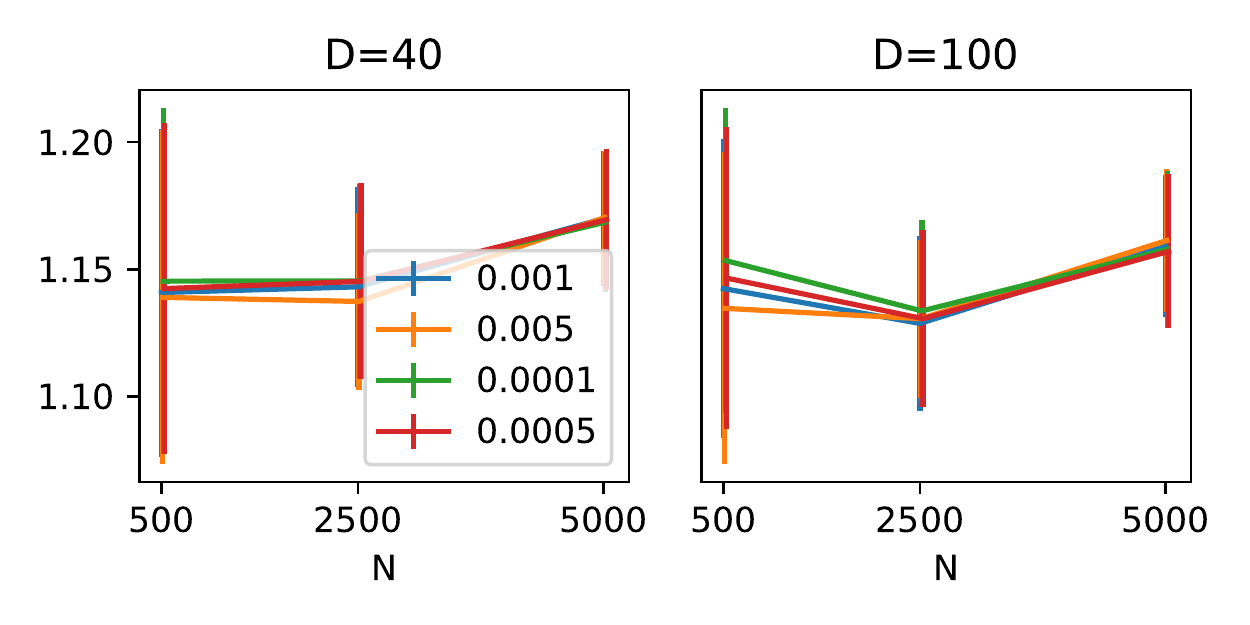}}
\subfigure[NN arch., test MSE]{\includegraphics[width=0.24\linewidth,clip,trim={0.4cm 0 0.25cm 0}]{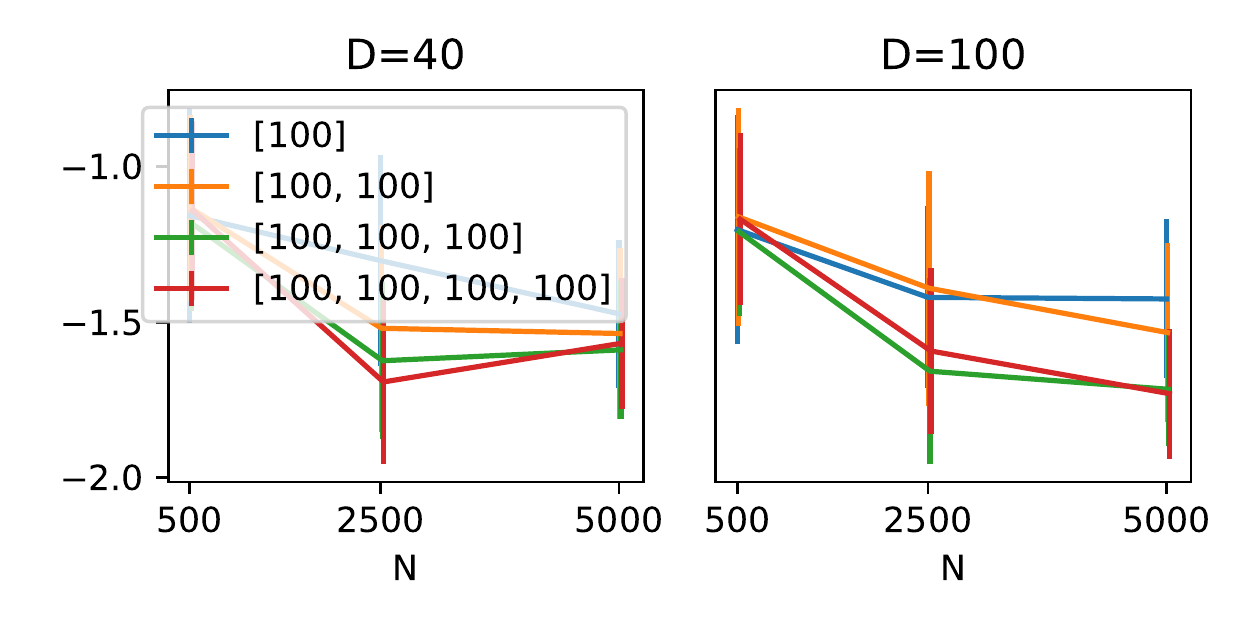}}
\subfigure[NN arch., val.~stats.]{\includegraphics[width=0.24\linewidth,clip,trim={0.4cm 0 0.25cm 0}]{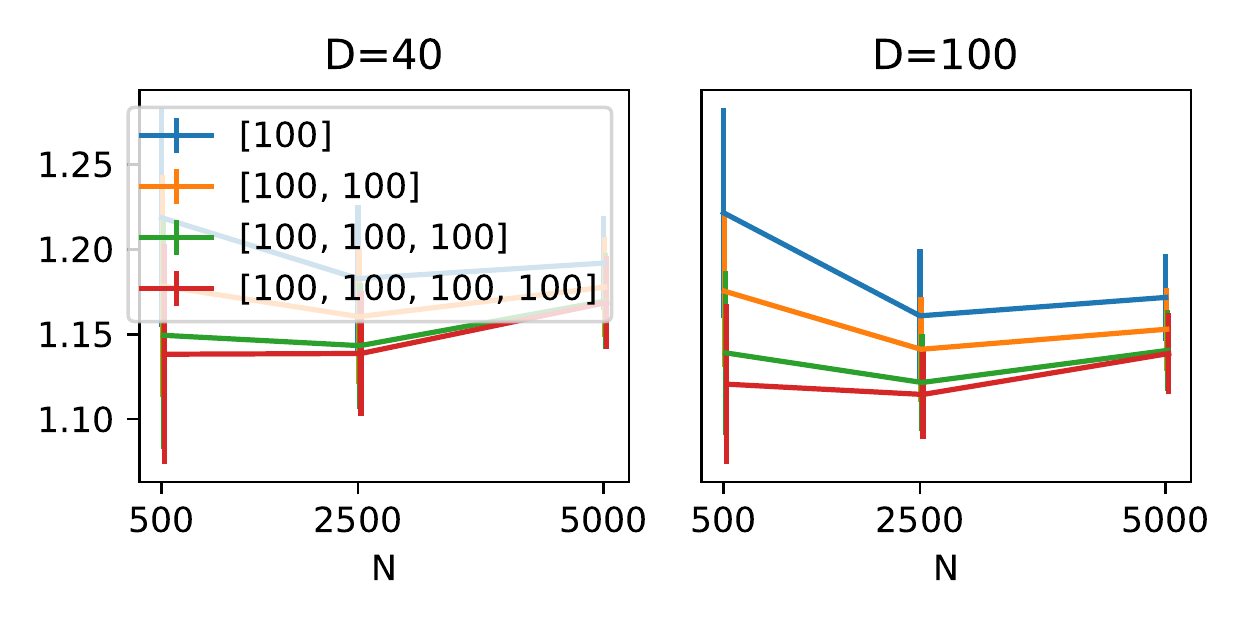}}
\subfigure[dropout rate, test MSE]{\includegraphics[width=0.24\linewidth,clip,trim={0.4cm 0 0.25cm 0}]{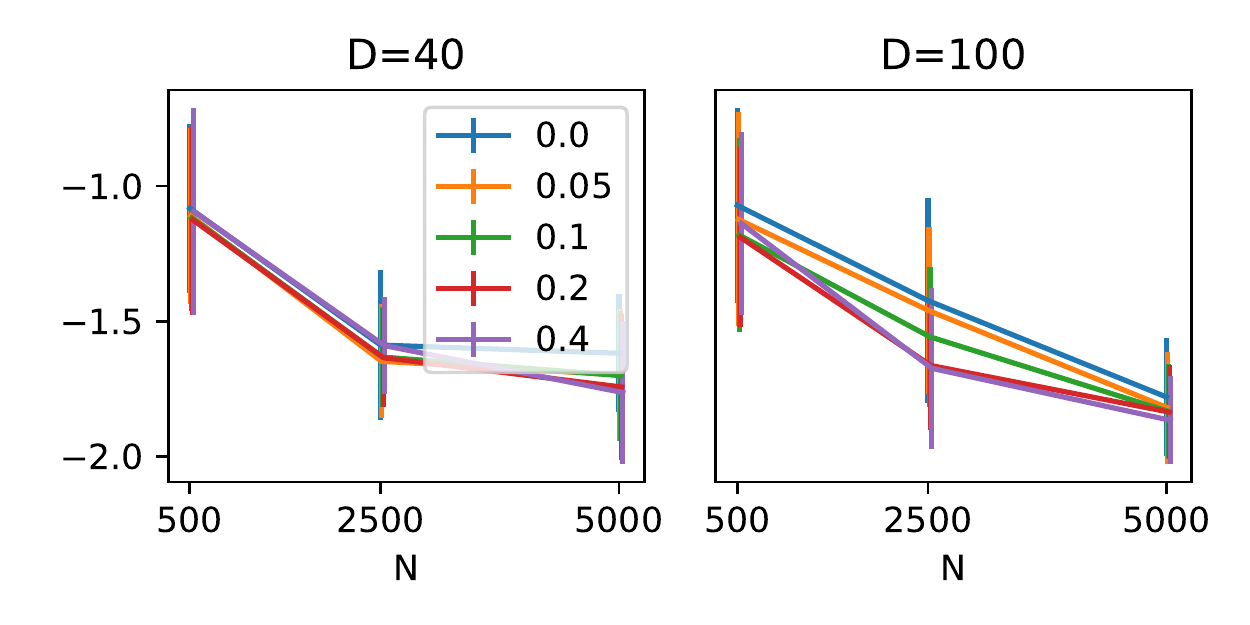}}
\subfigure[dropout rate, val.~stats.]{\includegraphics[width=0.24\linewidth,clip,trim={0.4cm 0 0.25cm 0}]{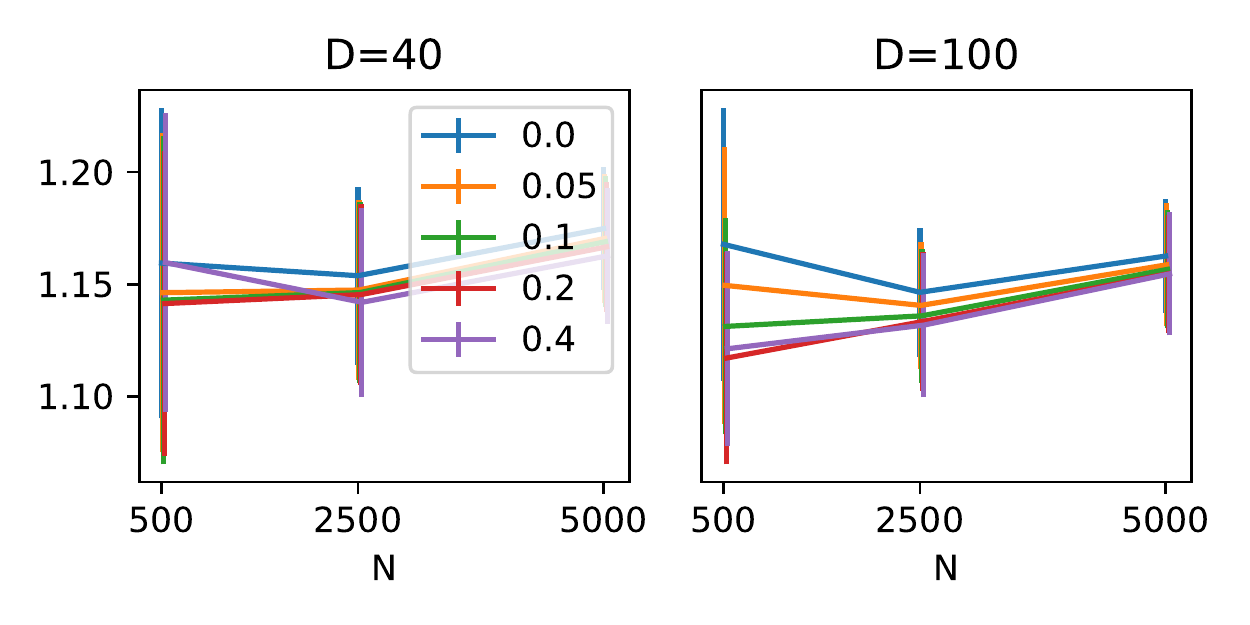}}
\caption{Test MSE and validation statistics grouped by different hyperparameters, for the experiment in Section~\ref{sec:exp-s1-randgp}. The test MSEs are plotted in logarithm scale.}\label{fig:s1-by-hps}
\end{figure}

\subsection{Predictive Performance}\label{sec:exp-main}

In this section we evaluate the predictive performance of the point estimator. 

\paragraph{Low-dimensional and NN-generated instruments} We first generate the observed $\bz$ from the latent feature $\bar\bz_1$ based on settings (i-ii) in the main text, 
and vary $f_0$ in the collection of fixed-form functions in \cite{bauer_deep_2019}. For baselines, we use adversarial GMM \citep[AGMM,][]{dikkala_minimax_2020} implemented with tree, DNN and RBF kernels. Note that AGMM-RBF is computationally similar to KernelIV \cite{singh_kernel_2020}, as can be seen from the closed-form expressions.  

For our method, we determine hyperparameters of the first-stage NN based on the validation results in last subsection's experiments: we use MLPs with hidden layers $\{100,100,100\}$ and swish activation, and set the learning rate to $10^{-3}$, and dropout rate to 0.2. 
We still use a RBF kernel for the second stage, but choose its bandwidth from $\{0.5,1,1.5\}$ times the median distance of input, based on the marginal quasi-likelihood. 

For baselines, the kernelized version of \cite{dikkala_minimax_2020} is implemented as in last subsection, with its first-stage bandwidth determined by the validation statistics \eqref{eq:s1-val-stats}, and second-stage bandwidth chosen from the same grid as our method. 
For AGMM-NN and AGMM-tree, we use the official implementation in \cite{dikkala_minimax_2020}; 
as no official instructions for hyperparameter selection are provided, for simplicity, we provide an optimistic estimate of their performance by enumerating hyperparameters in a reasonably defined grid (which always include the default setting in the official code), and reporting the configuration with the best test error. This amounts to
\begin{itemize}[leftmargin=*]
    \item AGMM-NN: we vary the learning rate hyperparameters in $\{1,5\}\times \{10^{-3},10^{-4}\}$, and the $L^2$ regularization hyperparameters in $\{1,5\}\times \{10^{-4},10^{-5}\}$. We evaluate all 5 variants of NN-based estimators in \cite[Fig.~21]{dikkala_minimax_2020}. 
    \item AGMM-Tree: we vary the depth of the tree in $\{2,4,6,8\}$, and the number of iterations in $\{100,200,400,800,1600\}$.
\end{itemize}

For each setup we evaluate on 20 independently generated datasets. For our method, training takes a total of 2.7 hours on 8 RTX 3090 GPUs. 

\begin{figure}[ht]
    \centering
\includegraphics[width=\linewidth,clip,trim={0.25cm 0.3cm 0.35cm 0.25cm}]{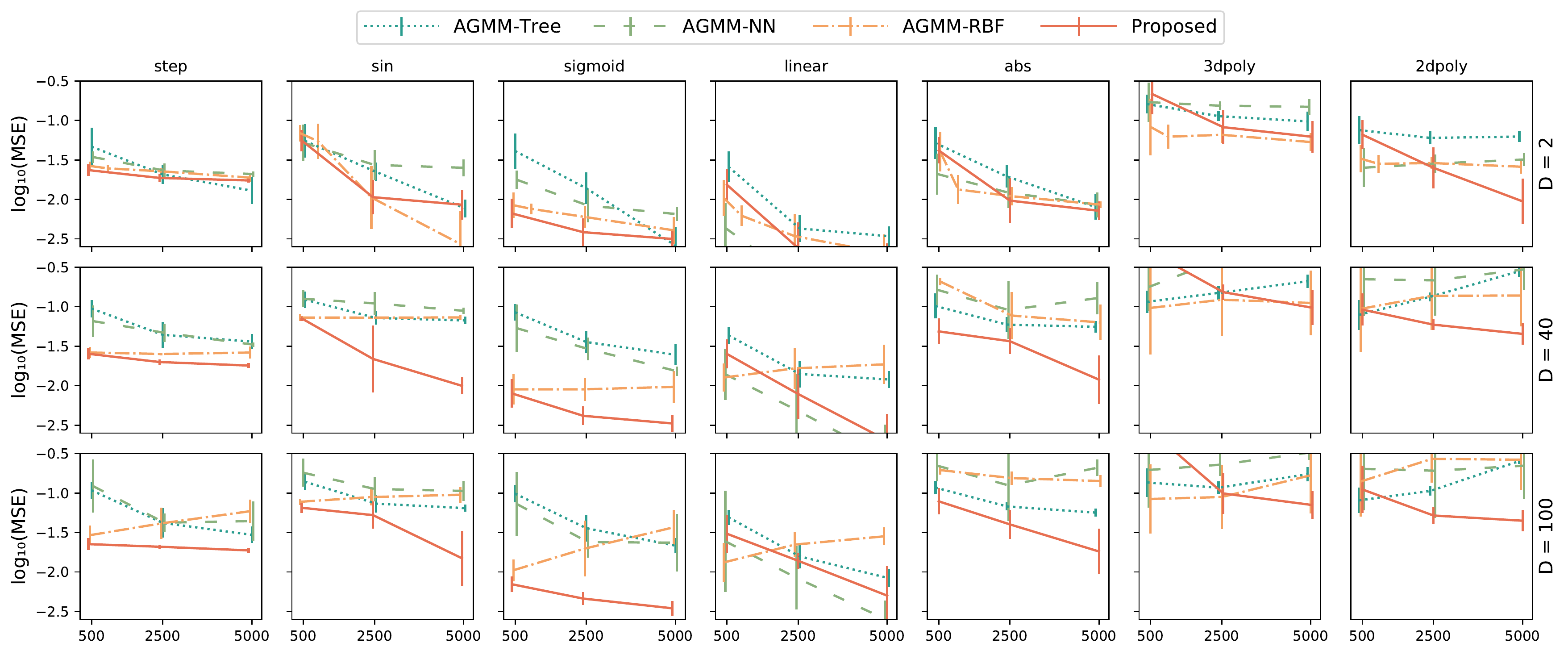}
\caption{Predictive performance: full results for the low-dimensional and NN-generated instruments. Error bars indicate standard deviation from 20 independent experiments.
}\label{fig:predictive-main-full}.
\end{figure}

Full results are plotted in Figure~\ref{fig:predictive-main-full}. All methods have competitive performance in the low-dimensional setting, which is consistent with the report in previous work. 
As $D$ grows, however, 
the performance of all baselines worsens, and only the proposed method is able to maintain a similar level of precision, which idicates very good adaptivity. 
The deterioration is to be expected for the kernelized baselines which do not adapt to any type of informative latent structure. While the tree models may perform variable selection, which can be viewed as adapting to a special type of linear latent structures, less is known about its adaptivity to nonlinear latent structures. 
The observed deterioration suggests tree-based models are less competitive in this regime, although it may also be attributed to challenges in minimax optimization, which clearly explains the deteriorated performance of the AGMM-NN baseline. 

\paragraph{Image Instruments}\label{sec:exp-image}
We now turn to image-based observations. Following previous work \citep{bauer_deep_2019,dikkala_minimax_2020,muandet_dual_2020}, we fix $f_0$ to be the abs function, and we map the latent feature $\bar \bz_1\sim\mrm{Unif}[-3,3]$ to $\lfloor\frac{3}{2}\bz_1+5\rfloor\sim\mrm{Unif}\{0,\ldots,9\}$, and use a random MNIST / CIFAR-10 image with the corresponding label as the observed instrument. We expect the MNIST setting to be less challenging than typical high-dimensional feature learning scenarios, as the image label explains a large proportion of variance. It is also known that kernel-based classifiers achieve very good accuracy on MNIST.\footnote{
\url{http://yann.lecun.com/exdb/mnist/} reports a test accuracy of $98.6\%$ for an SVM with RBF kernels. 
}

We compare with AGMM implemented with DNNs, and the kernelized version of MMR-IV \citep{zhang_maximum_2020}. We use the official implementation of the baselines. For our method, we use a convolutional neural network for instrument learning. Its architecture follows \cite{hartford2017deep}, with the exceptions that we double all hidden dimensions and remove the dropout layer. We find the increase in capacity necessary to approximate GP prior draws. We also experimented with a ResNet-18 model which led to similar results. 
We use the same RBF second stage as before. 

We generate $N_1+N_2=10000$ samples. For our method and AGMM-NN, we split the samples evenly. MMV-IV does not require a separate validation set, so we use all generated samples for training. 

All methods are evaluated on 10 independently generated datasets. For MMR-IV, its hyperparameter selection procedure is occasionally unstable, so for each randomly generated dataset, we repeat the procedure 20 times from random initial values. For our method, training takes a total of 1 hour on two TITAN X GPUs.

Results are presented in Table~\ref{tbl:image}. We can see that our method is still the most competitive, although the kernel-based MMR-IV also performs well, especially in the MNIST setting. 

\begin{table}[h]
    \centering\small
    \begin{tabular}{cccc}
        \toprule
        Test MSE & AGMM-NN & MMR-IV-Nystrom & Proposed \\ \midrule
        MNIST & $0.061$ {\footnotesize [$0.056$, $0.064$] } 
        & $0.011$ {\footnotesize [$0.008$, $0.018$] } 
        & $\mathbf{0.008}$ {\footnotesize [$0.007$, $0.009$] }
\\ 
        CIFAR-10 & $0.117$ {\footnotesize [$0.109$, $0.128$] } 
        & $0.024$ {\footnotesize [$0.013$, $0.045$] }
        & $\mathbf{0.012}$ {\footnotesize [$0.009$, $0.013$] }
        \\
        \bottomrule
    \end{tabular}
    \caption{Image experiment: median, $25\%$ and $75\%$ percentile of test MSE. {\bf Boldface} indicates the best result ($p<0.05$ in Mann-Whitney U test).
    }\label{tbl:image}
\end{table}

\begin{table}[h]\centering\small
    \begin{tabular}{ccc}
        \toprule 
        & MNIST & CIFAR-10 \\ \midrule 
    Test MSE & $0.008$ {\footnotesize [$0.007$, $0.009$] } & $0.012$ {\footnotesize [$0.009$, $0.013$] } \\ 
    CB.~Rad. %
    & $0.011$ {\footnotesize [$0.009$, $0.013$] } &
    $0.041$ {\footnotesize [$0.021$, $0.060$] } \\ \bottomrule
    \end{tabular}
    \caption{Image experiment: test MSE and radius of the $90\%$ $L_2$ credible ball, for the proposed method. We report median and $25\%$ and $75\%$ quantiles. 
    }\label{tbl:image-uq}
\end{table}

\subsection{Uncertainty Quantification}\label{app:exp-uq}
\paragraph{Setup}
We evaluate the quasi-Bayesian uncertainty estimates in two settings:
\begin{enumerate}[leftmargin=*]
    \item To rule out the influence from misspecification, we generate $f_0\sim\mc{GP}(0,k_x)$, and construct credible sets using the GP quasi-posterior based on $\mc{GP}(0,k_x)$. Note this is still a non-trivial setting, due to the need for instrument learning. 
    \item For the evaluation of predictive performance, we vary $f_0$ in the collection of functions in \cite{bennett2019deep}. Note the potential misspecification: the RBF kernel is only suited for approximating H\"older regular functions, i.e., functions with bounded (high-order) derivatives \citep{van_der_vaart_adaptive_2009}. \\
    In this setting we use (quasi-)Bayesian model averaging (BMA), constructed from a grid of GP priors based on RBF kernels $k_x(x,x')=\sigma^2\exp(-(x-x')^2/2h^2)$, with varying $\sigma$ and $h$. 
\end{enumerate}

The observed instrument is generated under the low-dimensional or NN-based setting. For image instruments, our method provides reliable coverage in the setting of Table~\ref{tbl:image}, as shown in Table~\ref{tbl:image-uq}. However, we did not experiment with other choices of $N,D$ or $f_0$ in the image setting, due to the increased computational cost.

Hyperparameters in the instrument learning algorithm are set as in Section~\ref{sec:exp-main}. 
The expression for the marginal likelihood is in \eqref{eq:s2-log-qlh}. 
For BMA, we consider $\sigma\in\{0.5,1,1.5,2,2.5,3\}$ and $h\in\{0.5,1,1.5\}$. We reweigh the hyperparameters on the two-dimensional grid, by imposing a (discretized) $\mrm{InvGamma}(2,2)$ prior for $\sigma$ and a $\mrm{Gamma}(2,1)$ prior for $h$. 

As a baseline, we replace the learned first-stage with a fixed-form RBF kernel, with bandwidth determined by the first-stage validation statistics. 

For each setup, evaluation is repeated on 300 independently generated datasets in the GP setting, and 20 generations in the setting of \cite{bennett2019deep}. 
For our method, training takes a total of 1.8 hours in the GP setting, and 2.7 hours in the setting of \cite{bennett2019deep}. The experiments are conducted on 8 RTX 3090 GPUs. 

\begin{table}[h]
    \centering\small
    \begin{tabular}{cccccc}\toprule
$D$ & $N_1=N_2$ & MSE & $90\%$ CB.~Rad. & $90\%$ CB.~Cvg. & Avg.~$90\%$ CI.~Cvg. \\ 
\midrule \multicolumn{5}{l}{ Proposed } \\ \midrule \multirow{3}{*}{$2$} 
& $500$ & $.069$ {\tiny $\pm .050$} & $.140$ {\tiny $\pm .019$} & $.930$ {\tiny $[.895, .954]$} & $.919$ {\tiny $\pm .125$}\\
& $2500$ & $.028$ {\tiny $\pm .020$} & $.056$ {\tiny $\pm .007$} & $.907$ {\tiny $[.868, .935]$} & $.914$ {\tiny $\pm .124$}\\
& $5000$ & $.020$ {\tiny $\pm .014$} & $.038$ {\tiny $\pm .004$} & $.910$ {\tiny $[.872, .937]$} & $.908$ {\tiny $\pm .124$}\\
\midrule \multirow{3}{*}{$40$} 
& $500$ & $.091$ {\tiny $\pm .064$} & $.179$ {\tiny $\pm .022$} & $.907$ {\tiny $[.868, .935]$} & $.908$ {\tiny $\pm .132$}\\
& $2500$ & $.033$ {\tiny $\pm .024$} & $.067$ {\tiny $\pm .007$} & $.910$ {\tiny $[.872, .937]$} & $.912$ {\tiny $\pm .137$}\\
& $5000$ & $.023$ {\tiny $\pm .014$} & $.046$ {\tiny $\pm .004$} & $.917$ {\tiny $[.880, .943]$} & $.908$ {\tiny $\pm .125$}\\
\midrule \multirow{3}{*}{$100$}
& $500$ & $.097$ {\tiny $\pm .065$} & $.201$ {\tiny $\pm .025$} & $.923$ {\tiny $[.888, .948]$} & $.915$ {\tiny $\pm .123$}\\
& $2500$ & $.035$ {\tiny $\pm .024$} & $.074$ {\tiny $\pm .008$} & $.917$ {\tiny $[.880, .943]$} & $.908$ {\tiny $\pm .127$}\\
& $5000$ & $.024$ {\tiny $\pm .016$} & $.049$ {\tiny $\pm .004$} & $.920$ {\tiny $[.884, .946]$} & $.905$ {\tiny $\pm .134$}\\
\midrule \multicolumn{5}{l}{ RBF first stage } \\ \midrule \multirow{3}{*}{$2$}
& $500$ & $.071$ {\tiny $\pm .052$} & $.144$ {\tiny $\pm .021$} & $.920$ {\tiny $[.884, .946]$} & $.920$ {\tiny $\pm .122$}\\
& $2500$ & $.028$ {\tiny $\pm .020$} & $.057$ {\tiny $\pm .007$} & $.916$ {\tiny $[.879, .943]$} & $.917$ {\tiny $\pm .123$}\\
& $5000$ & $.020$ {\tiny $\pm .013$} & $.039$ {\tiny $\pm .005$} & $.917$ {\tiny $[.880, .943]$} & $.913$ {\tiny $\pm .117$}\\
\midrule \multirow{3}{*}{ $40$ } 
& $500$ & $.124$ {\tiny $\pm .072$} & $.218$ {\tiny $\pm .034$} & $.870$ {\tiny $[.827, .903]$} & $.907$ {\tiny $\pm .107$}\\
& $2500$ & $.100$ {\tiny $\pm .069$} & $.132$ {\tiny $\pm .020$} & $.723$ {\tiny $[.670, .771]$} & $.837$ {\tiny $\pm .187$}\\
& $5000$ & $.094$ {\tiny $\pm .067$} & $.108$ {\tiny $\pm .017$} & $.660$ {\tiny $[.605, .711]$} & $.792$ {\tiny $\pm .211$}\\
\midrule \multirow{3}{*}{ $100$ } 
& $500$ & $.431$ {\tiny $\pm .192$} & $.240$ {\tiny $\pm .036$} & $.187$ {\tiny $[.147, .235]$} & $.640$ {\tiny $\pm .191$}\\
& $2500$ & $.176$ {\tiny $\pm .089$} & $.175$ {\tiny $\pm .023$} & $.517$ {\tiny $[.460, .573]$} & $.822$ {\tiny $\pm .136$}\\
& $5000$ & $.126$ {\tiny $\pm .072$} & $.156$ {\tiny $\pm .019$} & $.660$ {\tiny $[.605, .711]$} & $.855$ {\tiny $\pm .143$}\\
\bottomrule
\end{tabular}
\caption{Full results for single-model uncertainty quantification: test MSE, estimated coverage of $90\%$ $L_2$ credible ball and average coverage of pointwise $90\%$ credible interval. 
For the CB coverage rate estimate we report its $95\%$ Wilson score interval. For other statistics we report standard deviation. Results averaged over 300 independent runs. 
}\label{tbl:uq-gprand-full}
\end{table}

\paragraph{Results}
Full results in the correctly specified setting are reported in Table~\ref{tbl:uq-gprand-full}. As we can see, both methods have similar behaviors in the low-dimensional setting. However, in high dimensions, only the proposed method 
produces reliable uncertainty estimates with the correct nominal coverage, whereas using a fixed-form kernel as first stage leads to undercoverage and lareger credible sets. 
The latter observation can be understood from the correspondence between the quasi-posterior marginal variance, and a certain worst-case prediction error on a simplified data generating process \citep[Section 5]{wang2021quasibayesian}. The linear IV literature has also related the use of first stage models with insufficient predictive power with IV regression given weakly informative instruments \citep[see e.g.,][]{davies_many_2015}.  

\begin{figure}[ht]
    \centering
    \subfigure[$D=2,N_1=500$]{\includegraphics[width=0.95\linewidth]{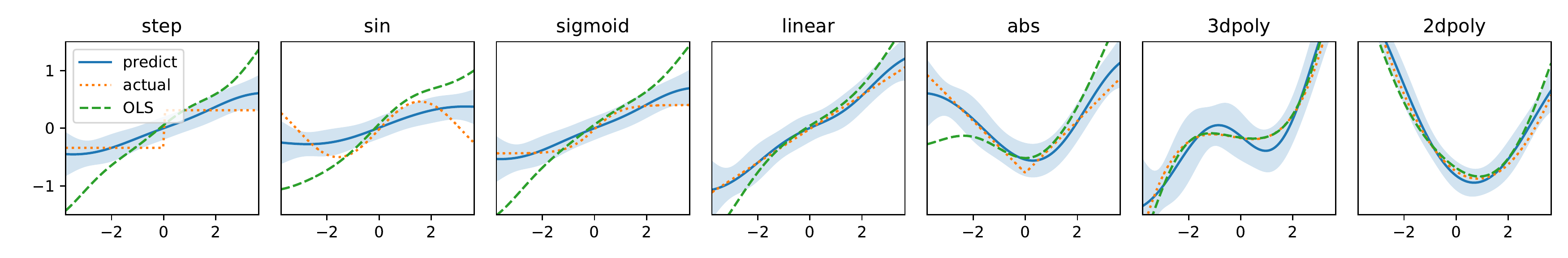}}
    \subfigure[$D=2,N_1=5000$]{\includegraphics[width=0.95\linewidth]{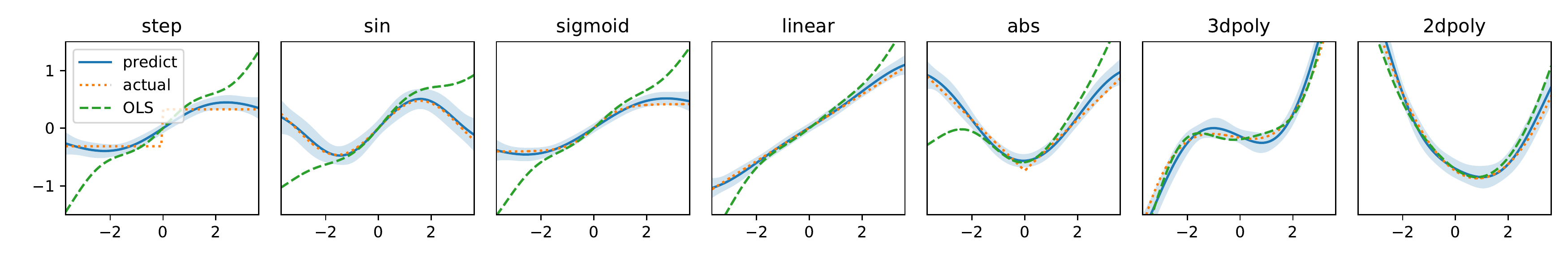}}
    \subfigure[$D=100,N_1=500$]{\includegraphics[width=0.95\linewidth]{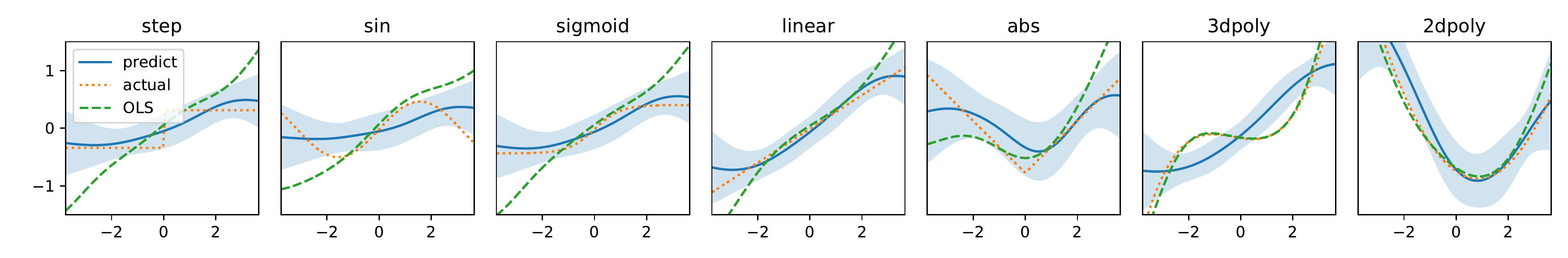}}
    \subfigure[$D=100,N_1=5000$]{\includegraphics[width=0.95\linewidth]{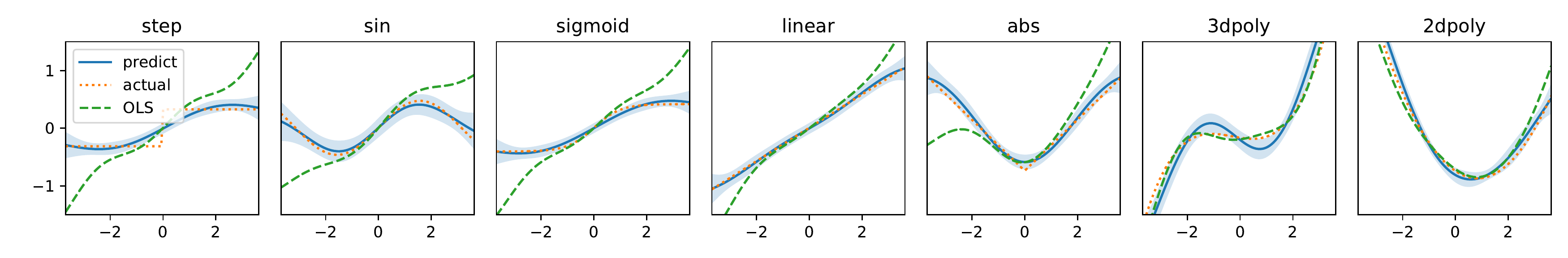}}
    \caption{Visualization of pointwise 90\% credible interval, using BMA and the learned instruments, for varying choices of $N_1=N_2$ and $D$. OLS denotes a biased regression estimate using KRR.}\label{fig:bma-full}
\end{figure}

\begin{figure}[hptb]
 \centering
 \subfigure[RBF first stage]{
 \includegraphics[width=0.92\linewidth]{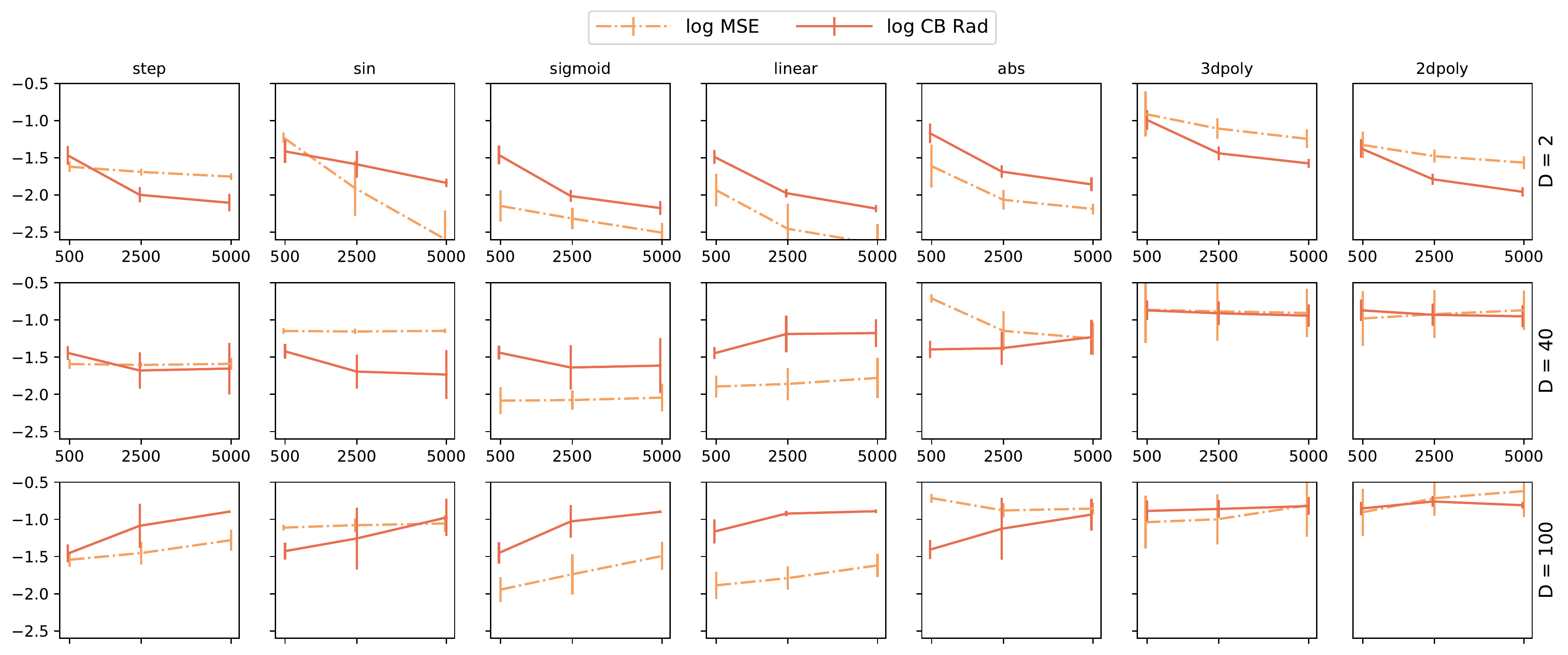}}
 \subfigure[Proposed]{
 \includegraphics[width=0.92\linewidth]{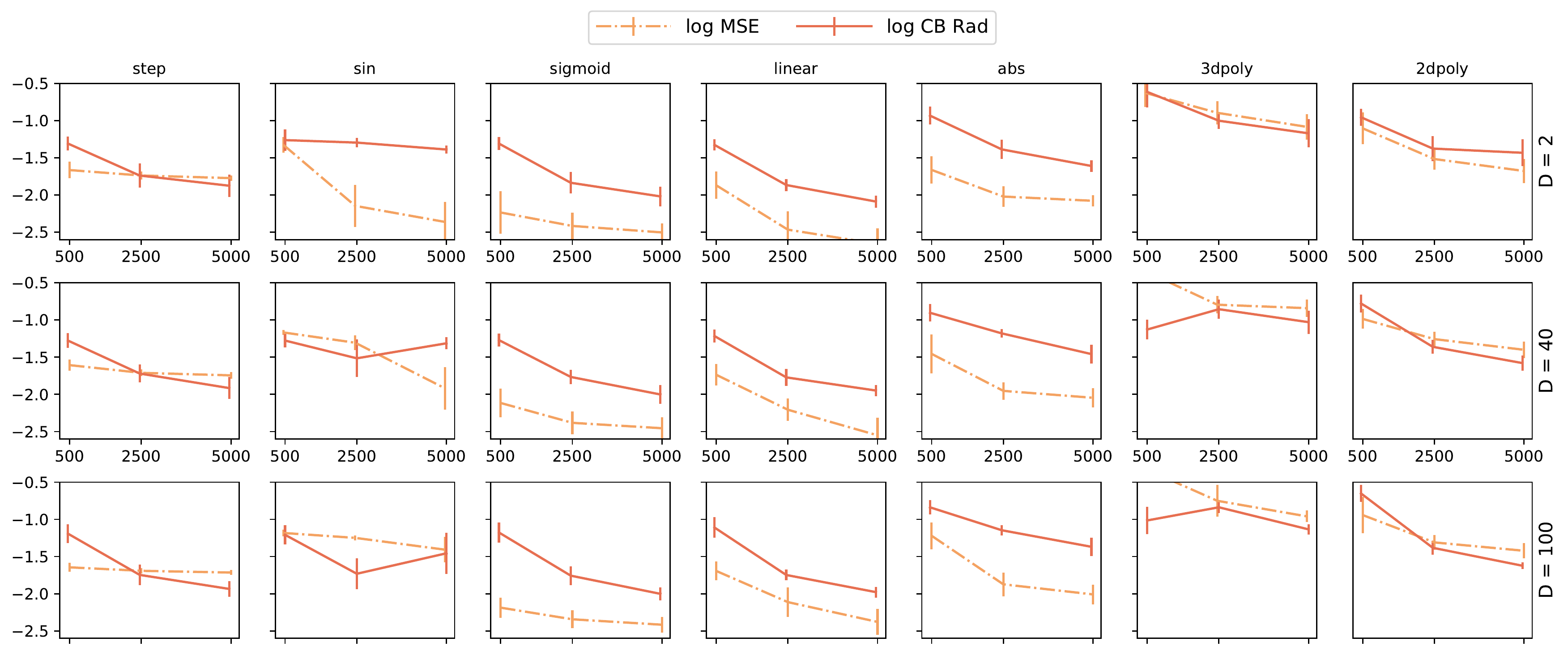}}
 \caption{Results for Bayesian model averaging: counterfactual MSE vs the radius of $90\%$ $L_2$ credible ball.}\label{fig:bma-cb-rad}
 \vspace{-1em}
\end{figure}

For BMA, we visualize the resulted quasi-posterior in Figure~\ref{fig:bma-full}, and report the radius of $90\%$ $L_2$ credible ball in Figure~\ref{fig:bma-cb-rad}. We can see that using learned instruments lead to sharper uncertainty estimates than using a fixed-form kernel for first stage. The resulted uncertainty estimates are still slightly conservative when the model is correctly specified (e.g., sin or linear), or the misspecification is mild (e.g., abs); but under-coverage can occur in the presence of more severe misspecification, such as the step design, or the two polynomial designs. For the polynomial designs, note that our data distribution for $\bx$ has a gaussian-like tail, so their unbounded growth of function values and derivatives can be problematic. Preliminary experiments show that polynomial kernels leads to significantly better marginal likelihood on the polynomial functions than the RBF kernels, and improved coverage and MSEs, without affecting the other designs. In aggregate, these results highlight the need for a correctly specified model for reliable uncertainty quantification.

\subsection{Exogenous Covariates}\label{app:exp-exo}

Finally, we evaluate our extended algorithm for exogenous covariates, developed in Appendix~\ref{app:exo-algo}. 

\paragraph{Setup} We use the setup in \cite{hartford2017deep}, which simulates the prediction of airline demand. The structural function is $$
f_0(p,t,s) = 100 + (10+p)s\psi(t) - 2p, ~~\text{where}~~\psi(t) = 2\Big[\frac{(t-5)^2}{600} + e^{-4(t-5)^2}+\frac{t}{10}-2p\Big].
$$
 The observational distribution is defined as 
$$
\begin{aligned}
\bs\sim\mrm{Unif}\{1,\ldots,7\},~ &\bt\sim\mrm{Unif}[0,10], ~(\bc,\bv) \sim\cN(0, I), ~\bp = 25+(\bc+3)\psi(\bt) + \bv,\\ 
\bu&\sim\cN(\rho\bv, 1-\rho^2), ~ \by = f_0(\bp,\bt,\bs)+\bu.
\end{aligned}
$$
In our notations, $\bx_o=\bp$ is the treatment, $\bz_o=\bc$ is the instrument, and $\bw=(\bt,\bs)$ are the additional exogenous covariates. Following all previous work, we consider two variants: 
\begin{enumerate}[leftmargin=*]
    \item In the {\em low-dimensional} setting, we directly observe $\bs$. Following \cite{singh_kernel_2020,muandet_dual_2020,xu_learning_2020} we use a univariate real-valued input as $\bs$.
    \item In the {\em image} setting, we only observe a high-dimensional surrogate of $\bs$, defined as a random MNIST image of the respective class.
\end{enumerate}

For our method, in the low-dimensional setting, we use an MLP with hidden layers $[128,64,32]$ and swish activation. (The architecture is changed to match \cite{hartford2017deep}.) In the high-dimensional setting, we first embed the image feature into a $64$-dimensional representation, using ConvNet architecture in Appendix~\ref{sec:exp-main}; then we concatenate it with the other inputs and feed into the aforementioned MLP. 
The other hyperparameters follow the image experiment in Appendix~\ref{sec:exp-main}. 
We conduct early stopping by evaluating the reduced-form prediction error $\ell'_{n_2}$ (see Algorithm~\ref{alg:exo}) on $\dataSI$. 

We compare our method with DeepIV \cite{hartford2017deep}, DeepGMM \cite{bennett2019deep}, AGMM \cite{dikkala_minimax_2020} instantiated with RBF kernel and DNN models, and DFIV \cite{xu_learning_2020}. For DeepIV, DeepGMM and DFIV we use the implementation in \cite{xu_learning_2020}. For AGMM, we use the implementation in \cite{wang2021quasibayesian}, as \cite{dikkala_minimax_2020} did not experiment on this dataset. Note that AGMM-RBF has a similar form to KernelIV \cite{singh_kernel_2020}, and our result for it is consistent with \cite{singh_kernel_2020}. 

All methods require two independent sets of observations, either directly used in the algorithm or for validation. We partition the training set evenly for this purpose. 

All methods are evaluated on 20 independently generated datasets. For our method, all experiments take a total of 6 minutes on 4 Tesla A40 GPUs. 

\paragraph{Results} The results are presented in Table~\ref{tbl:hllt}. \todo{.}We can see that our method has similar performance to \cite{xu_learning_2020}, and outperforms the other baselines by a large margin. As discussed before, %
our method is more appealing than \cite{xu_learning_2020} from a theoretical perspective. 

\begin{table}
    \centering\scriptsize\setlength{\tabcolsep}{5.5pt}
    \begin{tabular}{ccccccc}
        \toprule
     $n$ & DeepIV & DeepGMM & AGMM-RBF & AGMM-NN & DFIV & Proposed \\  \midrule 
     \multicolumn{6}{l}{Low-dimensional setting} \\  \midrule 
$1000$
    & $ 3.76 $ {\tiny $[ 3.74 , 3.77 ]$} &
    $ 3.97 $ {\tiny $[ 3.94 , 3.99 ]$} &
    $ 3.75 $ {\tiny $[ 3.71, 3.79 ]$} &
    $3.42$ {\tiny $[3.06, 3.99]$} &
    $3.00$ {\tiny $[2.94, 3.10]$} &
    $2.94$ {\tiny $[2.85, 3.06]$} \\
$5000$
    & $ 3.14 $ {\tiny $[ 3.10 , 3.21 ]$}
    & $ 3.94 $ {\tiny $[ 3.91 , 3.96 ]$}
    & $ 3.50 $ {\tiny $[ 3.46, 3.52 ]$}
    & $ 2.74 $ {\tiny $[ 2.66 , 2.76 ]$}
    & $2.38$ {\tiny $[2.31, 2.53]$}
    & $2.39$ {\tiny $[2.30, 2.47]$} \\ \midrule
     \multicolumn{6}{l}{Image setting} \\  \midrule 
$5000$
    & $3.96$ {\tiny $[3.93, 4.01]$} & 
    $ 4.41 $ {\tiny $[ 4.38 , 4.45 ]$} & 
    $ 4.03 $ {\tiny $[ 4.02, 4.05 ]$} & 
    $4.20$ {\tiny $ [ 4.10, 4.33] $} & 
    $3.83$ {\tiny $[3.78, 3.92]$} &
    $3.87$ {\tiny $[3.85, 3.92]$} \\
     \bottomrule
    \end{tabular}
    \caption{Demand design: log test MSE vs the total sample size ($n=n_1+n_2$). We report the median, $25\%$ and $75\%$ percentile over $20$ replications.}\label{tbl:hllt}
\end{table}
\end{document}